\documentclass[a4 paper,12pt]{article}
\usepackage[margin=1in,footskip=0.5in]{geometry}
\usepackage{amsthm,amsmath,amsfonts,amssymb}
\usepackage{bbm}
\usepackage{authblk}
\usepackage{tikz}
\usetikzlibrary{positioning, arrows.meta, calc}
\usepackage{hyperref}

\newtheorem{theorem}{Theorem}
\newtheorem{lemma}{Lemma}
\newtheorem{proposition}{Proposition}
\newtheorem{definition}{Definition}

\newtheorem{remark}{Remark}

\allowdisplaybreaks[4]

\title{Standard Transformers Achieve the Minimax Rate in Nonparametric Regression with $C^{s,\lambda}$ Targets}

\author[a]{Yanming Lai\thanks{Corresponding Author (yanming.lai@polyu.edu.hk)}}
\author[b]{Defeng Sun}
\affil[a,b]{Department of Applied Mathematics, The Hong Kong Polytechnic University, Hung Hom, Hong Kong, China}

\date{}

\begin{document}

\maketitle

\begin{abstract}

The tremendous success of Transformer models in fields such as large language models and computer vision necessitates a rigorous theoretical investigation. To the best of our knowledge, this paper is the first work proving that standard Transformers can approximate Hölder functions $  C^{s,\lambda}\left([0,1]^{d\times n}\right)  $$  (s\in\mathbb{N}_{\geq0},0<\lambda\leq1)  $ under the $L^t$ distance ($t \in [1, \infty]$) with arbitrary precision. Building upon this approximation result, we demonstrate that standard Transformers achieve the minimax optimal rate in nonparametric regression for Hölder target functions. It is worth mentioning that, by introducing two metrics: the size tuple and the dimension vector, we provide a fine-grained characterization of Transformer structures, which facilitates future research on the generalization and optimization errors of Transformers with different structures. As intermediate results, we also derive the upper bounds for the Lipschitz constant of standard Transformers and their memorization capacity, which may be of independent interest. These findings provide theoretical justification for the powerful capabilities of Transformer models.

\end{abstract}

\section{Introduction}

\subsection{Background}

In recent years, the tremendous success of machine learning in various applications has sparked theoretical research in this field. As an important model in machine learning, neural networks are a key area of study. One of the reasons for the powerful success of neural networks is their strong expressive capability; in other words, neural networks can approximate a wide range of function classes with arbitrary precision. As the most fundamental and simple model in neural networks, feedforward neural networks (FNNs) have received the most extensive and mature investigation in this area. This includes studies on both shallow and deep architectures, different activation functions, and the approximation of various function classes, etc.; see \cite{yarotsky2017error,yarotsky2018optimal,suzuki2018adaptivity,shen2020deep,lu2021deep,shen2022optimal,schmidt2021kolmogorov,guhring2021approximation,mao2023rates,siegel2023optimal,yang2024optimal} and the references therein. In addition to FNNs, there is also a body of work exploring the approximation capabilities of other types of neural networks, such as convolutional neural networks (CNNs) \cite{petersen2020equivalence,zhou2020universality,zhou2020theory}, recurrent neural networks (RNNs) \cite{li2022approximation,hoon2023minimal,jiao2024approximation}. Furthermore, the memorization capacity of neural networks is also closely related to their approximation ability. Research in this direction can be found in \cite{park2021provable,vardi2022on} and the references therein.

Since its introduction by \cite{vaswani2017attention}, the Transformer model has achieved remarkable success across a wide variety of domains, including but not limited to large language models \cite{devlin2019bert} and computer vision \cite{dosovitskiy2021an}. Compared to traditional FNNs, Transformers are more efficient due to their use of self-attention mechanisms, which enable parameter sharing and parallel token processing. However, this architecture also introduces challenges in studying the approximation properties of Transformers: they must capture the entire context of each input sequence rather than simply assigning a label to each token independently. 

Theoretical research on the expressive power of Transformers began with \cite{Yun2020Are}, which introduced the concept of contextual mapping—the ability to distinguish tokens that are in the same sequence but at different positions, or tokens across different input sequences. By demonstrating that Transformers with biased self-attention layers can implement contextual mapping, the authors established a universal approximation theorem for such Transformers regarding continuous functions. This method was later extended to sparse Transformers \cite{yun2020n}. \cite{kratsios2022universal} also applied this approach to study constrained Transformers. \cite{kim2023provable} investigated the closely related memorization problem. They refined the techniques in \cite{Yun2020Are}, enhancing the parameter efficiency of attention layers, and established an upper bound for Transformers solving memorization tasks. Subsequently, \cite{kajitsuka2025on} improved the upper bounds of \cite{kim2023provable} and provided  lower bounds. Other studies on the memorization capabilities of Transformers include \cite{mahdavi2024memorization,madden2024upper}. Leveraging the properties of the Boltzmann operator, \cite{kajitsuka2024are} proved that a single-layer, single-head self-attention mechanism without bias terms is sufficient to achieve contextual mapping. However, this comes at the cost of a separation parameter that decays exponentially with the size of the dictionary. Building on this result, they also established a universal approximation theorem for continuous functions. Recently, there have been some quantitative works in this area. \cite{jiao2025approximation} provided a quantitative characterization of the convergence rate when Transformers approximate functions in Hölder space $C^{0, \lambda}(\Omega)$, i.e., the dependence of Transformer size on approximation accuracy.
Based on the Kolmogorov-Arnold Representation Theorem, \cite{jiao2025transformers} constructed several Transformers that can overcome the curse of dimensionality when approximating functions in Hölder space $C^{0, \lambda}(\Omega)$ by replacing the activation functions in the feed-forward layers. \cite{hu2025universal} utilized an interpolation-based approach to show that self-attention can approximate the ReLU function. By combining this with classical ReLU FNN approximation results, they proved that two-layer multi-head attention can approximate any continuous function. Additional works regarding the expressive capability of Transformers can be found in \cite{bhojanapalli2020low, kim2021lipschitz, gurevych2022rate, edelman2022inductive, wei2022statistically, takakura2023approximation, sanford2023representational, jiang2024approximation, wang2024transformers, cheng2025a, takeshita2025approximation, liu2025generalization,wang2025prompt} and the references therein. However, in these prior studies, there still does not exist any quantitative characterization result for the approximation of functions in general Hölder space $C^{s, \lambda}(\Omega)$ with $s\in\mathbb{N}_{\geq0}$ by Transformers, even though $C^{s, \lambda}(\Omega)$ is a central object of study in approximation theory. In other words, no existing work has considered the approximation rate of the standard Transformer when the target function possesses higher-order continuous derivatives. We therefore pose the first question: 

\textit{At what rate can standard Transformers approximate functions in general Hölder space $C^{s, \lambda}(\Omega)$?}

As a class of non-parametric models, neural networks have been widely used in recent years to solve regression problems. A series of prior works \cite{schmidt2020nonparametric, nakada2020adaptive, kohler2021rate, farrell2021deep, chen2022nonparametric, jiao2023deep, fan2024noise, yang2024nonparametric} has demonstrated that FNNs with various architectures can achieve the minimax optimal rate in regression tasks. In contrast, research on utilizing Transformers for regression remains relatively limited at present. \cite{takakura2023approximation} investigated Transformers with infinite-dimensional inputs and derived error bounds for regression under the assumption of anisotropic smoothness in the target function. \cite{gurevych2022rate, havrilla2024understanding} studied error estimation for Transformers using hardmax and ReLU, respectively, as activation functions in the attention layers.
\cite{jiao2025approximation} analyzed the error bounds of standard Transformers for regression problems and obtained a sub-optimal rate under the assumptions of $C^{0, \lambda}$-continuous target functions and weakly dependent data. None of these studies have answered the second key question we raise:

\textit{Can standard Transformers achieve the minimax optimal rate when applied to regression problems with $C^{s, \lambda}$-continuous targets?}

\subsection{Main Results}

Our main results answer the two questions posed in the previous section. Before stating these results, we first give a brief introduction to the standard Transformer architecture. A Transformer consists of three components: the embedding layer, the feedforward block and the self-attention layer. An embedding layer $\boldsymbol{\mathcal{F}}_{EB}:\mathbb{R}^{d_{in}\times n}\to\mathbb{R}^{d_{EB}\times n}$ is an affine transformation, defined as
\begin{align*}
\boldsymbol{\mathcal{F}}_{EB}(\boldsymbol{X}):=\boldsymbol{W}_{EB}\boldsymbol{X}+\boldsymbol{B}_{EB},
\end{align*}
where $\boldsymbol{W}_{EB}\in\mathbb{R}^{d_{EB}\times d_{in}},\boldsymbol{B}_{EB}\in\mathbb{R}^{d_{EB}\times n}$. A feedforward block $\boldsymbol{\mathcal{F}}_{FF}:\mathbb{R}^{d_{FF}^{(in)}\times n}\to\mathbb{R}^{d_{FF}^{(out)}\times n}$ of depth $L$ and width $W$ is recursively defined as
\begin{align*}
\boldsymbol{\mathcal{F}}_0&:=\boldsymbol{X};\\
\boldsymbol{\mathcal{F}}_{l}&:=\sigma_R(\boldsymbol{W}_{l}\boldsymbol{\mathcal{F}}_{l-1}+\boldsymbol{b}_{l}\boldsymbol{1}_{1\times n}),\quad l\in\{1,2,\dots,L-1\};\\
\boldsymbol{\mathcal{F}}_{FF}&:=\boldsymbol{W}_{L}\boldsymbol{\mathcal{F}}_{L-1}+\boldsymbol{b}_{L}\boldsymbol{1}_{1\times n},
\end{align*}
where $\boldsymbol{W}_{1}\in\mathbb{R}^{W\times d_{FF}^{(in)}},\boldsymbol{W}_{l}\in\mathbb{R}^{W\times W}(l=2,\cdots,L-1),\boldsymbol{W}_{L}\in\mathbb{R}^{d_{FF}^{(out)}\times W},\boldsymbol{b}_{l}\in\mathbb{R}^{W}(l=1,\cdots,L-1),\boldsymbol{b}_{L}\in\mathbb{R}^{d_{FF}^{(out)}}$, $\sigma_R$ is the element-wise ReLU function: $\sigma_R(x)=\max\{x,0\}$. 

In this paper, we say that $\boldsymbol{\mathcal{F}}_{FF}$ is generated from its feedforward neural network (FNN) counterpart $\boldsymbol{f}_{FF}:\mathbb{R}^{d_{FF}^{(in)}}\to\mathbb{R}^{d_{FF}^{(out)}}$, which acts on vectors:

\begin{align*}
\boldsymbol{f}_0&:=\boldsymbol{x};\\
\boldsymbol{f}_{l}&:=\sigma_R(\boldsymbol{W}_{l}\boldsymbol{f}_{l-1}+\boldsymbol{b}_l),\quad l\in\{1,2,\dots,L-1\};\\
\boldsymbol{f}_{FF}&:=\boldsymbol{W}_{L}\boldsymbol{f}_{L-1}+\boldsymbol{b}_{L}.
\end{align*}
A self-attention layer $\boldsymbol{\mathcal{F}}_{SA}:\mathbb{R}^{d_{SA}\times n}\to\mathbb{R}^{d_{SA}\times n}$ of head number $H$ and head size $S$ is defined as
\begin{align*}
&\boldsymbol{\mathcal{F}}_{SA}(\boldsymbol{X}):=
\boldsymbol{X}+\sum_{h=1}^H \boldsymbol{W}_{O}^{(h)} \boldsymbol{W}_{V}^{(h)} \boldsymbol{X} \sigma_S\left( \boldsymbol{X}^{\top}\boldsymbol{W}_{K}^{(h)\top}\boldsymbol{W}_{Q}^{(h)} \boldsymbol{X}\right),
\end{align*}
where $\boldsymbol{W}_{O}^{(h)}\in\mathbb{R}^{d_{SA}\times S},\boldsymbol{W}_{V}^{(h)},\boldsymbol{W}_{K}^{(h)},\boldsymbol{W}_{Q}^{(h)}\in\mathbb{R}^{S\times d_{SA}}$, $\sigma_S$ is the column-wise softmax function: $[\sigma_S(\boldsymbol{x})]_i=e^{x_i}/\left(\sum_{j=1}^ne^{x_j}\right),i\in[n],\boldsymbol{x}\in\mathbb{R}^n$.

A Transformer $\boldsymbol{T}:\mathbb{R}^{d_{in}\times n}\to\mathbb{R}^{d_{out}\times n}$ of length $K$ is defined as an initial embedding layer $\boldsymbol{\mathcal{F}}_{EB}:\mathbb{R}^{d_{in}\times n}\to\mathbb{R}^{d_{0}\times n}$ followed by the alternating composition of $  K+1  $ feed-forward blocks $\boldsymbol{\mathcal{F}}_{FF}^{(k)}:\mathbb{R}^{d_k\times n}\to\mathbb{R}^{d_{k+1}\times n}$ and $  K  $ self-attention layers $\boldsymbol{\mathcal{F}}_{SA}^{(k)}:\mathbb{R}^{d_k\times n}\to\mathbb{R}^{d_{k}\times n}$:
\begin{align}\label{Transformer}
\boldsymbol{T}(\boldsymbol{X}):=
\boldsymbol{\mathcal{F}}_{FF}^{(K)}\circ\boldsymbol{\mathcal{F}}_{SA}^{(K)}\circ\boldsymbol{\mathcal{F}}_{FF}^{(K-1)}\circ\cdots\circ\boldsymbol{\mathcal{F}}_{FF}^{(1)}\circ\boldsymbol{\mathcal{F}}_{SA}^{(1)}\circ\boldsymbol{\mathcal{F}}_{FF}^{(0)}\circ\boldsymbol{\mathcal{F}}_{EB}(\boldsymbol{X}).
\end{align}
We use two metrics to characterize the structure of $\boldsymbol{T}$: the size tuple
\begin{align*}
\left\{\left(L_0,W_0\right),\left(H_1,S_1\right),\left(L_1,W_1\right),\cdots,\left(H_K,S_K\right),\left(L_K,W_K\right)\right\}
\end{align*}
and the dimension vector
\begin{align*}
\boldsymbol{d}:=\begin{pmatrix}
d_{in}&d_0&d_1&\cdots&d_{K}&d_{out}
\end{pmatrix}.
\end{align*}
We use $B_{EB},B_{FF}$ and $B_{SA}$ to denote the upper bounds of the parameters in the embedding layer, all feedforward blocks and all self-attention layers, respectively. And we denote by $M_{EB},M_{FF}$ and $M_{SA}$ the total number of the parameters in all feedforward blocks and self-attention layers, respectively.

We first consider the approximation of H\"older function by Transformers. Let $\Omega$ be a set in $\mathbb{R}^d$. Let $\boldsymbol{\mu} = (\mu_1, \dots, \mu_d)\in\mathbb{N}_{\geq1}^d$ be a multi-index and denote $|\boldsymbol{\mu}|:=\mu_1+\cdots+\mu_d$. Denote $D^{\boldsymbol{\mu}}f:=\frac{\partial^{|\boldsymbol{\mu}|}f(\boldsymbol{x})}{\partial x_1^{\mu_1}\cdots \partial x_d^{\mu_d}}$ as the partial derivatives of function $f:\mathbb{R}^d\to\mathbb{R}$. Let $s\in\mathbb{N}_{\geq0},0<\lambda\leq1$. Denote $\gamma:=s+\lambda$. A function $f$ belongs to the H\"older space $C^{s, \lambda}(\Omega)$ if it has finite $C^{s, \lambda}$ norm, which is defined as
\[
\|f\|_{C^{s, \lambda}(\Omega)} := \max \left\{ \|f\|_{C^s(\Omega)}, \max_{|\boldsymbol{\mu}| = s} |D^{\boldsymbol{\mu}} f|_{C^{0, \lambda}(\Omega)} \right\},
\]
where
\begin{align*}
\|f\|_{C^s(\Omega)} &:= \max_{|\boldsymbol{\mu}|\leq s}\sup_{\boldsymbol{x}\in\Omega} |D^{\boldsymbol{\mu}}f(\boldsymbol{x}) |, \quad
|f|_{C^{0,\lambda}(\Omega)} := \sup_{\boldsymbol{x} \neq \boldsymbol{y} \in\Omega} \frac{|f(\boldsymbol{x}) - f(\boldsymbol{y})|}{\|\boldsymbol{x} - \boldsymbol{y}\|_2^{\lambda}}.
\end{align*}
In Theorems \ref{Lr approximation} and \ref{Linfty}, we respectively establish approximation results for Hölder functions by Transformers in the $  L^t  $ norm ($  t \in [1,\infty)  $) and in the $  L^\infty  $ norm. The quantity $C_{s+dn}^{dn}$ appearing in both theorems is a binomial coefficient, whose value is given by $\frac{(s+dn)!}{(dn)!s!}$. Our results include a quantitative characterization of the Transformer structure. Comparing the two theorems shows that achieving pointwise approximation ($  L^\infty  $ approximation) comes at the cost of a larger Transformer size.

\begin{theorem}\label{Lr approximation}
Let $1\leq t<\infty$. For any $0<\epsilon<1$ and any $\boldsymbol{f}:[0,1]^{d\times n}\to\mathbb{R}^{d\times n}$ with components in H\"older space $C^{s,\lambda}\left([0,1]^{d\times n}\right)$, there exists a Transformer $\boldsymbol{T}:\mathbb{R}^{d\times n}\to\mathbb{R}^{d\times n}$
with size
\begin{align*}
&L_0=4,\quad W_0=\mathcal{O}\left(\epsilon^{-1/\gamma}\right);\\
&H_1=(3d+1)C_{s+dn}^{dn},\quad S_1=\max\{dn,3\};\\
&L_l=3,\quad W_l=dC_{s+dn}^{dn}(4n+11);\quad l=1,\cdots,n-1;\\
&H_l=3dC_{s+dn}^{dn},\quad S_l=3;\quad l=2,\cdots,n;\\
&L_n=\mathcal{O}\left(\log\frac{1}{\epsilon}\right),\quad W_n=\mathcal{O}\left(\epsilon^{-dn/{\gamma}}\right)
\end{align*}
and dimension vector 
\begin{align*}
\begin{pmatrix}
d&d+n+1&(2dn+5d)C_{s+dn}^{dn}\cdot\boldsymbol{1}_{1\times n}&d
\end{pmatrix}
\end{align*}
such that for $p\in[d],q\in[n]$,
\begin{align*}
\left\|T_{pq}-f_{pq}\right\|_{L^t([0,1]^{d\times n})}
&\leq\epsilon.
\end{align*}
Furthermore, the weight bounds of $\boldsymbol{T}$ are $B_{EB}=\mathcal{O}(1),B_{FF}=\mathcal{O}\left(\epsilon^{-(7dn+\gamma+2)t/\gamma}\right),B_{SA}=\mathcal{O}\left(\log\frac{1}{\epsilon}\right)$ and the number of parameters of $\boldsymbol{T}$ are $M_{EB}=\mathcal{O}\left(1\right),M_{FF}=\mathcal{O}\left(\epsilon^{-dn/\gamma}\right),M_{SA}=\mathcal{O}\left(1\right)$.
\end{theorem}

\begin{theorem}\label{Linfty}
For any $0<\epsilon<1$ and any $\boldsymbol{f}:[0,1]^{d\times n}\to\mathbb{R}^{d\times n}$ with components in H\"older space $C^{s,\lambda}\left([0,1]^{d\times n}\right)$, there exists a Transformer $\boldsymbol{T}:\mathbb{R}^{d\times n}\to\mathbb{R}^{d\times n}$
with size
\begin{align*}
&L_0=4,\quad W_0=\mathcal{O}\left(\epsilon^{-1/\gamma}\right);\\
&H_1=(3d+1)3^{dn}C_{s+dn}^{dn},\quad S_1=\max\{dn,3\};\\
&L_l=3,\quad W_l=d3^{dn}C_{s+dn}^{dn}(4n+11);\quad l=1,\cdots,n-1;\\
&H_l=d3^{dn+1}C_{s+dn}^{dn},\quad S_l=3;\quad l=2,\cdots,n;\\
&L_n=\mathcal{O}\left(\log\frac{1}{\epsilon}\right),\quad W_n=\mathcal{O}\left(\epsilon^{-dn/\gamma}\right)
\end{align*}
and dimension vector 
\begin{align*}
\begin{pmatrix}
d&d+n+1+3^{n}&(2dn+5d)3^{dn}C_{s+dn}^{dn}\cdot\boldsymbol{1}_{1\times n}&d
\end{pmatrix}
\end{align*}
such that for $p\in[d],q\in[n]$,
\begin{align*}
\left\|T_{pq}-f_{pq}\right\|_{L^{\infty}([0,1]^{d\times n})}\leq\epsilon.
\end{align*}
Furthermore, the weight bounds of $\boldsymbol{T}$ are $B_{EB}=\mathcal{O}(1),B_{FF}=\mathcal{O}\left(\epsilon^{-\max\{(6dn+2)/\gamma,1/\lambda\}}\right),B_{SA}=\mathcal{O}\left(\log\frac{1}{\epsilon}\right)$ and the number of parameters of $\boldsymbol{T}$ are $M_{EB}=\mathcal{O}\left(1\right),M_{FF}=\mathcal{O}\left(\epsilon^{-dn/\gamma}\right),M_{SA}=\mathcal{O}\left(1\right)$.
\end{theorem}

We next consider the application of Transformers to nonparametric regression:
\begin{align*}
{Y}_i = f_0(\boldsymbol{X}_i) + \xi_i,\quad i\in[m].    
\end{align*}
Here $f_0(\boldsymbol{x})=\mathbb{E}[Y|\boldsymbol{X}=\boldsymbol{x}]:[0,1]^{d\times n}\to\mathbb{R}$ is the unknown target function, $\{(\boldsymbol{X}_i,Y_i)\}_{i=1}^m\subset[0,1]^{d\times n}\times\mathbb{R}$ are observation pairs, $\{\xi\}_{i=1}^m$ are i.i.d. Gaussian noises with $\mathbb{E}\xi_i=0,\mathrm{Var}(\xi_i)=\sigma^2$. Let $\mu$ be the marginal distribution of $\boldsymbol{X}$. Our goal is to estimate $f_0$ based on the given observation pairs $\{(\boldsymbol{X}_i,Y_i)\}_{i=1}^m$. Specifically, we consider the following least square problem over a function class $\mathcal{F}$:
\begin{align*}
\widehat{f}=\arg\min_{f\in\mathcal{F}}\frac{1}{m}\sum_{i=1}^m[f(X_i)-Y_i]^2.
\end{align*}
The function class $\mathcal{F}$ can be selected in various forms, such as reproducing kernel Hilbert spaces, polynomial spaces, spline spaces, and neural network function classes. In this paper, for some $B\in\mathbb{R}_{>0}$, we choose 
\begin{align}\label{truncation operator}
\mathcal{F}:=f_{tr,B}\circ\mathcal{F}_{\mathcal{T}}
\end{align}
with $f_{tr,B}(x):=\max\{\min\{x,B\},-B\}$ being a truncation function and $\mathcal{F}_{\mathcal{T}}$ being a function class generated by Transformers:

\begin{align}
&\mathcal{F}_{\mathcal{T}}(K,L,W,H,S,M_{FF},M_{SA},B_{FF},B_{SA},\boldsymbol{d}):=\nonumber\\
&\left\{f_{\boldsymbol{T}}:f_{\boldsymbol{T}}=\left\langle\boldsymbol{T},\boldsymbol{E}_{11}\right\rangle,\boldsymbol{T}\in \mathcal{T}(K,L,W,H,S,M_{FF},M_{SA},B_{FF},B_{SA},\boldsymbol{d})\right\},\label{inner Transformer calss}
\end{align}
where
\begin{align*}
&\mathcal{T}(K,L,W,H,S,M_{FF},M_{SA},B_{FF},B_{SA},\boldsymbol{d}):=\\
&\left\{\boldsymbol{T}: \boldsymbol{T} \text{ is a Transformer defined by \eqref{Transformer} with }\right.\\
&\left.\ L:=\max_{k=0,1,\cdots,K}L_k,\ W:=\max_{k=0,1,\cdots,K}W_k,\ H:=\max_{k=1,\cdots,K}H_k,\ S:=\max_{k=1,\cdots,K}S_k\right\}.
\end{align*}
Here $\left\langle\cdot,\cdot\right\rangle$ is the matrix inner product and $\boldsymbol{E}_{11}$ is the matrix that the entry in the first row and first column is $1$, while all other entries are $0$. $\boldsymbol{E}_{11}$ can be replaced by any other fixed matrix, as there is no essential difference here. Our next theorem provides an upper bound of the excess risk $\left\|\widehat{f}-f_{0}\right\|_{L^{2}(\mu)}^2:=\int_{[0,1]^{d\times n}}\left|\widehat{f}-f_{0}\right|^2d\mu$, which examines the distance between the estimator $\widehat{f}$ and the target $f_0$.
\begin{theorem}\label{regression bound}
Assume $f_0\in C^{s,\lambda}\left([0,1]^{d\times n}\right)$. Choose $B=\|f_0\|_{C^{s,\lambda}\left([0,1]^{d\times n}\right)}$ in \eqref{truncation operator}. Let the parameters in \eqref{inner Transformer calss} be
\begin{gather*}
K=n,\ L=\mathcal{O}\left(\log m\right),\ W=\mathcal{O}\left(m^{dn/(2\gamma+dn)}\right),\ H=\mathcal{O}(1),\ S=\mathcal{O}(1),\\
M_{EB}=\mathcal{O}\left(1\right),M_{FF}=\mathcal{O}\left(m^{dn/(2\gamma+dn)}\right),M_{SA}=\mathcal{O}\left(1\right),\\
B_{EB}=\mathcal{O}\left(1\right),B_{FF}=\mathcal{O}\left(m^{\max\{6dn+2,\gamma/\lambda\}/(2\gamma+dn)}\right),B_{SA}=\mathcal{O}\left(\log m\right),\\
\boldsymbol{d}=
\begin{pmatrix}
d&d+n+1+3^{n}&(2dn+5d)3^{dn}C_{s+dn}^{dn}\cdot\boldsymbol{1}_{1\times n}&d
\end{pmatrix}.
\end{gather*}
Then with probability at least $1 - 2\exp \left( -m^{dn/(2\gamma+dn)} \right)$, there holds
\begin{align*}
\left\|\widehat{f}-f_{0}\right\|_{L^{2}(\mu)}^{2}
\lesssim m^{-2\gamma/(2\gamma+dn)}(\log m)^2.
\end{align*}
\end{theorem}
As shown in the classical work \cite{stone1980optimal}, the minimax optimal convergence rate for nonparametric regression with $  C^{s,\lambda}  $ targets is $  \Theta(m^{-2\gamma/(2\gamma+dn)})  $. Therefore, ignoring logarithmic factors, our result shows that Transformers achieve minimax optimal rate in the nonparametric regression.

\subsection{Our Contributions}

Summarizing the content of the previous section, our contributions are as follows:

\begin{itemize}
\item To the best of our knowledge, this is the first work to derive the approximation rates of standard Transformers for functions in Hölder spaces $C^{s,\lambda}\left([0,1]^{d\times n}\right)$ with $s\in\mathbb{N}_{\geq0}$ (Theorems \ref{Lr approximation} and \ref{Linfty}). Our results generalize those of \cite{jiao2025approximation}, which studied approximation rates of standard Transformers for functions in $C^{0,\lambda}\left([0,1]^{d\times n}\right)$. We show that higher smoothness of the target function leads to improved approximation rates for the Transformer. Given the fundamental importance of Hölder spaces in the study of partial differential equations and dynamical systems, our results lay a theoretical foundation for investigating the application of Transformers in these areas.

\item We prove that standard Transformers can achieve the minimax optimal rate under the classical nonparametric regression setting (Theorem \ref{regression bound}). As an intermediate result, we estimate the Lipschitz constant of Transformers (Lemma \ref{covering number bound}), a result that may be of independent interest.

\item As an intermediate step during the derivation of the approximation results, we obtain a memorization result for the standard Transformer (Lemma \ref{Transformer memorization}). This improves upon the results in \cite{kim2023provable}, which investigated Transformers with bias terms in the attention layers. This finding may be of independent interest.

\item Previous works either lacked a precise characterization of the Transformer architecture or provide only a coarse structural description. In contrast, by introducing two metrics: the size tuple and the dimension vector, we achieve a fine-grained characterization of Transformer structures. This facilitates the analysis of both generalization error and optimization error for Transformers with varying structures in future work.

\end{itemize}

\subsection{Organization of This Paper}

The remainder of the paper is organized as follows: In Section \ref{Proofs of Main Results}, we present the proofs of Theorems \ref{Lr approximation}, \ref{Linfty} and \ref{regression bound} along with the intermediate results needed in the proofs. Sections \ref{proof of basic approximation}–\ref{proof of covering number bound} contain the proofs of these intermediate results. Finally, we summarize the paper in Section \ref{Conclusions}.

\section{Proofs of Main Results}\label{Proofs of Main Results}

In this section, we present the proofs of Theorems \ref{Lr approximation}, \ref{Linfty} and \ref{regression bound} in Sections \ref{proof of thm1}, \ref{proof of thm2} and \ref{proof of thm3}, respectively, along with the intermediate results required for these proofs.

\subsection{Proof of Theorem \ref{Lr approximation}}\label{proof of thm1}

To prove Theorem \ref{Lr approximation}, we first construct a Transformer that can approximate the target function to arbitrary precision over the regions
\begin{align*}
\Omega_{\boldsymbol{\beta}}:=\left\{\boldsymbol{X}\in[0,1]^{d\times n}:X_{ik}\in\left[\frac{\beta_{ik}}{K},\frac{\beta_{ik}+1-\delta}{K}\right),i\in[d],k\in[n]\right\}, \quad\boldsymbol{\beta}\in\{0,1,\cdots,K-1\}^{d\times n},
\end{align*}
where $K\in\mathbb{N}_{\geq1}$ and $\delta\in\mathbb{R}_{>0}$. We sort the index set $\{0,1,\cdots,K-1\}^{d\times n}$ as $\left\{\boldsymbol{\beta}_1,\cdots,\boldsymbol{\beta}_{K^{dn}}\right\}$ and denote $\Omega_{\boldsymbol{\beta}_j}$ simply as $\Omega_j$ for 
$j\in\left[K^{dn}\right]$ hereafter. The complement of $\bigcup_{j\in\left[K^{dn}\right]}\Omega_j$ in $[0,1]^{d\times n}$ is
\begin{align*}
\Omega^{(flaw)}:=\left\{\boldsymbol{X}\in[0,1]^{d\times n}:\text{there exist }i\in[d],j\in[n],k\in[K]\text{ such that }X_{ij}\in\left[\frac{k-\delta}{K},\frac{k}{K}\right)\right\}.
\end{align*}

Such a Transformer is given in Proposition \ref{basic approximation} below. We provide a detailed characterization of its structure and derive an upper bound on its magnitude on $\Omega^{(flaw)}$. The proof of Proposition \ref{basic approximation} is deferred to Section \ref{proof of basic approximation}.

\begin{proposition}\label{basic approximation}
For any $0<\epsilon<1$ and any $\boldsymbol{f}:[0,1]^{d\times n}\to\mathbb{R}^{d\times n}$ with components in H\"older space $C^{s,\lambda}\left([0,1]^{d\times n}\right)$, there exists a Transformer $\boldsymbol{T}:\mathbb{R}^{d\times n}\to\mathbb{R}^{d\times n}$
with size
\begin{align*}
&L_0=4,\quad W_0=\mathcal{O}\left(\epsilon^{-1/\gamma}\right);\\
&H_1=(3d+1)C_{s+dn}^{dn},\quad S_1=\max\{dn,3\};\\
&L_l=3,\quad W_l=dC_{s+dn}^{dn}(4n+11);\quad l=1,\cdots,n-1;\\
&H_l=3dC_{s+dn}^{dn},\quad S_l=3;\quad l=2,\cdots,n;\\
&L_n=\mathcal{O}\left(\log\frac{1}{\epsilon}\right),\quad W_n=\mathcal{O}\left(\epsilon^{-dn/\gamma}\right)
\end{align*}
and dimension vector 
\begin{align*}
\begin{pmatrix}
d&d+n+1&(2dn+5d)C_{s+dn}^{dn}\cdot\boldsymbol{1}_{1\times n}&d
\end{pmatrix}
\end{align*}
such that for $p\in[d],q\in[n]$,
\begin{itemize}
\item for $\boldsymbol{X}\in\bigcup_{j\in[K^{dn}]}\Omega_j$,
\begin{align*}
|T_{pq}(\boldsymbol{X})-f_{pq}(\boldsymbol{X})|\leq\epsilon;
\end{align*}
\item for $\boldsymbol{X}\in\Omega^{(flaw)}$,
\begin{align*}
\left|T_{pq}(\boldsymbol{X})\right|\lesssim \epsilon^{-(7dn+2)/\gamma}.
\end{align*}
\end{itemize}
Here the granularity $K=\Theta\left(\epsilon^{-1/\gamma}\right)$. Furthermore, the weight bounds of $\boldsymbol{T}$ are $B_{EB}=\mathcal{O}\left(1\right),B_{FF}=\max\left\{C(\boldsymbol{f},d,n,s)\epsilon^{-(6dn+2)/\gamma},1/\delta\right\},B_{SA}=\mathcal{O}\left(\log\frac{1}{\epsilon}\right)$ and the number of parameters of $\boldsymbol{T}$ are $M_{EB}=\mathcal{O}\left(1\right),M_{FF}=\mathcal{O}\left(\epsilon^{-dn/\gamma}\right),M_{SA}=\mathcal{O}\left(1\right)$.
\end{proposition}
With Proposition \ref{basic approximation} in hand, we are able to prove Theorem \ref{Lr approximation}.
\begin{proof}[Proof of Theorem \ref{Lr approximation}]
We divide the integral into two parts:
\begin{align*}
&\left\|T_{pq}-f_{pq}\right\|_{L^t([0,1]^{d\times n})}^t=\\
&\int_{\boldsymbol{X}\in\bigcup_{j\in\left[K^{dn}\right]}\Omega_{j}}|T_{pq}(\boldsymbol{X})-f_{pq}(\boldsymbol{X})|^td\boldsymbol{X}
+\int_{\boldsymbol{X}\in\Omega^{(flaw)}}|T_{pq}(\boldsymbol{X})-f_{pq}(\boldsymbol{X})|^td\boldsymbol{X}.
\end{align*}
For the region $\bigcup_{j\in\left[K^{dn}\right]}\Omega_{j}$, we apply Proposition \ref{basic approximation} by replacing $\epsilon$ with $\frac{\epsilon}{2^{1/t}(1-\delta)^{dn/t}}$ therein and obtain
\begin{align*}
\int_{\boldsymbol{X}\in\bigcup_{j\in\left[K^{dn}\right]}\Omega_{j}}|T_{pq}(\boldsymbol{X})-f_{pq}(\boldsymbol{X})|^td\boldsymbol{X}&\leq\frac{\epsilon^t}{2(1-\delta)^{dn}}\cdot\sum_{j=1}^{K^{dn}}|\Omega_j|\\
&=\frac{\epsilon^t}{2(1-\delta)^{dn}}\cdot K^{dn}\left(\frac{1-\delta}{K}\right)^{dn}=\frac{\epsilon^t}{2}.
\end{align*}
According to Bernoulli's inequality, we have the following estimate for the measure of $\Omega^{(flaw)}$:
\begin{align*}
\left|\Omega^{(flaw)}\right|=1-\sum_{j=1}^{K^{dn}}|\Omega_j|=1-K^{dn}\left(\frac{1-\delta}{K}\right)^{dn}
=1-(1-\delta)^{dn}\leq dn\delta.
\end{align*}
Using the above estimate and Proposition \ref{basic approximation}, we derive that
\begin{align*}
\int_{\boldsymbol{X}\in\Omega^{(flaw)}}|T_{pq}(\boldsymbol{X})-f_{pq}(\boldsymbol{X})|^td\boldsymbol{X}
&\lesssim \delta(1-\delta)^{dn(7dn+2)/\gamma}\epsilon^{-(7dn+2)t/\gamma}\lesssim \delta\epsilon^{-(7dn+2)t/\gamma}.
\end{align*}
Choosing
\begin{align*}
\delta\asymp\epsilon^{(7dn+\gamma+2)t/\gamma},
\end{align*}
we can make
\begin{align*}
\int_{\boldsymbol{X}\in\Omega^{(flaw)}}|T_{pq}(\boldsymbol{X})-f_{pq}(\boldsymbol{X})|^td\boldsymbol{X}
&\leq \frac{\epsilon^t}{2}.
\end{align*}
It follows that
\begin{align*}
\left\|T_{pq}-f_{pq}\right\|_{L^t([0,1]^{d\times n})}
&\leq\epsilon.
\end{align*}

\end{proof}

\subsection{Proof of Theorem \ref{Linfty}}\label{proof of thm2}

The proof of Theorem \ref{Linfty} relies on the horizontal shift technique, first proposed by \cite{lu2021deep} to prove the approximation capabilities of ReLU feedforward neural networks. Specifically, by exploiting properties of the middle value function, this technique strengthens approximation results on $  [0,1]^{d\times n} \setminus \Omega^{({flaw})}  $ to the full domain $  [0,1]^{d\times n}  $, at the cost of a larger network size. The technique is presented in the two lemmas below, where $  \omega_f(\delta)  $ denotes the modulus of continuity of $  f \in C([0,1]^{d\times n})  $, defiend as

\[
\omega_f(\delta) := \sup \bigl\{ |f(\boldsymbol{X}) - f(\boldsymbol{Y})| : \|\boldsymbol{X} - \boldsymbol{Y}\|_F \leq \delta, \; \boldsymbol{X}, \boldsymbol{Y} \in [0,1]^{d\times n} \bigr\}, \qquad \text{for any } \delta \geq 0.
\]

\begin{lemma}[\cite{lu2021deep}, Lemma 3.1]\label{middle}
There exists a ReLU FNN function \(f_{FF}^{(mid)}:\mathbb{R}^3\to\mathbb{R}\) with width \(14\), depth \(2\) and weight bound $1$ that outputs the middle value of the components of $\boldsymbol{x}$ for any $\boldsymbol{x}\in\mathbb{R}^3$.
\end{lemma}

\begin{lemma}[\cite{lu2021deep}, Lemma 3.4]\label{horizontal shift}
Let $K\in\mathbb{N}_{>0},\epsilon,\delta\in\mathbb{R}_{>0}$. Suppose that $\delta\leq\frac{1}{3K}$. Assume the target function \(f \in C([0, 1]^{d\times n})\) and there exists \(g: \mathbb{R}^{d\times n} \to \mathbb{R}\) such that for any $\boldsymbol{X}\in[0,1]^{d\times n}\setminus\Omega^{(flaw)}$,
\[|g(\boldsymbol{X}) - f(\boldsymbol{X})| \leq \epsilon.\]  
Then for any \(\boldsymbol{X} \in [0, 1]^{d\times n}\), 
\[|\phi(\boldsymbol{X}) - f(\boldsymbol{X})| \leq \epsilon + dn \cdot \omega_f(\delta),\]  
where \(\phi := \phi_{dn}\) is defined by induction through  
\[\phi_{i}(\boldsymbol{X}) := f_{FF}^{(mid)}\bigl( \phi_{i-1}(\boldsymbol{X} - \delta \boldsymbol{E}_{i}), \phi_{i-1}(\boldsymbol{X}), \phi_{i-1}(\boldsymbol{X} + \delta \boldsymbol{E}_{i})\bigr),\]  
for \(i \in [dn]\),  
with \(\phi_0=g\) and $\boldsymbol{E}_i\in\mathbb{R}^{d\times n}$ being the matrix whose entry is 1 at the position $\left(\lceil i/n\rceil,i-n(\lceil i/n\rceil-1)\right)$ and $0$ elsewhere.
\end{lemma}

In this paper, we extend the horizontal shift technique from FNNs to Transformers, thereby establishing approximation in the $  L^\infty  $ norm. This extension relies on the parallelizability of Transformers, which is formalized in Lemma \ref{parallel}. This lemma will also be used frequently in the later technical proofs. Its proof is deferred to Section \ref{proof of lemma parallel}.

\begin{lemma}\label{parallel}

\begin{enumerate}
\item[(1)]
Let $\boldsymbol{\mathcal{F}}_{FF}^{(1)}:\mathbb{R}^{d^{(1)}\times n}\to\mathbb{R}^{\bar{d}^{(1)}\times n}$ and $\boldsymbol{\mathcal{F}}_{FF}^{(2)}:\mathbb{R}^{d^{(2)}\times n}\to\mathbb{R}^{\bar{d}^{(2)}\times n}$ be two feedforward blocks, both with depth \( L \), widths \( W^{(1)} \) and \( W^{(2)} \), weight bounds \( B^{(1)} \) and \( B^{(2)} \), respectively. There exists a feedforward block $\boldsymbol{\mathcal{F}}_{FF}^{(prl)}:\mathbb{R}^{(d^{(1)}+d^{(2)})\times n}\to\mathbb{R}^{(\bar{d}^{(1)}+\bar{d}^{(2)})\times n}$ with depth $L$, width not greater then $W^{(1)}+W^{(2)}$ and weight bound $\max\{B^{(1)},B^{(2)}\}$ such that for any $\boldsymbol{X}\in\mathbb{R}^{d^{(1)}\times n},\boldsymbol{Y}\in\mathbb{R}^{d^{(2)}\times n}$,
\begin{align*}
\boldsymbol{\mathcal{F}}_{FF}^{(prl)}\left(\begin{pmatrix}
\boldsymbol{X} \\
\boldsymbol{Y}
\end{pmatrix}\right)&
=\begin{pmatrix}
\boldsymbol{\mathcal{F}}_{FF}^{(1)}(\boldsymbol{X})\\
\boldsymbol{\mathcal{F}}_{FF}^{(2)}(\boldsymbol{Y})
\end{pmatrix}.
\end{align*}

\item[(2)] 
Let $\boldsymbol{\mathcal{F}}_{SA}^{(1)}:\mathbb{R}^{d^{(1)}\times n}\to\mathbb{R}^{{d}^{(1)}\times n}$ and $\boldsymbol{\mathcal{F}}_{SA}^{(2)}:\mathbb{R}^{d^{(2)}\times n}\to\mathbb{R}^{{d}^{(2)}\times n}$ be two self-attention layers with head numbers $H^{(1)}$ and $H^{(2)}$, head sizes $S^{(1)}$ and $S^{(2)}$, weight bounds \( B^{(1)} \) and \( B^{(2)} \), respctively. There exists a self-attention layer $\boldsymbol{\mathcal{F}}_{SA}^{(prl)}:\mathbb{R}^{(d^{(1)}+d^{(2)})\times n}\to\mathbb{R}^{({d}^{(1)}+{d}^{(2)})\times n}$ with head number $H^{(1)}+H^{(2)}$, head size $\max\{S^{(1)},S^{(2)}\}$ and weight bound $\max\{B^{(1)},B^{(2)}\}$ such that for any $\boldsymbol{X}\in\mathbb{R}^{d^{(1)}\times n},\boldsymbol{Y}\in\mathbb{R}^{d^{(2)}\times n}$,
\begin{align*}
\boldsymbol{\mathcal{F}}_{SA}^{(prl)}\left(\begin{pmatrix}
\boldsymbol{X} \\
\boldsymbol{Y}
\end{pmatrix}\right)&
=\begin{pmatrix}
\boldsymbol{\mathcal{F}}_{SA}^{(1)}(\boldsymbol{X})\\
\boldsymbol{\mathcal{F}}_{SA}^{(2)}(\boldsymbol{Y})
\end{pmatrix}.
\end{align*}

\item[(3)]

Let $\boldsymbol{T}^{(1)}:\mathbb{R}^{d_{in}^{(1)}\times n}\to\mathbb{R}^{d_{out}^{(1)}\times n}$ and $\boldsymbol{T}^{(2)}:\mathbb{R}^{d_{in}^{(2)}\times n}\to\mathbb{R}^{d_{out}^{(2)}\times n}$ be two Transformers with sizes
\begin{align*}
&\left\{\left(L_0,W_0^{(1)}\right),\left(H_1^{(1)},S_1^{(1)}\right),\left(L_1,W_1^{(1)}\right),\cdots,\left(H_K^{(1)},S_K^{(1)}\right),\left(L_K,W_K^{(1)}\right)\right\},\\
&\left\{\left(L_0,W_0^{(2)}\right),\left(H_1^{(2)},S_1^{(2)}\right),\left(L_1,W_1^{(2)}\right),\cdots,\left(H_K^{(2)},S_K^{(2)}\right),\left(L_K,W_K^{(2)}\right)\right\},
\end{align*}
dimension vectors
\begin{align*}
&\begin{pmatrix}
d_{in}^{(1)}&d_0^{(1)}&d_1^{(1)}&\cdots&d_{K}^{(1)}&d_{out}^{(1)}
\end{pmatrix},\quad
\begin{pmatrix}
d_{in}^{(2)}&d_0^{(2)}&d_1^{(2)}&\cdots&d_{K}^{(2)}&d_{out}^{(2)}
\end{pmatrix},
\end{align*}
weight bounds $\left(B_{FF}^{(1)},B_{SA}^{(1)}\right),\left(B_{FF}^{(2)},B_{SA}^{(2)}\right)$, respectively. There exists a Transformer $\boldsymbol{T}^{(prl)}:\mathbb{R}^{\left(d_{in}^{(1)}+d_{in}^{(2)}\right)\times n}\to\mathbb{R}^{\left(d_{out}^{(1)}+d_{out}^{(2)}\right)\times n}$ with size
\begin{align*}
&\left\{\left(L_0,W_0^{(1)}+W_0^{(2)}\right),\left(H_1^{(1)}+H_1^{(2)},\max\left\{S_1^{(1)},S_1^{(2)}\right\}\right),\left(L_1,W_1^{(1)}+W_1^{(2)}\right),\cdots,\right.\\
&\left.\ \ \left(H_K^{(1)}+H_K^{(2)},\max\left\{S_K^{(1)},S_K^{(2)}\right\}\right),\left(L_K,W_K^{(1)}+W_K^{(2)}\right)\right\},
\end{align*}
dimension vector
\begin{align*}
&\begin{pmatrix}
d_{in}^{(1)}+d_{in}^{(2)}&d_0^{(1)}+d_0^{(2)}&d_1^{(1)}+d_1^{(2)}&\cdots&d_{K}^{(1)}+d_{K}^{(2)}&d_{out}^{(1)}+d_{out}^{(2)}
\end{pmatrix}
\end{align*}
and weight bounds 
\begin{align*}
B_{FF}=\max\left\{B_{FF}^{(1)},B_{FF}^{(2)}\right\},\quad B_{SA}=\max\left\{B_{SA}^{(1)}, B_{SA}^{(2)}\right\}
\end{align*}
such that for any $\boldsymbol{X}\in\mathbb{R}^{d_{in}^{(1)}\times n},\boldsymbol{Y}\in\mathbb{R}^{d_{in}^{(2)}\times n}$,
\begin{align*}
\boldsymbol{T}^{(prl)}\left(\begin{pmatrix}
\boldsymbol{X} \\
\boldsymbol{Y}
\end{pmatrix}\right)&
=\begin{pmatrix}
\boldsymbol{T}^{(1)}(\boldsymbol{X})\\
\boldsymbol{T}^{(2)}(\boldsymbol{Y})
\end{pmatrix}.
\end{align*}
\end{enumerate}
\end{lemma}

\begin{proof}[Proof of Theorem \ref{Linfty}]

Using Proposition \ref{basic approximation} with $\epsilon$ replaced by $\frac{\epsilon}{2}$ therein, we obtain a Transformer $\boldsymbol{\widetilde{T}}$ that can approximate $\boldsymbol{f}$ entry-wise on $[0,1]^{d\times n}\setminus\Omega^{(flaw)}$. By shifting $\boldsymbol{\widetilde{T}}$ in various ways, where each entry of the shift matrix takes values in $\{-\delta, 0, \delta\}$, we obtain $3^{dn}$ distinct shifted versions of $\boldsymbol{\widetilde{T}}$. The shifting operation can be implemented via the embedding layer. Parallelizing these $3^{dn}$ shifted versions of $\boldsymbol{\widetilde{T}}$ yields a new Transformer $\boldsymbol{T}^{(prl)}$. Then, by composing $\boldsymbol{T}^{(prl)}$ with $dn$ feedforward layers (each consisting of several parallel $\boldsymbol{\mathcal{F}}_{FF}^{(mid)}$ modules), we obtain the desired Transformer $\boldsymbol{T}$.

Note that 
$\omega_{f_{pq}}(\delta)\lesssim\delta^{\lambda}$. Setting $\delta\asymp{\epsilon^{1/\lambda}}$ and applying Proposition \ref{basic approximation}, Lemma \ref{middle}, Lemma \ref{horizontal shift} and Lemma \ref{parallel}, we achieve the result.

\end{proof}

\subsection{Proof of Theorem \ref{regression bound}}\label{proof of thm3}

Based on classical empirical process theory, we can derive an upper bound of the excess risk in terms of the covering number of the function class. The proof of Proposition \ref{general bound} is deferred to Section \ref{proof of proposition general bound}.

\begin{definition}[covering number]
Let $T$ be a set in a metric space $(\bar{T}, \tau)$. An $\epsilon$-cover of $T$ is a subset $T_c\subset \bar{T}$ such  that for each $t\in T$, there exists a $t_c\in T_c$ such that $\tau(t, t_c) \leq\epsilon$. The $\epsilon$-covering number of $T$, denoted as $\mathcal{N}(\epsilon, T,\tau)$ is  defined to be the minimum cardinality among all $\epsilon$-cover of $T$ with respect to the metric $\tau$.
\end{definition}

\begin{proposition}\label{general bound}
Let $B_{\mathcal{F}}\in\mathbb{R}_{>0}$. Let $\mathcal{F}$ be any given function class in $\mathbb{R}^d$ such that $|f|\leq B_{\mathcal{F}}$ for all $f\in\mathcal{F}$. There holds
\begin{align*}
&\left\|\widehat{f}-f_{0}\right\|_{L^{2}(\mu)}^{2}\\
&\leq \left(\frac{896 B_{\mathcal{F}}^2}{3}+2^{17}\sigma^2+20\right)m^{-2\gamma/(2\gamma+d)} + \frac{896 B_{\mathcal{F}}^2 \log \mathcal{N}(m^{-\gamma/(2\gamma+d)}, \mathcal{F}, \|\cdot\|_{L^{\infty}(\mu)})}{3m}\\
&\quad+146\|f^{*} - f_0\|_{L^{\infty}(\mu)}^2 + \frac{2^{10}\sigma}{ m^{(\gamma+d)/(2\gamma+d)}}\left(\int_{0}^{2^7\sigma m^{-\gamma/(2\gamma+d)}} \sqrt{\log 2\mathcal{N}(\varsigma,\mathcal{F},\|\cdot\|_{L^{\infty}(\mu)})^2} \, d\varsigma\right)^2
\end{align*}
with probability at least $1 - 2\exp \left( -m^{d/(2s+d)} \right)$, where $f^*$ can be any function in $\mathcal{F}$.
\end{proposition}

It remains to upper-bound the covering number of Transformer class, which is done in Lemma \ref{covering number bound} with proof deferred to Section \ref{proof of covering number bound}.
\begin{lemma}\label{covering number bound}
For $\mathcal{F}_{\mathcal{T}}$ defined in \eqref{inner Transformer calss}, there holds
\begin{align*}
\mathcal{N}(\varsigma,\mathcal{F}_{\mathcal{T}},\|\cdot\|_{L^\infty(\mu)})\leq\left(\frac{2B_{EB}\mathfrak{L}}{\varsigma}\right)^{M_{EB}}\left(\frac{2B_{FF}\mathfrak{L}}{\varsigma}\right)^{M_{FF}}\left(\frac{2B_{SA}\mathfrak{L}}{\varsigma}\right)^{M_{SA}},\quad\forall \varsigma>0,
\end{align*}
where
\begin{align}\label{Lip-1}
\mathfrak{L}=&(2K+2)6^{K}4^{K^2+K+4}n^{K^2+5K/2+3}d_{in}^{K+1/2} d_0^{2K+1}\left(\prod_{k'=1}^{K}d_{k'}^{4(K-k')+6}\right)d_{out}^{1/2}\nonumber\\
&H^{K^2+K-1}S^{K^2+2K+1}L^{K^2+2K+3}W^{(L-1)(K^2+3K+3)}B_{EB}^{2K+1}B_{FF}^{L(K^2+3K+3)}B_{SA}^{2(K^2+2K+1)}.
\end{align}
\end{lemma}
We also need the Lipschitz continuity of the truncation function.
\begin{lemma}\label{trunction}
Let $B\in\mathbb{R}_{>0}$. Let $U$ be a set. For any two functions $f_1,f_2: U\to\mathbb{R}$,
\begin{align*}
|f_{tr,B}\circ f_1(x)-f_{tr,B}\circ f_2(x)|\leq |f_1(x)-f_2(x)|,\quad x\in U.
\end{align*}
\end{lemma}
\begin{proof}
Given $x\in U$, assume without loss of generality that $f_1(x)\geq f_2(x)$. If $f_1(x),f_2(x)\geq B$ or $-B\leq f_1(x),f_2(x)\leq B$ or $f_1(x),f_2(x)\leq -B$, then $f_{tr,B}\circ f_1(x)-f_{tr,B}\circ f_2(x)=0$. If $f_1(x)\geq B,-B\leq f_2(x)\leq B$, then $f_{tr,B}\circ f_1(x)=B,f_{tr,B}\circ f_2(x)=f_2(x)$, thereby 
\begin{align*}
0\leq f_{tr,B}\circ f_1(x)-f_{tr,B}\circ f_2(x)=B-f_2(x)\leq f_1(x)-f_2(x).
\end{align*}
If $f_1(x)>B,f_2(x)<-B$, then $f_{tr,B}\circ f_1(x)=B,f_{tr,B}\circ f_2(x)=-B$, thereby
\begin{align*}
0\leq f_{tr,B}\circ f_1(x)-f_{tr,B}\circ f_2(x)=2B\leq f_1(x)-f_2(x).
\end{align*}
If $-B\leq f_1(x)\leq B,f_2(x)\leq-B$, then $f_{tr,B}\circ f_1(x)=f_1(x),f_{tr,B}\circ f_2(x)=-B$, thereby
\begin{align*}
0\leq f_{tr,B}\circ f_1(x)-f_{tr,B}\circ f_2(x)=f_1(x)+B\leq f_1(x)-f_2(x).
\end{align*}

\end{proof}

\begin{proof}[Proof of Theorem \ref{regression bound}]

By Theorem \ref{Linfty}, for a given $0<\epsilon<1$, there exists  $f:\mathbb{R}^{d\times n}\to\mathbb{R}$, which lies in
\begin{align}
\mathcal{F}_{\mathcal{T}}&\left(K=n,L=\mathcal{O}\left(\log\frac{1}{\epsilon}\right),W=\mathcal{O}\left(\epsilon^{-dn/\gamma}\right),H=\mathcal{O}(1),S=\mathcal{O}(1),\right.\nonumber\\
&\quad \left.M_{EB}=\mathcal{O}\left(1\right),M_{FF}=\mathcal{O}\left(\epsilon^{-dn/\gamma}\right),M_{SA}=\mathcal{O}\left(1\right),\right.\nonumber\\
&\quad \left.B_{EB}=\mathcal{O}\left(1\right),B_{FF}=\mathcal{O}\left(\epsilon^{-\max\{(6dn+2)/\gamma,1/\lambda\}}\right),B_{SA}=\mathcal{O}\left(\log\frac{1}{\epsilon}\right),\boldsymbol{d}\right)\label{regression1}
\end{align}
with
\begin{align*}
\boldsymbol{d}=
\begin{pmatrix}
d&d+n+1+3^{n}&(2dn+5d)3^{dn}C_{s+dn}^{dn}\cdot\boldsymbol{1}_{1\times n}&d
\end{pmatrix},
\end{align*}
such that
\begin{align*}
\left\|f-f_0\right\|_{L^{\infty}([0,1]^{d\times n})}\leq\epsilon.
\end{align*}
It follows from Lemma \ref{trunction} that
\begin{align}\label{regression2}
\left\|f_{tr,B_Y}\circ f-f_0\right\|_{L^{\infty}([0,1]^{d\times n})}\leq\epsilon.
\end{align}
By Lemma \ref{covering number bound} and some calculations, the entropy of the function class \eqref{regression1} can be bounded as

\begin{align*}
\log\mathcal{N}(\varsigma,\mathcal{F}_{\mathcal{T}},\|\cdot\|_{L^\infty(\mu)})
&\lesssim \epsilon^{-dn/\gamma}\left(\log\frac{1}{\epsilon}\right)^2+\epsilon^{-dn/\gamma}\log\frac{1}{\varsigma},\quad\forall \varsigma>0.
\end{align*}
Truncation does not increase entropy:
\begin{align}
\log\mathcal{N}(\varsigma,f_{tr,B_{Y}}\circ\mathcal{F}_{\mathcal{T}},\|\cdot\|_{L^\infty(\mu)})
&\leq\log\mathcal{N}(\varsigma,\mathcal{F}_{\mathcal{T}},\|\cdot\|_{L^\infty(\mu)})\nonumber\\
&\lesssim \epsilon^{-dn/\gamma}\left(\log\frac{1}{\epsilon}\right)^2+\epsilon^{-dn/\gamma}\log\frac{1}{\varsigma}.\label{regression3}
\end{align}
We can then apply Proposition \ref{general bound} with $f^*=f_{tr,B_Y}\circ f$ to derive an upper bound for $\left\|\widehat{f}-f_{0}\right\|_{L^{2}(\mu)}$. To this end, we need to evaluate the integral in Proposition \ref{general bound}:
\begin{align}
&\left(\int_{0}^{2^7\sigma m^{-\gamma/(2\gamma+dn)}} \sqrt{\log 2\mathcal{N}(\varsigma,f_{tr,B_{Y}}\circ\mathcal{F}_{\mathcal{T}},\|\cdot\|_{L^{\infty}(\mu)})^2} \, d\varsigma\right)^2\nonumber\\
&\leq2^7\sigma m^{-\gamma/(2\gamma+dn)}\int_{0}^{2^7\sigma m^{-\gamma/(2\gamma+dn)}} \log 2\mathcal{N}(\varsigma,f_{tr,B_{Y}}\circ\mathcal{F}_{\mathcal{T}},\|\cdot\|_{L^{\infty}(\mu)})^2 \, d\varsigma\nonumber\\
&\lesssim m^{-\gamma/(2\gamma+dn)}\int_{0}^{2^7\sigma m^{-\gamma/(2\gamma+dn)}} \left(\epsilon^{-dn/\gamma}\left(\log\frac{1}{\epsilon}\right)^2+\epsilon^{-dn/\gamma}\log\frac{1}{\varsigma}\right) \, d\varsigma\nonumber\\
&\lesssim m^{-2\gamma/(2\gamma+dn)}\epsilon^{-dn/\gamma}\left(\left(\log\frac{1}{\epsilon}\right)^2+\log m\right).\label{regression4}
\end{align}
Here we use H\"older's inequality. Now, combining Proposition \ref{general bound} and \eqref{regression2}-\eqref{regression4} yields

\begin{align*}
\left\|\widehat{f}-f_{0}\right\|_{L^{2}(\mu)}^{2}
&\lesssim m^{-2\gamma/(2\gamma+dn)} + m^{-1}\epsilon^{-dn/\gamma}\left(\log\frac{1}{\epsilon}\right)^2+\epsilon^{-dn/\gamma}m^{-1}\log m\\
&\quad+\epsilon^2 + m^{-(3\gamma+dn)/(2\gamma+dn)}\epsilon^{-dn/\gamma}\left(\left(\log\frac{1}{\epsilon}\right)^2+\log m\right)
\end{align*}
with probability at least $1 - 2\exp \left( -m^{dn/(2\gamma+dn)} \right)$. We complete the proof by setting
\begin{align*}
\epsilon\asymp m^{-\gamma/(2\gamma+dn)}.
\end{align*}

\end{proof}

\section{Proof of Proposition \ref{basic approximation}: Construction of Transformers}\label{proof of basic approximation}
In Section \ref{proof of proposition basic approximation}, we present the proof of Proposition \ref{basic approximation} along with the technical lemmas required during the proof. Sections \ref{proof of FNN to transformer} and \ref{proof of lemma monomial} contain the proofs of the technical lemmas.

\subsection{Proof of Proposition \ref{basic approximation}}\label{proof of proposition basic approximation}

Inspired by the approximation of highly smooth functions using FNNs \cite{yarotsky2017error,lu2021deep}, the basic idea of proving Proposition \ref{basic approximation} is to construct Transformers to achieve the approximation of Taylor polynomials. Consider the set $\left\{\boldsymbol{\alpha}\in\mathbb{N}_{\geq0}^{d\times n}:\sum_{u=1}^{d}\sum_{v=1}^{n}\alpha_{uv}\leq s\right\}$. Given that its cardinality is $C_{s+dn}^{dn}$, we can rewrite it as $\left\{\boldsymbol{\alpha}_1,\cdots,\boldsymbol{\alpha}_{C_{s+dn}^{dn}}\right\}$. For each $p\in[d]$ and $q\in[n]$, let 
\begin{align*}
P_{j}(f_{pq})=\sum_{i\in\left[C_{s+dn}^{dn}\right]}c_{i,j}(f_{pq})\left(\boldsymbol{X}-\boldsymbol{X}^{(j)}\right)^{\boldsymbol{\alpha}_i}
\end{align*}
be the $s$ degree Taylor polynomial of $f_{pq}$ at the grid point $\boldsymbol{X}^{(j)}:={\boldsymbol{\beta}_j}/{K}$. Here we use the notation $\boldsymbol{X}^{\boldsymbol{\alpha}}:=\prod_{u=1}^{d}\prod_{v=1}^{n}x_{uv}^{\alpha_{uv}}$. For any $\boldsymbol{X}\in\Omega_{j}$ with $j\in\left[K^{dn}\right]$, the standard Taylor remainder estimate gives
\begin{align}\label{Taylor error}
\left|f_{pq}(\boldsymbol{X})-\sum_{i\in\left[C_{s+dn}^{dn}\right]}c_{i,j}(f_{pq})\left(\boldsymbol{X}-\boldsymbol{X}^{(j)}\right)^{\boldsymbol{\alpha}_i}\right|\leq C(\boldsymbol{f},s,d,n)\frac{1}{K^{s+\lambda}}.
\end{align}
Based on this estimation, we can ahieve the approximation of $f_{pq}$ once we achieve the approximation of Taylor polynomials  $P_{j}(f_{pq})$. To this end, we divide our proof into seven steps. In step 1, we use the following lemma to contruct a feedforward block that maps $\Omega_j$ to the grid point $\boldsymbol{X}^{(j)}$.
\begin{lemma}\label{discretization}

There exists a ReLU FNN function $f_{FF}^{(dsc)}:\mathbb{R}\to\mathbb{R}$ with width $K$, depth $3$ and weight bound $1/\delta$ such that for any $x\in\left[\frac{k}{K},\frac{k+1}{K}\right]$ with $k\in\{0,1,\cdots,K-1\}$,
\begin{align*}
f_{FF}^{(dsc)}(x)=\left\{\begin{matrix}
\frac{k}{K},  & x\in\left[\frac{k}{K},\frac{k+1-\delta}{K}\right);\\
\frac{k+1}{K}-\frac{k+1-Kx}{K\delta},  & x\in\left[\frac{k+1-\delta}{K},\frac{k+1}{K}\right].
\end{matrix}\right.
\end{align*}
\end{lemma}
\begin{proof}
We first construct a FNN function that approximate the step function $\mathbbm{1}_{x\geq1}$:
\begin{align*}
g(x):=\sigma_R\left(1-\sigma_R\left(-\frac{x}{\delta}+\frac{1}{\delta}\right)\right)=\left\{\begin{matrix}
0,  & x<1-\delta;\\
1-\frac{1-x}{\delta},  & 1-\delta\leq x\leq1;\\
1,  & x>1.
\end{matrix}\right.
\end{align*}
Then $f_{FF}^{(dsc)}$ is contructed through a summation of a series of $g$ that have undergone translation and scaling transformations:
\begin{align*}
f_{FF}^{(dsc)}(x):=\frac{1}{K}\sum_{k=0}^{K-1}g(Kx-k).
\end{align*}
\end{proof}
In step 2, we construct a Transformer to map the grid point $\boldsymbol{X}_j$ to the Taylor coefficients $c_{i,j}$, which is exactly a memorization task. The following tokenwise separatedness assumption on the input sequences is common in the literature of Transformer memorization \cite{kim2023provable,kajitsuka2024are,kajitsuka2025on}. 
\begin{definition}[Tokenwise separatedness]
Let \( N \in \mathbb{N}_{\geq1} \) and $r,\phi\in\mathbb{R}_{>0}$. Let \( \boldsymbol{X}^{(1)}, \cdots, \boldsymbol{X}^{(N)} \in \mathbb{R}^{d \times n} \) be a set of \( N \) input sequences. Then, we say that \( \left\{\boldsymbol{X}^{(i)}\right\}_{i\in[N]}\) are tokenwise \((r, \phi)\)-separated if the following two conditions are satisfied:
\begin{itemize}
\item For any \( i \in [N] \) and \( k \in [n] \), \( \left\|\boldsymbol{X}_{:,k}^{(i)}\right\|_2 \leq r \) holds.

\item For any \( i, j \in [N] \) and \( k, l \in [n] \), either \( \boldsymbol{X}_{:,k}^{(i)} = \boldsymbol{X}_{:,l}^{(j)} \) or \( \left\|\boldsymbol{X}_{:,k}^{(i)} - \boldsymbol{X}_{:,l}^{(j)}\right\|_2 \geq \phi \) holds.
\end{itemize}
\end{definition}
With this assumption, we are able to construct a Transformer realizing memorization. We provide a detailed characterization of its structure in Lemma \ref{Transformer memorization}. Since this result is of independent interest, its proof is presented in a separate section (Section \ref{proof of transformer memorization}).

\begin{lemma}\label{Transformer memorization}
Let \( N\in\mathbb{N}_{\geq2} \) and $r,\phi,B_y\in\mathbb{R}_{>0}$. Suppose $r>\phi$. For any $N$ data pairs $\left\{\left(\boldsymbol{X}^{(i)},\boldsymbol{Y}^{(i)}\right)\right\}_{i\in[N]}\subset\mathbb{R}^{d\times n}\times [-B_y,B_y]^{1\times n}$ such that the input sequences $\left\{\boldsymbol{X}^{(i)}\right\}_{i\in[N]}$ are distinct and are tokenwise \((r, \phi)\)-separated, there eixsts a Transformer $\boldsymbol{T}^{(mmr)}:\mathbb{R}^{d\times n}\to\mathbb{R}^{1\times n}$ with size $\{(2,\max\{d,5\}),((3,3),(3,11))\times (n-1)\text{ times},(3,3),(3,\max\left\{5,nN-1\right\})\}$ and dimension vector $\begin{pmatrix}
d&d&5\cdot\boldsymbol{1}_{1\times n}&1
\end{pmatrix}$ and a positional encoding matrix 
\begin{align*}
\boldsymbol{E}:=\frac{3r}{\sqrt{d}}\begin{pmatrix}
\boldsymbol{1}_{d\times 1}&2\boldsymbol{1}_{d\times 1}&\cdots&n\boldsymbol{1}_{d\times 1}
\end{pmatrix}\in\mathbb{R}^{d\times n}
\end{align*}
such that
\begin{itemize}
\item 
\begin{align*}
\boldsymbol{T}^{(mmr)}\left(\boldsymbol{X}^{(i)}\right)=\boldsymbol{Y}^{(i)},\quad i\in[N].
\end{align*}
In this case, the labels $\left\{\boldsymbol{Y}^{(i)}\right\}_{i\in[N]}$ need to satisfy the consistency condition: for any $i,j\in[N],k,l\in[n]$, $y_{1,k}^{(i)}=y_{1,l}^{(j)}$ if \( \boldsymbol{X}_{:,k}^{(i)} = \boldsymbol{X}_{:,l}^{(j)} \) and \( \boldsymbol{X}^{(i)} = \boldsymbol{X}^{(j)} \) up to permutations.
\item \begin{align*}
\boldsymbol{T}^{(mmr)}\left(\boldsymbol{X}^{(i)}+\boldsymbol{E}\right)=\boldsymbol{Y}^{(i)},\quad i\in[N].
\end{align*}
In this case, the labels $\left\{\boldsymbol{Y}^{(i)}\right\}_{i\in[N]}$ need not to satisfy the consistency condition.
\end{itemize}
The weight bounds of $\boldsymbol{T}^{(mmr)}$ are $B_{FF}=\max\{R,2B_y\},B_{SA}=\frac{1}{2}\log(3\sqrt{ d} {\pi} n^{4}(3n+1) N^{4} r\phi^{-1})$, where
\begin{align*}
R:=(\sqrt{2\pi d}n^{2} N^{2}(3n+1)r\phi^{-1}+1)\left(\frac{3{\pi}}{4}\sqrt{d}n^{3} N^{4}(3n+1)r\phi^{-1}+\frac{3}{2}\right).
\end{align*}
Furthermore, if $\boldsymbol{X}\in\mathbb{R}^{d\times n}$ satisfies $\left\|\boldsymbol{X}_{:,k}\right\|_2\leq r$ for all $k\in[n]$, then
\begin{align*}
\left|\boldsymbol{T}^{(mmr)}\left(\boldsymbol{X}\right)_{1,k}\right|,\left|\boldsymbol{T}^{(mmr)}\left(\boldsymbol{X}+\boldsymbol{E}\right)_{1,k}\right|\leq[4(nN-1)R+1]B_y 
\end{align*}
for all $k\in[n]$.
\end{lemma}

In step 3, we construct a Transformer to approximate the monomials $\left(\boldsymbol{X}-\boldsymbol{X}^{(j)}\right)^{\boldsymbol{\alpha}_i}$ with the aid of the following two lemmas, whose proofs are deferred to Sections \ref{proof of FNN to transformer} and \ref{proof of lemma monomial}, respectively.
\begin{lemma}\label{FNN to transformer}
For any ReLU FNN $f_{FF}:\mathbb{R}^{dn}\to\mathbb{R}$ with depth $L$, width $W$ and weight bound $B$, there exists a Transformer $\boldsymbol{T}^{(FF)}:\mathbb{R}^{(d+n)\times n}\to\mathbb{R}^{1\times n}$ with size $\{(2,2dn),(1,dn),(L,\max\{W,2dn\})\}$ and dimension vector $\begin{pmatrix}
d+n&d+n&2dn&1
\end{pmatrix}$ such that for any $\boldsymbol{X}\in[0,1]^{d\times n}$, 
\begin{align*}
\boldsymbol{T}^{(FF)}\left(\begin{pmatrix}
\boldsymbol{X}\\
\boldsymbol{I}_{n\times n}-\boldsymbol{1}_{n\times n}
\end{pmatrix}\right)=f_{FF}\left(\boldsymbol{X}^{(flt)}\right)\boldsymbol{1}_{1\times n},
\end{align*}
where $\boldsymbol{X}^{(flt)}\in\mathbb{R}^{dn}$ is the flatten of $\boldsymbol{X}$, defined in the following way:
\begin{align*}
\boldsymbol{X}^{(flt)}:=
\begin{pmatrix}
x_{11}&
\dots&
x_{1n}&
x_{21}&
\dots&
x_{2n}&
\dots&
x_{d1}&
\dots&
x_{dn}
\end{pmatrix}^{\top}.
\end{align*}
Furthermore, $B_{FF}=\max\{B,1\},B_{SA}=n$.
\end{lemma}

\begin{lemma}\label{monomial}
Let $d,s\in\mathbb{N}_{\geq1}$. Let $\boldsymbol{\alpha}\in\mathbb{N}_{\geq0}^d$ and $\sum_{i=1}^{d}\alpha_i=\bar{\alpha}$. For any $0<\epsilon\leq\frac{3^{\lceil\log_2\bar{\alpha}\rceil}-1}{3^{\lceil\log_2\bar{\alpha}\rceil-1}-1}$, there exists a ReLU FNN function $f_{FF}^{(mnm)}:\mathbb{R}^d\to\mathbb{R}$ with width $21\cdot2^{\lceil\log_2\bar{\alpha}\rceil-1}$, depth $C\ln\frac{3^{\lceil\log_2\bar{\alpha}\rceil}-1}{2\epsilon}$ and weight bound $C$ such that for any $\boldsymbol{x}\in[0,1]^d$,
\begin{align*}
\left|f_{FF}^{(mnm)}(\boldsymbol{x})-x_1^{\alpha_1}x_2^{\alpha_2}\cdots x_d^{\alpha_d}\right|\leq\epsilon.
\end{align*} 
\end{lemma}
In step 4, we contruct a Transformer to parallel $c_{i,j}$ and the appromants of $\left(\boldsymbol{X}-\boldsymbol{X}^{(j)}\right)^{\boldsymbol{\alpha}_i}$, followed by contructing a feedforward block that approximates the multiplications of $c_{ij}$ and $\left(\boldsymbol{X}-\boldsymbol{X}^{(j)}\right)^{\boldsymbol{\alpha}_i}$ and sums them over $i\in\left[C_{s+dn}^{dn}\right]$ in step 5. We need the following result on the approximation of multiplication by ReLU FNNs in step 5.
\begin{lemma}[\cite{yarotsky2017error}, Proposition 3]\label{multiplication}
Let $B,B'\in\mathbb{R}_{\geq1}$. For any $\epsilon>0$, there exists a ReLU FNN function $\widetilde{\times}:\mathbb{R}^2\to\mathbb{R}$ with width $21$, depth $C(B)\ln\frac{1}{\epsilon}$ and weight bound $B^2$ such that for any $\boldsymbol{x}\in[-B,B]^2$,
\begin{align*}
\left|\widetilde{\times}(\boldsymbol{x})-x_1x_2\right|\leq\epsilon.
\end{align*}
Furthermore, for any $\boldsymbol{x}\in[-B',B']^2$, $\left|\widetilde{\times}(\boldsymbol{x})\right|\leq\max\{12B^2,4BB'\}$.

\end{lemma}

\begin{remark}
Although \cite{yarotsky2017error} does not explicitly characterize the width, weight bound, and magnitude of $\widetilde{\times}$, these can be readily derived from its proof and arguments.

\end{remark}
Finally, in step 6 and 7, by properly setting the values of parameters, we estimate the approximation error on $\bigcup_{j\in[K^{dn}]}\Omega_j$ and the magnitude of constructed Transformer on $\Omega^{(flaw)}$, resepctively. Figure \ref{Illustration of the proof process} illustrates the flow of the proof for Proposition \ref{basic approximation}.

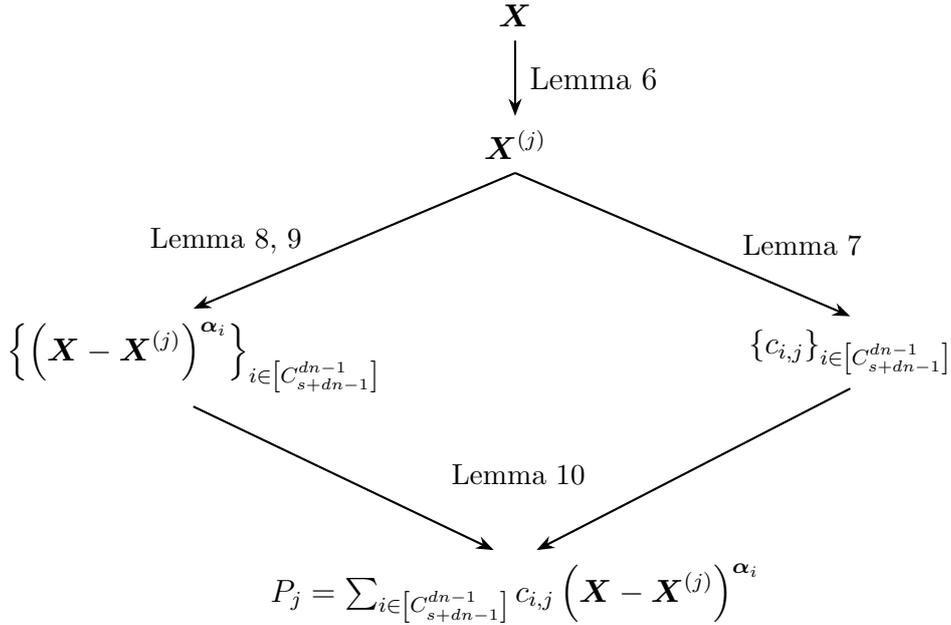
\begin{figure}[h]
    \centering
\begin{tikzpicture}[
    node distance=1.5cm and 2.5cm, 
    every node/.style={inner sep=5pt},
    arrow/.style={-{Stealth}, thick}
]

\node (X) [xshift=0.8cm] {$\boldsymbol{X}$}; 
\node (Xj) [below=1.0cm of X] {$\boldsymbol{X}^{(j)}$};

\node (upper) [below left=1.8cm and 1.0cm of Xj] {$\left\{ \left(\boldsymbol{X}-\boldsymbol{X}^{(j)}\right)^{\boldsymbol{\alpha}_i} \right\}_{i\in\left[C_{s+dn}^{dn}\right]}$};

\node (lower) [below right=1.9cm and 2.3cm of Xj] {$\left\{ c_{i,j} \right\}_{i\in\left[C_{s+dn}^{dn}\right]}$};

\node (final) [below=5.0cm of Xj] {$P_j=\sum_{i\in\left[C_{s+dn}^{dn}\right]} c_{i,j} \left(\boldsymbol{X}-\boldsymbol{X}^{(j)}\right)^{\boldsymbol{\alpha}_i}$};

\draw [arrow] (X) -- node[right] {Lemma \ref{discretization}} (Xj);

\draw [arrow] (Xj.south) -- node[left, font=\small, pos=0.5, xshift=-0.5cm] {Lemma \ref{FNN to transformer}, \ref{monomial}} (upper.north);

\draw [arrow] (Xj.south) -- node[right, font=\small, pos=0.5, xshift=0.6cm] {Lemma \ref{Transformer memorization}} (lower.north);

\draw [arrow] (upper.south) -- ([xshift=-8pt]final.north);
\draw [arrow] (lower.south) -- ([xshift=8pt]final.north);

\node[font=\small] at ($(upper.south)!0.5!(lower.south)!0.5!(final.north)$) {Lemma \ref{multiplication}};

\end{tikzpicture}
\caption{Illustration of the proof process}
\label{Illustration of the proof process}
\end{figure}

In the proof of Proposition \ref{basic approximation}, we also require the following two trivial lemmas, which state that both the feedforward block and the self-attention layer can realize the identity mapping. We may sometimes use these two lemmas without explicit mention, particularly when we employ Lemma \ref{parallel} to parallel two feedforward blocks of different depths.

\begin{lemma}\label{identity}
There exists a ReLU FNN function $f_{FF}^{(idt)}:\mathbb{R}\to\mathbb{R}$ with width $2$, depth $2$ and weight bound $1$ such that for any ${x}\in\mathbb{R}$,
\begin{align*}
f_{FF}^{(idt)}(x)=x.
\end{align*}
\end{lemma}
\begin{proof}
The result follows directly from the property
\begin{align*}
x=\sigma_R(x)-\sigma_R(-x),\quad x\in\mathbb{R}.
\end{align*}
\end{proof}

\begin{lemma}\label{SA-identity}

There exists a self-attention layer $\boldsymbol{\mathcal{F}}_{SA}^{(idt)}:\mathbb{R}^{1\times n}\to\mathbb{R}^{1\times n}$ with head number $1$, head size $1$ and weight bound $1$ such that for any $\boldsymbol{x}\in\mathbb{R}^{1\times n}$,
\begin{align*}
\boldsymbol{\mathcal{F}}_{SA}^{(idt)}\left(\boldsymbol{x}\right)=\boldsymbol{x}.
\end{align*}
\end{lemma}
\begin{proof}
Setting $\boldsymbol{W}_O=0$ yields the result.
\end{proof}

We now present the formal proof of Proposition \ref{basic approximation}.
\begin{proof}[Proof of Proposition \ref{basic approximation}]

\textbf{Step 1: Map $\Omega_j$ to the grid point $\boldsymbol{X}^{(j)}={\boldsymbol{\beta}_j}/{K}$.}

Denote
\begin{align*}
\boldsymbol{W}_{EB}:=\begin{pmatrix}
\boldsymbol{I}_{d\times d}\\
\boldsymbol{0}_{(n+1)\times d}
\end{pmatrix}\in\mathbb{R}^{(d+n+1)\times d},\quad  
\boldsymbol{B}_{EB}:=\begin{pmatrix}
\boldsymbol{0}_{d\times n}\\
\boldsymbol{I}_{n\times n}-\boldsymbol{1}_{n\times n}\\
\boldsymbol{\widetilde{B}}_{EB}
\end{pmatrix}\in\mathbb{R}^{(d+n+1)\times n},
\end{align*}
where
\begin{align*}
\boldsymbol{\widetilde{B}}_{EB}:=\begin{pmatrix}
3&6&\cdots&3n
\end{pmatrix}\in\mathbb{R}^{1\times n}.
\end{align*}
The embedding layer $\boldsymbol{\mathcal{F}}_{EB}:\mathbb{R}^{d\times n}\to\mathbb{R}^{(d+n+1)\times n}$ is defined to be
\begin{align*}
\boldsymbol{\mathcal{F}}_{EB}(\boldsymbol{X}):=\boldsymbol{W}_{EB}\boldsymbol{X}+\boldsymbol{B}_{EB}=\begin{pmatrix}
\boldsymbol{X}\\
\boldsymbol{I}_{n\times n}-\boldsymbol{1}_{n\times n}\\
\boldsymbol{\widetilde{B}}_{EB}
\end{pmatrix}\in\mathbb{R}^{(d+n+1)\times n}.
\end{align*}

Let $\boldsymbol{\mathcal{F}}_{FF}^{(dsc)}:\mathbb{R}^{1\times n}\to\mathbb{R}^{1\times n}$ be the feedforward block generated by the ReLU FNN function $f_{FF}^{(dsc)}$ in Lemma \ref{discretization}. Let $\boldsymbol{\mathcal{F}}_{FF}^{(prl-dsc)}:\mathbb{R}^{(d+n+1)\times n}\to\mathbb{R}^{d\times n}$ be the parallelization of  $\boldsymbol{\mathcal{F}}_{FF}^{(dsc)}$: 
\begin{align*}
\boldsymbol{\mathcal{F}}_{FF}^{(prl-dsc)}(\boldsymbol{Y}):=\begin{pmatrix}
\boldsymbol{\mathcal{F}}_{FF}^{(dsc)}(\boldsymbol{Y}_{1,:})\\
\boldsymbol{\mathcal{F}}_{FF}^{(dsc)}(\boldsymbol{Y}_{2,:})\\
\vdots\\
\boldsymbol{\mathcal{F}}_{FF}^{(dsc)}(\boldsymbol{Y}_{d,:})
\end{pmatrix},\quad\boldsymbol{Y}\in\mathbb{R}^{(d+n+1)\times n}.
\end{align*}
It follows that $\boldsymbol{\mathcal{F}}_{FF}^{(prl-dsc)}$ is of width $\max\{dK,d+n+1\}$, depth $3$ and weight bound $1/\delta$, and
\begin{align*}
\boldsymbol{\mathcal{F}}_{FF}^{(prl-dsc)}\circ\boldsymbol{\mathcal{F}}_{EB}(\boldsymbol{X})=\boldsymbol{X}^{(j)},\quad \boldsymbol{X}\in\Omega_{j}.
\end{align*}

\textbf{Step 2: Map $\boldsymbol{X}_j$ to the Taylor coefficients $c_{i,j}$.}

Applying Lemma \ref{Transformer memorization} with paramters $r=\sqrt{d},\phi=\frac{1}{K},N=K^{dn},B_y=C(\boldsymbol{f})$ therein, we can find Transformers $\left\{\boldsymbol{T}_{i,k}^{(mmr)}\right\}_{i\in[C_{s+dn}^{dn}],k\in[d]}$ with size 
\begin{align*}
L_0&=2,\quad W_0=\max\{d,5\};\\
H_l&=3,\quad S_l=3;\quad l=1,\cdots,n;\\
L_l&=3,\quad W_l=11;\quad l=1,\cdots,n-1;\\
L_n&=3,\quad W_n=\max\left\{5,nK^{dn}-1\right\}
\end{align*}
and dimension vector $\begin{pmatrix}
d&d&5\cdot\boldsymbol{1}_{1\times n}&1
\end{pmatrix}$ such that for $j\in\left[K^{dn}\right]$,
\begin{align*}
\boldsymbol{T}_{i,k}^{(mmr)}\left(\boldsymbol{X}^{(j)}+\boldsymbol{1}_{d\times 1}\boldsymbol{\widetilde{B}}_{EB}\right):=\begin{pmatrix}
c_{i,j}(f_{k1})&c_{i,j}(f_{k2})&\cdots&c_{i,j}(f_{kn})
\end{pmatrix},\quad i\in[C_{s+dn}^{dn}],k\in[d].
\end{align*}
The weight bounds of $\boldsymbol{T}_{i,k}^{(mmr)}$ are $B_{FF}=C(\boldsymbol{f})d^2n^{7}K^{6dn+2},B_{SA}=\frac{1}{2}\log(3{ d} {\pi} n^{4}(3n+1) K^{4dn+1})$. Based on Lemma \ref{parallel}, we can construct a Transformer $\boldsymbol{T}^{(prl-mmr)}:\mathbb{R}^{d\times n}\to\mathbb{R}^{dC_{s+dn}^{dn}\times n}$ with with size

\begin{align*}
L_0&=2,\quad W_0=dC_{s+dn}^{dn}\max\{d,5\};\\
H_l&=3dC_{s+dn}^{dn},\quad S_l=3;\quad l=1,\cdots,n;\\
L_l&=3,\quad W_l=11dC_{s+dn}^{dn};\quad l=1,\cdots,n-1;\\
L_n&=3,\quad W_n=dC_{s+dn}^{dn}\max\left\{5,nK^{dn}-1\right\}
\end{align*}
and dimension vector $\begin{pmatrix}
d&d&5dC_{s+dn}^{dn}\cdot\boldsymbol{1}_{1\times n}&dC_{s+dn}^{dn}
\end{pmatrix}$ such that for $j\in\left[K^{dn}\right]$,
\begin{align*}
\boldsymbol{T}^{(prl-mmr)}\left(\boldsymbol{X}^{(j)}+\boldsymbol{1}_{d\times 1}\boldsymbol{\widetilde{B}}_{EB}\right)
&:=\begin{pmatrix}
\boldsymbol{T}_{1,1}^{(mmr)}\left(\boldsymbol{X}^{(j)}+\boldsymbol{1}_{d\times 1}\boldsymbol{\widetilde{B}}_{EB}\right)\\
\vdots\\
\boldsymbol{T}_{1,d}^{(mmr)}\left(\boldsymbol{X}^{(j)}+\boldsymbol{1}_{d\times 1}\boldsymbol{\widetilde{B}}_{EB}\right)\\
\boldsymbol{T}_{2,1}^{(mmr)}\left(\boldsymbol{X}^{(j)}+\boldsymbol{1}_{d\times 1}\boldsymbol{\widetilde{B}}_{EB}\right)\\
\vdots\\
\boldsymbol{T}_{2,d}^{(mmr)}\left(\boldsymbol{X}^{(j)}+\boldsymbol{1}_{d\times 1}\boldsymbol{\widetilde{B}}_{EB}\right)\\
\vdots\\
\boldsymbol{T}_{C_{s+dn}^{dn},1}^{(mmr)}\left(\boldsymbol{X}^{(j)}+\boldsymbol{1}_{d\times 1}\boldsymbol{\widetilde{B}}_{EB}\right)\\
\vdots\\
\boldsymbol{T}_{C_{s+dn}^{dn},d}^{(mmr)}\left(\boldsymbol{X}^{(j)}+\boldsymbol{1}_{d\times 1}\boldsymbol{\widetilde{B}}_{EB}\right)
\end{pmatrix}\\
&=\begin{pmatrix}
c_{1,j}(f_{11})&c_{1,j}(f_{12})&\cdots&c_{1,j}(f_{1n})\\
\vdots&\vdots&\ddots&\vdots\\
c_{1,j}(f_{d1})&c_{1,j}(f_{d2})&\cdots&c_{1,j}(f_{dn})\\
c_{2,j}(f_{11})&c_{2,j}(f_{12})&\cdots&c_{2,j}(f_{1n})\\
\vdots&\vdots&\ddots&\vdots\\
c_{2,j}(f_{d1})&c_{2,j}(f_{d2})&\cdots&c_{2,j}(f_{dn})\\
\vdots&\vdots&&\vdots\\
c_{C_{s+dn}^{dn},j}(f_{11})&c_{C_{s+dn}^{dn},j}(f_{12})&\cdots&c_{C_{s+dn}^{dn},j}(f_{1n})\\
\vdots&\vdots&\ddots&\vdots\\
c_{C_{s+dn}^{dn},j}(f_{d1})&c_{C_{s+dn}^{dn},j}(f_{d2})&\cdots&c_{C_{s+dn}^{dn},j}(f_{dn})
\end{pmatrix}.
\end{align*}
The weight bounds of $\boldsymbol{T}^{(prl-mmr)}$ are $B_{FF}=C(\boldsymbol{f})d^2n^{7}K^{6dn+2},B_{SA}=\frac{1}{2}\log(3{ d} {\pi} n^{4}(3n+1) K^{4dn+1})$.

\textbf{Step 3: Approximate the monomials $\left(\boldsymbol{X}-\boldsymbol{X}^{(j)}\right)^{\boldsymbol{\alpha}_i}$.}

For $\boldsymbol{X}\in\Omega_j$ with some $j\in\left[K^{dn}\right]$, denote
\begin{align*}
\boldsymbol{\bar{X}}:=\boldsymbol{X}-\boldsymbol{X}^{(j)}.
\end{align*}
Let $0<\epsilon_1\leq\frac{3^{\lceil\log_2(s-1)\rceil}-1}{3^{\lceil\log_2(s-1)\rceil-1}-1}$ be some accuracy to be determined later. By Lemma \ref{monomial}, there exists ReLU FNN functions $\left\{f_{FF,i}^{(mnm)}\right\}_{i\in\left[C_{s+dn}^{dn}\right]}$ with width $21\cdot2^{\lceil\log_2(s-1)\rceil-1}$, depth $C\ln\frac{3^{\lceil\log_2(s-1)\rceil}-1}{2\epsilon_1}$ and weight bound $C$ such that 
\begin{align}\label{Linfty1}
\left|f_{FF,i}^{(mnm)}\left(\boldsymbol{\bar{X}}^{(flt)}\right)-\boldsymbol{\bar{X}}^{\boldsymbol{\alpha}_i}\right|\leq\epsilon_1,\quad i\in\left[C_{s+dn}^{dn}\right].
\end{align} 
By Lemma \ref{FNN to transformer}, there exists Transformers $\left\{\boldsymbol{T}_i^{(mnm)}\right\}_{i\in\left[C_{s+dn}^{dn}\right]}$ with size 
\begin{align*}
\left\{(2,2dn),(1,dn),\left(C\ln\frac{3^{\lceil\log_2(s-1)\rceil}-1}{2\epsilon_1},\max\left\{21\cdot2^{\lceil\log_2(s-1)\rceil-1},2dn\right\}\right)\right\} 
\end{align*}
and dimension vector $\begin{pmatrix}
d+n&d+n&2dn&1
\end{pmatrix}$ such that for $i\in\left[C_{s+dn}^{dn}\right]$, 
\begin{align*}
\boldsymbol{T}_i^{(mnm)}\left(
\begin{pmatrix}
\boldsymbol{\bar{X}}\\
\boldsymbol{I}_{n\times n}-\boldsymbol{1}_{n\times n}
\end{pmatrix}
\right)=f_{FF,i}^{(mnm)}\left(\boldsymbol{\bar{X}}^{(flt)}\right)\boldsymbol{1}_{1\times n}.
\end{align*}
The weight bounds of $\boldsymbol{T}_{i}^{(mnm)}$ are $B_{FF}=C,B_{SA}=1$. Based on Lemma \ref{parallel}, there exists a Transformer $\boldsymbol{T}^{(prl-mnm)}:\mathbb{R}^{(d+n)\times n}\to\mathbb{R}^{C_{s+dn}^{dn}\times n}$ with size 
\begin{align*}
\left\{\left(2,2dnC_{s+dn}^{dn}\right),\left(C_{s+dn}^{dn},dn\right),\left(C\ln\frac{3^{\lceil\log_2(s-1)\rceil}-1}{2\epsilon_1},\max\left\{21\cdot2^{\lceil\log_2(s-1)\rceil-1},2dn\right\}C_{s+dn}^{dn}\right)\right\} 
\end{align*}
and dimension vector $\begin{pmatrix}
d+n&d+n&2dnC_{s+dn}^{dn}&C_{s+dn}^{dn}
\end{pmatrix}$ such that 
\begin{align*}
\boldsymbol{T}^{(prl-mnm)}\left(\boldsymbol{Y}\right)
:=\begin{pmatrix}
\boldsymbol{T}_1^{(mnm)}\left(\boldsymbol{Y}\right)\\
\boldsymbol{T}_2^{(mnm)}\left(\boldsymbol{Y}\right)\\
\vdots\\
\boldsymbol{T}_{C_{s+dn}^{dn}}^{(mnm)}\left(\boldsymbol{Y}\right)
\end{pmatrix},\quad \boldsymbol{Y}\in\mathbb{R}^{(d+n)\times n}.
\end{align*}
It follows that
\begin{align*}
\boldsymbol{T}^{(prl-mnm)}\left(
\begin{pmatrix}
\boldsymbol{\bar{X}}\\
\boldsymbol{I}_{n\times n}-\boldsymbol{1}_{n\times n}
\end{pmatrix}
\right)
=
\begin{pmatrix}
f_{FF,1}^{(mnm)}\left(\boldsymbol{\bar{X}}^{(flt)}\right)\\
f_{FF,2}^{(mnm)}\left(\boldsymbol{\bar{X}}^{(flt)}\right)\\
\vdots\\
f_{FF,C_{s+dn}^{dn}}^{(mnm)}\left(\boldsymbol{\bar{X}}^{(flt)}\right)
\end{pmatrix}\boldsymbol{1}_{1\times n}.
\end{align*}
The weight bounds of $\boldsymbol{T}^{(prl-mnm)}$ are $B_{FF}=C,B_{SA}=1$.

\textbf{Step 4: Parallel $c_{i,j}$ and the appromants of $\left(\boldsymbol{X}-\boldsymbol{X}^{(j)}\right)^{\boldsymbol{\alpha}_i}$.}

From the result of step 1, we can construct a feedforward block $\boldsymbol{\mathcal{F}}_{FF}^{(mid)}:\mathbb{R}^{(d+n+1)\times n}\to\mathbb{R}^{(2d+n+1)\times n}$ as
\begin{align*}
\boldsymbol{\mathcal{F}}_{FF}^{(mid)}(\boldsymbol{Y}
):=\begin{pmatrix}
\boldsymbol{\mathcal{F}}_{FF}^{(idt)}(\boldsymbol{Y}_{1:d,:}
)-\boldsymbol{\mathcal{F}}_{FF}^{(prl-drc)}(\boldsymbol{Y}_{1:d,:}
)\\
\boldsymbol{\mathcal{F}}_{FF}^{(idt)}(\boldsymbol{Y}_{(d+1):(d+n),:}
)\\
\boldsymbol{\mathcal{F}}_{FF}^{(prl-drc)}(\boldsymbol{Y}_{1:d,:}
)\\
\boldsymbol{\mathcal{F}}_{FF}^{(idt)}(\boldsymbol{Y}
_{d+n+1,:})
\end{pmatrix},\quad\boldsymbol{Y}\in\mathbb{R}^{(d+n+1)\times n},
\end{align*}
where $\boldsymbol{\mathcal{F}}_{FF}^{(idt)}$ is the identical mapping. It follows that for $\boldsymbol{X}\in\Omega_j$ with some $j\in[K^{dn}]$,
\begin{align*}
\boldsymbol{\mathcal{F}}_{FF}^{(mid)}\circ\boldsymbol{\mathcal{F}}_{EB}(\boldsymbol{X})=
\begin{pmatrix}
\boldsymbol{X}-\boldsymbol{X}^{(j)}\\
\boldsymbol{I}_{n\times n}-\boldsymbol{1}_{n\times n}\\
\boldsymbol{X}^{(j)}\\
\boldsymbol{\widetilde{B}}_{EB}
\end{pmatrix}.
\end{align*}
Based on Lemma \ref{identity} and the structure of $\boldsymbol{\mathcal{F}}_{FF}^{(drc)}$ constructed in step 1, it is not hard to see that $\boldsymbol{\mathcal{F}}_{FF}^{(prl-drc)}$ is of depth $3$, width $(K+2)d+2n+2$ and weight bound  $\max\{1/\delta,1\}$.

Lemma \ref{parallel} ensures the existence of a Transformer  $\boldsymbol{T}^{(prl)}:\mathbb{R}^{(2d+n+1)\times n}\to\mathbb{R}^{(d+1)C_{s+dn}^{dn}\times n}$ with size

\begin{align*}
L_0&=2,\quad W_0=dC_{s+dn}^{dn}\left[\max\{d+1,5\}+2n\right];\\
H_1&=(3d+1)C_{s+dn}^{dn},\quad S_1=\max\{dn,3\};\\
L_l&=3,\quad W_l=dC_{s+dn}^{dn}(4n+11);\quad l=1,\cdots,n-1;\\
H_l&=3dC_{s+dn}^{dn},\quad S_l=3;\quad l=2,\cdots,n;\\
L_n&=C\ln\frac{3^{\lceil\log_2(s-1)\rceil}-1}{2\epsilon_1},\quad W_n=dC_{s+dn}^{dn}(nK^{dn}-1) + \max\left\{21 \cdot 2^{\lceil \log_2(s-1) \rceil - 1}, 2dn\right\} C_{s+dn}^{dn}
\end{align*}
and dimension vector
\begin{align*}
\begin{pmatrix}
2d+n+1&2d+n+1&&(2dn+5d)C_{s+dn}^{dn}\cdot\boldsymbol{1}_{1\times n}&(d+1)C_{s+dn}^{dn}
\end{pmatrix}
\end{align*}
such that
\begin{align*}
\boldsymbol{T}^{(prl)}(\boldsymbol{Y}):=\begin{pmatrix}
\boldsymbol{T}^{(prl-mnm)}\left(\boldsymbol{Y}_{1:(d+n),:}\right)\\
\boldsymbol{T}^{(prl-mmr)}\left(\boldsymbol{W}_1\boldsymbol{Y}_{(d+n+1):(2d+n+1),:}\right)
\end{pmatrix},\quad \boldsymbol{Y}\in\mathbb{R}^{(2d+n+1)\times n},
\end{align*}
where
\begin{align*}
\boldsymbol{W}_{1}:=\begin{pmatrix}
\boldsymbol{I}_{d\times d}&\boldsymbol{1}_{d\times 1}
\end{pmatrix}.
\end{align*}
The weight bounds of $\boldsymbol{T}^{(prl)}$ are $B_{FF}=C(\boldsymbol{f})d^2n^{7}K^{6dn+2},B_{SA}=\frac{1}{2}\log(3{ d} {\pi} n^{4}(3n+1) K^{4dn+1})$. For $\boldsymbol{X}\in\Omega_j$ with some $j\in\left[K^{dn}\right]$,
\begin{align*}
\boldsymbol{T}^{(prl)}\circ\boldsymbol{\mathcal{F}}_{FF}^{(mid)}\circ\boldsymbol{\mathcal{F}}_{EB}(\boldsymbol{X})=\begin{pmatrix}
\boldsymbol{T}^{(prl-mnm)}\left(\begin{pmatrix}
\boldsymbol{X}-\boldsymbol{X}^{(j)}\\
\boldsymbol{I}_{n\times n}-\boldsymbol{1}_{n\times n}
\end{pmatrix}\right)\\
\boldsymbol{T}^{(prl-mmr)}\left(
\boldsymbol{X}^{(j)}+
\boldsymbol{1}_{d\times 1}\boldsymbol{\widetilde{B}}_{EB}\right)
\end{pmatrix}.
\end{align*}
It is worth mentioning that the last feedforward block of $\boldsymbol{T}^{(prl)}$ has only $dC_{s+dn}^{dn}(nK^{dn}-1) + \max\left\{21 \cdot 2^{\lceil \log_2(s-1) \rceil - 1}, 2dn\right\} C_{s+dn}^{dn}$ neurons in the final layer, while the number of neurons in each of the other layers does not exceed $10dC_{s+dn}^{dn}+\max\left\{21 \cdot 2^{\lceil \log_2(s-1) \rceil - 1}, 2dn\right\} C_{s+dn}^{dn}$.

\textbf{Step 5: Approximate the multiplications of $c_{ij}$ and $\left(\boldsymbol{X}-\boldsymbol{X}^{(j)}\right)^{\boldsymbol{\alpha}_i}$ and sum them over $i$.}

We apply Lemma \ref{multiplication} to approximate the multiplications of $c_{ij}$ and $f_{FF,i}^{(mnm)}\left(\boldsymbol{\bar{X}}^{(flt)}\right)$ (the approximant of $\left(\boldsymbol{X}-\boldsymbol{X}^{(j)}\right)^{\boldsymbol{\alpha}_i}$). Since $|c_{ij}|\leq C(\boldsymbol{f})$ and $\left|f_{FF,i}^{(mnm)}\left(\boldsymbol{\bar{X}}^{(flt)}\right)\right|\leq2$ (assuming $\epsilon_1<1$), we choose $B$ in Lemma \ref{multiplication} to be $C(\boldsymbol{f})$. Then according to Lemma \ref{multiplication}, for $i=jd+k\in\left[dC_{s+dn}^{dn}\right]$ with some $0\leq j\leq C_{s+dn}^{dn}-1$ and $1\leq k\leq d$, there exists a ReLU FNN $f_{FF,i}^{(mtp-two)}:\mathbb{R}^{(d+1)C_{s+dn}^{dn}}\to\mathbb{R}$ with width $\max\left\{21,(d+1)C_{s+dn}^{dn}\right\}$, depth $C(\boldsymbol{f})\ln\frac{1}{\epsilon_2}$ and weight bound $C(\boldsymbol{f})$ such that for any $\boldsymbol{x}\in\left[-C(\boldsymbol{f}),C(\boldsymbol{f})\right]^{(d+1)C_{s+dn}^{dn}}$, $f_{FF,i}^{(mtp-two)}(\boldsymbol{x}):=\widetilde{\times}\left(x_{j+1},x_{C_{s+dn}^{dn}+i}\right)$ and
\begin{align}\label{Linfty2}
\left|f_{FF,i}^{(mtp-two)}(\boldsymbol{x})-x_{j+1}x_{C_{s+dn}^{dn}+i}\right|\leq\epsilon_2,
\end{align}
where $\epsilon_2$ will be determined later. Let $\boldsymbol{\mathcal{F}}_{FF,i}^{(mtp-two)}:\mathbb{R}^{(d+1)C_{s+dn}^{dn}\times n}\to\mathbb{R}^{1\times n}$ be the feedforward block generated from $f_{FF,i}^{(mtp-two)}$.
By Lemma \ref{parallel}, there exists a feedforward block $\boldsymbol{\mathcal{F}}_{FF}^{(prl-mtp)}:\mathbb{R}^{(d+1)C_{s+dn}^{dn}\times n}\to\mathbb{R}^{dC_{s+dn}^{dn}\times n}$ with width $21dC_{s+dn}^{dn}$, depth $C(\boldsymbol{f})\ln\frac{1}{\epsilon_2}$ and weight bound $C(\boldsymbol{f})$ such that
\begin{align*}
\boldsymbol{\mathcal{F}}_{FF}^{(prl-mtp)}(\boldsymbol{Y}):=
\begin{pmatrix}
\boldsymbol{\mathcal{F}}_{FF,1}^{(mtp-two)}(\boldsymbol{Y})&\cdots&\boldsymbol{\mathcal{F}}_{FF,dC_{s+dn}^{dn}}^{(mtp-two)}(\boldsymbol{Y})
\end{pmatrix}^{\top},\quad\boldsymbol{Y}\in\mathbb{R}^{(d+1)C_{s+dn}^{dn}\times n}.
\end{align*}
Then for $\boldsymbol{X}\in\Omega_j$ with some $j\in\left[K^{dn}\right]$,
\begin{align*}
&\boldsymbol{\mathcal{F}}_{FF}^{(prl-mtp)}\circ\boldsymbol{T}^{(prl)}\circ\boldsymbol{\mathcal{F}}_{FF}^{(mid)}\circ\boldsymbol{\mathcal{F}}_{EB}(\boldsymbol{X})=
\\
&\begin{pmatrix}
\widetilde{\times}\left(c_{1,j}(f_{11}),f_{1}^{(mnm)}\right)&\widetilde{\times}\left(c_{1,j}(f_{12}),f_{1}^{(mnm)}\right)&\cdots&\widetilde{\times}\left(c_{1,j}(f_{1n}),f_{1}^{(mnm)}\right)\\
\vdots&\vdots&\ddots&\vdots\\
\widetilde{\times}\left(c_{1,j}(f_{d1}),f_{1}^{(mnm)}\right)&\widetilde{\times}\left(c_{1,j}(f_{d2}),f_{1}^{(mnm)}\right)&\cdots&\widetilde{\times}\left(c_{1,j}(f_{dn}),f_{1}^{(mnm)}\right)\\
\widetilde{\times}\left(c_{2,j}(f_{11}),f_{2}^{(mnm)}\right)&\widetilde{\times}\left(c_{2,j}(f_{12}),f_{2}^{(mnm)}\right)&\cdots&\widetilde{\times}\left(c_{2,j}(f_{1n}),f_{2}^{(mnm)}\right)\\
\vdots&\vdots&\ddots&\vdots\\
\widetilde{\times}\left(c_{2,j}(f_{d1}),f_{2}^{(mnm)}\right)&\widetilde{\times}\left(c_{2,j}(f_{d2}),f_{2}^{(mnm)}\right)&\cdots&\widetilde{\times}\left(c_{2,j}(f_{dn}),f_{2}^{(mnm)}\right)\\
\vdots&\vdots&&\vdots\\
\widetilde{\times}\left(c_{C_{s+dn}^{dn},j}(f_{11}),f_{C_{s+dn}^{dn}}^{(mnm)}\right)&\widetilde{\times}\left(c_{C_{s+dn}^{dn},j}(f_{12}),f_{C_{s+dn}^{dn}}^{(mnm)}\right)&\cdots&\widetilde{\times}\left(c_{C_{s+dn}^{dn},j}(f_{1n}),f_{C_{s+dn}^{dn}}^{(mnm)}\right)\\
\vdots&\vdots&\ddots&\vdots\\
\widetilde{\times}\left(c_{C_{s+dn}^{dn},j}(f_{d1}),f_{C_{s+dn}^{dn}}^{(mnm)}\right)&\widetilde{\times}\left(c_{C_{s+dn}^{dn},j}(f_{d2}),f_{C_{s+dn}^{dn}}^{(mnm)}\right)&\cdots&\widetilde{\times}\left(c_{C_{s+dn}^{dn},j}(f_{dn}),f_{C_{s+dn}^{dn}}^{(mnm)}\right)\\
\end{pmatrix},
\end{align*}
where $f_{i}^{(mnm)}$ is short for $f_{FF,i}^{(mnm)}\left(\boldsymbol{\bar{X}}^{(flt)}\right)$. Denoting
\begin{align*}
\boldsymbol{W}^{sum}:=
\begin{pmatrix}
\boldsymbol{I}_{d\times d}&\boldsymbol{I}_{d\times d}&\cdots&\boldsymbol{I}_{d\times d}
\end{pmatrix}\in\mathbb{R}^{d\times dC_{s+dn}^{dn}},
\end{align*}
then for $\boldsymbol{X}\in\Omega_j$ with some $j\in\left[K^{dn}\right]$, we have

\begin{small}
\begin{align*}
&\boldsymbol{W}^{(sum)}\boldsymbol{\mathcal{F}}_{FF}^{(prl-mtp)}\circ\boldsymbol{T}^{(prl)}\circ\boldsymbol{\mathcal{F}}_{FF}^{(mid)}\circ\boldsymbol{\mathcal{F}}_{EB}(\boldsymbol{X})=\\
&\begin{pmatrix}
\sum_{i=1}^{C_{s+dn}^{dn}}\widetilde{\times}\left(c_{i,j}(f_{11}),f_{i}^{(mnm)}\right)&\sum_{i=1}^{C_{s+dn}^{dn}}\widetilde{\times}\left(c_{i,j}(f_{12}),f_{i}^{(mnm)}\right)&\cdots&\sum_{i=1}^{C_{s+dn}^{dn}}\widetilde{\times}\left(c_{i,j}(f_{1n}),f_{i}^{(mnm)}\right)\\
\sum_{i=1}^{C_{s+dn}^{dn}}\widetilde{\times}\left(c_{i,j}(f_{21}),f_{i}^{(mnm)}\right)&\sum_{i=1}^{C_{s+dn}^{dn}}\widetilde{\times}\left(c_{i,j}(f_{22}),f_{i}^{(mnm)}\right)&\cdots&\sum_{i=1}^{C_{s+dn}^{dn}}\widetilde{\times}\left(c_{i,j}(f_{2n}),f_{i}^{(mnm)}\right)\\
\vdots&\vdots&\ddots&\vdots\\
\sum_{i=1}^{C_{s+dn}^{dn}}\widetilde{\times}\left(c_{i,j}(f_{d1}),f_{i}^{(mnm)}\right)&\sum_{i=1}^{C_{s+dn}^{dn}}\widetilde{\times}\left(c_{i,j}(f_{d2}),f_{i}^{(mnm)}\right)&\cdots&\sum_{i=1}^{C_{s+dn}^{dn}}\widetilde{\times}\left(c_{i,j}(f_{dn}),f_{i}^{(mnm)}\right)
\end{pmatrix}.
\end{align*}
\end{small}
The Transformer $\boldsymbol{T}$ we are going to find is exactly defined as
\begin{align*}
\boldsymbol{T}(\boldsymbol{X}):=\boldsymbol{W}^{(sum)}\boldsymbol{\mathcal{F}}_{FF}^{(prl-mtp)}\circ\boldsymbol{T}^{(prl)}\circ\boldsymbol{\mathcal{F}}_{FF}^{(mid)}\circ\boldsymbol{\mathcal{F}}_{EB}(\boldsymbol{X}).
\end{align*}
Its size is
\begin{align*}
L_0&=4,\quad W_0=\max\{(K+2)d+2n+2,dC_{s+dn}^{dn}\left[\max\{d+1,5\}+2n\right]\};\\
H_1&=(3d+1)C_{s+dn}^{dn},\quad S_1=\max\{dn,3\};\\
L_l&=3,\quad W_l=dC_{s+dn}^{dn}(4n+11);\quad l=1,\cdots,n-1;\\
H_l&=3dC_{s+dn}^{dn},\quad S_l=3;\quad l=2,\cdots,n;\\
L_n&=C\ln\frac{3^{\lceil\log_2(s-1)\rceil}-1}{2\epsilon_1}+C(\boldsymbol{f})\ln\frac{1}{\epsilon_2},\\
W_n&=dC_{s+dn}^{dn}(nK^{dn}-1) + \max\left\{21 \cdot 2^{\lceil \log_2(s-1) \rceil - 1}, 2dn,21d\right\} C_{s+dn}^{dn}
\end{align*}
and its dimension vector is
\begin{align*}
\begin{pmatrix}
d&d+n+1&(2dn+5d)C_{s+dn}^{dn}\cdot\boldsymbol{1}_{1\times n}&d
\end{pmatrix}.
\end{align*}
The weight bounds of $\boldsymbol{T}$ are $B_{FF}=\max\left\{C(\boldsymbol{f})d^2n^{7}K^{6dn+2},1/\delta\right\},B_{SA}=\frac{1}{2}\log(3{ d} {\pi} n^{4}(3n+1) K^{4dn+1})$.

It is worth mentioning that the last feedforward block of $\boldsymbol{T}$ has only $dC_{s+dn}^{dn}(nK^{dn}-1) + \max\left\{21 \cdot 2^{\lceil \log_2(s-1) \rceil - 1}, 2dn,21d\right\} C_{s+dn}^{dn}$ neurons in a certain layer, while the number of neurons in each of the other layers does not exceed $10dC_{s+dn}^{dn}\linebreak+\max\left\{21 \cdot 2^{\lceil \log_2(s-1) \rceil - 1}, 2dn,21d\right\} C_{s+dn}^{dn}$.

\textbf{Step 6: Estimate the approximation error on $\bigcup_{j\in[K^{dn}]}\Omega_j$.}

Let $p\in[d],q\in[n]$. For $\boldsymbol{X}\in\Omega_j$ with some $j\in\left[K^{dn}\right]$, combining \eqref{Taylor error}\eqref{Linfty1}\eqref{Linfty2} and letting 
\begin{align*}
K=\Theta\left(\epsilon^{-1/\gamma}\right),\quad\epsilon_1=\Theta\left(\epsilon\right),\quad\epsilon_2=\Theta\left(\epsilon\right) 
\end{align*}
yields
\begin{align*}
&|T_{pq}(\boldsymbol{X})-f_{pq}(\boldsymbol{X})|\\
&=\left|\sum_{i=1}^{C_{s+dn}^{dn}}\widetilde{\times}\left(c_{i,j}(f_{pq}),f_{i}^{(mnm)}\left(\boldsymbol{\bar{X}}^{(flt)}\right)\right)-f_{pq}(\boldsymbol{X})\right|\\
&=\left|\sum_{i=1}^{C_{s+dn}^{dn}}\widetilde{\times}\left(c_{i,j}(f_{pq}),f_{i}^{(mnm)}\left(\boldsymbol{\bar{X}}^{(flt)}\right)\right)-\sum_{i=1}^{C_{s+dn}^{dn}}c_{i,j}(f_{pq})f_{i}^{(mnm)}\left(\boldsymbol{\bar{X}}^{(flt)}\right)\right|\\
&\quad+\left|\sum_{i=1}^{C_{s+dn}^{dn}}c_{i,j}(f_{pq})f_{i}^{(mnm)}\left(\boldsymbol{\bar{X}}^{(flt)}\right)-\sum_{i=1}^{C_{s+dn}^{dn}}c_{i,j}(f_{pq})\boldsymbol{\bar{X}}^{\boldsymbol{\alpha}_i}\right|\\
&\quad+\left|\sum_{i=1}^{C_{s+dn}^{dn}}c_{i,j}(f_{pq})\boldsymbol{\bar{X}}^{\boldsymbol{\alpha}_i}-f_{pq}(\boldsymbol{X})\right|\\
&\leq\sum_{i=1}^{C_{s+dn}^{dn}}\left|\widetilde{\times}\left(c_{i,j}(f_{pq}),f_{i}^{(mnm)}\left(\boldsymbol{\bar{X}}^{(flt)}\right)\right)-c_{i,j}(f_{pq})f_{i}^{(mnm)}\left(\boldsymbol{\bar{X}}^{(flt)}\right)\right|\\
&\quad+\sum_{i=1}^{C_{s+dn}^{dn}}|c_{i,j}(f_{pq})|\left|f_{i}^{(mnm)}\left(\boldsymbol{\bar{X}}^{(flt)}\right)-\boldsymbol{\bar{X}}^{\boldsymbol{\alpha}_i}\right|\\
&\quad+\left|\sum_{i=1}^{C_{s+dn}^{dn}}c_{i,j}(f_{pq})\boldsymbol{\bar{X}}^{\boldsymbol{\alpha}_i}-f_{pq}(\boldsymbol{X})\right|\\
&\leq\frac{\epsilon}{3}+\frac{\epsilon}{3}+\frac{\epsilon}{3}=\epsilon.
\end{align*}

\textbf{Step 7: Estimate the magnitude of $T$ on $\Omega^{(flaw)}$.}

For $p\in[d],q\in[n]$ and $\boldsymbol{X}\in\Omega^{(flaw)}$, by the above construction, we have
\begin{align}\label{Linfty3}
T_{pq}(\boldsymbol{X})=\sum_{i=1}^{C_{s+dn}^{dn}}\widetilde{\times}\left(\left[\boldsymbol{T}_{i,p}^{(mmr)}\left(\boldsymbol{\mathcal{F}}_{FF}^{(prl-dsc)}(\boldsymbol{X})\right)\right]_{1,q},\left[\boldsymbol{T}_i^{(mnm)}\left(\boldsymbol{X}-\boldsymbol{\mathcal{F}}_{FF}^{(prl-dsc)}(\boldsymbol{X})\right)\right]_{1,q}\right).
\end{align}
Applying Lemma \ref{Transformer memorization} with paramters $r=\sqrt{d},\delta=\frac{1}{K},N=K^{dn},B_y=C(\boldsymbol{f})$ therein, we have
\begin{align*}
\left|\left[\boldsymbol{T}_{i,p}^{(mmr)}\left(\boldsymbol{\mathcal{F}}_{FF}^{(prl-dsc)}(\boldsymbol{X})\right)\right]_{1,q}\right|\leq C(\boldsymbol{f})d^2n^{8}K^{7dn+2}.
\end{align*}
Noting that
\begin{align*}
\left|\left[\boldsymbol{T}_i^{(mnm)}\left(\boldsymbol{X}-\boldsymbol{\mathcal{F}}_{FF}^{(prl-dsc)}(\boldsymbol{X})\right)\right]_{1,q}\right|\leq2
\end{align*}
and hence applying Lemma \ref{multiplication} with  $B'=C(\boldsymbol{f})d^2n^{8}K^{7dn+2}$, we have
\begin{align*}
\left|\widetilde{\times}\left(\left[\boldsymbol{T}_{i,p}^{(mmr)}\left(\boldsymbol{\mathcal{F}}_{FF}^{(prl-dsc)}(\boldsymbol{X})\right)\right]_{1,q},\left[\boldsymbol{T}_i^{(mnm)}\left(\boldsymbol{X}-\boldsymbol{\mathcal{F}}_{FF}^{(prl-dsc)}(\boldsymbol{X})\right)\right]_{1,q}\right)\right|\leq C(\boldsymbol{f})d^2n^{8}K^{7dn+2}.
\end{align*}
Plugging the above estimate into \eqref{Linfty3} yields
\begin{align*}
\left|T_{pq}(\boldsymbol{X})\right|\leq C(\boldsymbol{f})C_{s+dn}^{dn}d^2n^{8}K^{7dn+2}.
\end{align*}

\end{proof}

\subsection{Proof of Lemma \ref{FNN to transformer}}\label{proof of FNN to transformer}

Denote
\begin{align*}
\boldsymbol{\widetilde{X}}:=\begin{pmatrix}
\boldsymbol{X}\\
\boldsymbol{I}_{n\times n}-\boldsymbol{1}_{n\times n}
\end{pmatrix}.
\end{align*}
Define $\boldsymbol{\mathcal{F}}_{FF,1}^{(FF)}:\mathbb{R}^{(d+n)\times n}\to\mathbb{R}^{2dn\times n}$ as
\begin{align*}
\boldsymbol{\mathcal{F}}_{FF,1}^{(FF)}(\boldsymbol{Y}):=\sigma_R\left(\boldsymbol{W}_1\boldsymbol{Y}\right),\quad\boldsymbol{Y}\in\mathbb{R}^{(d+n)\times n}
\end{align*}
with
\begin{align*}
\boldsymbol{W}_1:=
\begin{pmatrix}
\boldsymbol{1}_{n\times1}&\boldsymbol{0}_{n\times1}&\cdots&\boldsymbol{0}_{n\times1}&\boldsymbol{I}_{n\times n}\\
\boldsymbol{0}_{n\times1}&\boldsymbol{1}_{n\times1}&\cdots&\boldsymbol{0}_{n\times1}&\boldsymbol{I}_{n\times n}\\
\vdots&\vdots&\ddots&\vdots&\vdots\\
\boldsymbol{0}_{n\times1}&\boldsymbol{0}_{n\times1}&\cdots&\boldsymbol{1}_{n\times1}&\boldsymbol{I}_{n\times n}\\
&&\boldsymbol{0}_{dn\times (d+n)}
\end{pmatrix}\in\mathbb{R}^{2dn\times (d+n)}.
\end{align*}
It can be checked that
\begin{align*}
\boldsymbol{\mathcal{F}}_{FF,1}^{(FF)}\left(\boldsymbol{\widetilde{X}}\right)=\sigma_R\left(\boldsymbol{W}_1\boldsymbol{\widetilde{X}}\right)=\begin{pmatrix}
\text{diag}(x_{11},x_{12},\dots,x_{1n})\\
\text{diag}(x_{21},x_{22},\dots,x_{2n})\\
\vdots\\
\text{diag}(x_{d1},x_{d2},\dots,x_{dn})\\
\boldsymbol{0}_{dn\times n}
\end{pmatrix}\in\mathbb{R}^{2dn\times n},
\end{align*}
where
\begin{align*}
\text{diag}(x_{i1},x_{i2},\dots,x_{in}):=\begin{pmatrix}
x_{i1}&0&\cdots&0\\
0&x_{i2}&\cdots&0\\
\vdots&\vdots&\ddots&\vdots\\
0&0&\cdots&x_{in}
\end{pmatrix}.
\end{align*}
Define the matrices in the self-attention layer to be
\begin{align*}
\boldsymbol{W}_O:=\begin{pmatrix}
\boldsymbol{0}_{dn\times dn}\\
n\boldsymbol{I}_{dn\times dn}
\end{pmatrix},\quad
\boldsymbol{W}_V:=\begin{pmatrix}
\boldsymbol{I}_{dn\times dn}&\boldsymbol{0}_{dn\times dn}\end{pmatrix},\quad\boldsymbol{W}_K:=\boldsymbol{W}_Q:=\boldsymbol{0}_{1\times 2dn}.
\end{align*}
It follows that the output of the softmax function is $\frac{1}{n}\boldsymbol{1}_{n\times n}$ and hence
\begin{align*}
&\boldsymbol{\mathcal{F}}_{SA}^{(FF)}\circ\boldsymbol{\mathcal{F}}_{FF,1}^{(FF)}\left(\boldsymbol{\widetilde{X}}\right)\\
&=\begin{pmatrix}
\text{diag}(x_{11},x_{12},\dots,x_{1n})\\
\text{diag}(x_{21},x_{22},\dots,x_{2n})\\
\vdots\\
\text{diag}(x_{d1},x_{d2},\dots,x_{dn})\\
\boldsymbol{0}_{dn\times n}
\end{pmatrix}+\frac{1}{n}\begin{pmatrix}
\boldsymbol{0}_{dn\times dn}\\
n\boldsymbol{I}_{dn\times dn}
\end{pmatrix}\begin{pmatrix}
\boldsymbol{I}_{dn\times dn}&\boldsymbol{0}_{dn\times dn}\end{pmatrix}\begin{pmatrix}
\text{diag}(x_{11},x_{12},\dots,x_{1n})\\
\text{diag}(x_{21},x_{22},\dots,x_{2n})\\
\vdots\\
\text{diag}(x_{d1},x_{d2},\dots,x_{dn})\\
\boldsymbol{0}_{dn\times n}
\end{pmatrix}\boldsymbol{1}_{n\times n}\\
&=\begin{pmatrix}
\text{diag}(x_{11},x_{12},\dots,x_{1n})\\
\text{diag}(x_{21},x_{22},\dots,x_{2n})\\
\vdots\\
\text{diag}(x_{d1},x_{d2},\dots,x_{dn})\\
\boldsymbol{X}^{(flt)}\boldsymbol{1}_{1\times n}
\end{pmatrix}\in\mathbb{R}^{2dn\times n}.
\end{align*}
Denote
\begin{align*}
\boldsymbol{W}_2:=\begin{pmatrix}
\boldsymbol{0}_{dn\times dn}
&\boldsymbol{\boldsymbol{I}}_{dn\times dn}
\end{pmatrix}
\end{align*}
and let
\begin{align*}
\boldsymbol{\mathcal{F}}_{FF,2}^{(FF)}(\boldsymbol{Y}):=\boldsymbol{\mathcal{F}}_{FF}(\boldsymbol{W}_2\boldsymbol{Y}),\quad\boldsymbol{Y}\in\mathbb{R}^{2dn\times n},
\end{align*}
where ${\boldsymbol{\mathcal{F}}}_{FF}$ is the feedforward block generated from $f_{FF}$. Then the Transformer $\boldsymbol{T}^{(FF)}$ is defined as
\begin{align*}
\boldsymbol{T}^{(FF)}\left(\boldsymbol{\widetilde{X}}\right)&:=\boldsymbol{\mathcal{F}}_{FF,2}^{(FF)}\circ\boldsymbol{\mathcal{F}}_{SA}^{(FF)}\circ\boldsymbol{\mathcal{F}}_{FF,1}^{(FF)}\left(\boldsymbol{\widetilde{X}}\right)\\
&=\boldsymbol{\mathcal{F}}_{FF}\left(\begin{pmatrix}
\boldsymbol{0}_{dn\times dn}
&\boldsymbol{\boldsymbol{I}}_{dn\times dn}
\end{pmatrix}\begin{pmatrix}
\text{diag}(x_{11},x_{12},\dots,x_{1n})\\
\text{diag}(x_{21},x_{22},\dots,x_{2n})\\
\vdots\\
\text{diag}(x_{d1},x_{d2},\dots,x_{dn})\\
\boldsymbol{X}^{(flt)}\boldsymbol{1}_{1\times n}
\end{pmatrix}\right)\\
&=\boldsymbol{\mathcal{F}}_{FF}\left(\boldsymbol{X}^{(flt)}\boldsymbol{1}_{1\times n}\right)=f_{FF}\left(\boldsymbol{X}^{(flt)}\right)\boldsymbol{1}_{1\times n}.
\end{align*}

\subsection{Proof of Lemma \ref{monomial}}\label{proof of lemma monomial}

\begin{lemma}\label{mul-multiplication}
Let $d\in\mathbb{N}_{>1}$. For any $0<\epsilon\leq\frac{3^{\lceil\log_2d\rceil}-1}{3^{\lceil\log_2d\rceil-1}-1}$, there exists a ReLU FNN function $f_{FF,d}^{(mtp)}:\mathbb{R}^d\to\mathbb{R}$ with width $21\cdot2^{\lceil\log_2d\rceil-1}$, depth $C\ln\frac{3^{\lceil\log_2d\rceil}-1}{2\epsilon}$ and weight bound $C$ such that for any $\boldsymbol{x}\in[0,1]^d$,
\begin{align*}
\left|f_{FF,d}^{(mtp)}(\boldsymbol{{x}})-x_1\cdots x_d\right|\leq\epsilon.
\end{align*}
\end{lemma}
\begin{proof}
Assume $2^{\widetilde{d}-1}< d\leq2^{\widetilde{d}}$ for some $\widetilde{d}\in\mathbb{N}_{\geq1}$. We first adopt a linear mapping to expand $\boldsymbol{{x}}$ into
\begin{align*}
\boldsymbol{\widetilde{x}}:=\begin{pmatrix}
\boldsymbol{x}\\
\boldsymbol{1}_{2^{\widetilde{d}}-d}
\end{pmatrix}\in\mathbb{R}^{2^{\widetilde{d}}}.
\end{align*}
Without loss of generality, in the following we simply assume $d=2^{\widetilde{d}}$. $f_{FF,2^{\widetilde{d}}}^{(mtp)}$ is recursively defined by
\begin{align}\label{mul-multiplication1}
f_{FF,2^{\widetilde{d}}}^{(mtp)}(\boldsymbol{{x}}):=\widetilde{\times}\left(f_{FF,2^{\widetilde{d}-1}}^{(mtp)}(\boldsymbol{{x}}_1),f_{FF,2^{\widetilde{d}-1}}^{(mtp)}(\boldsymbol{{x}}_2)\right),\quad \boldsymbol{x}\in[0,1]^{2^{\widetilde{d}}},
\end{align}
where $\widetilde{\times}$ is the FNN function in Lemma \ref{multiplication} that achieves approximation accuracy $\widetilde{\epsilon}$ and
\begin{align*}
\boldsymbol{{x}}_1=\left(x_1,\cdots,x_{2^{\widetilde{d}-1}}\right),\quad \boldsymbol{{x}}_2=\left(x_{2^{\widetilde{d}-1}+1},\cdots,x_{2^{\widetilde{d}}}\right).
\end{align*}
We show by induction on $\widetilde{d}$ that the ReLU FNN function $f_{FF,2^{\widetilde{d}}}^{(mtp)}$ defined in \eqref{mul-multiplication1} is of width $21\cdot2^{\widetilde{d}-1}$, depth $C\widetilde{d}\ln\frac{1}{\widetilde{\epsilon}}$ and weight bound $C$, and 
\begin{align*}
\left|f_{FF,2^{\widetilde{d}}}^{(mtp)}(\boldsymbol{{x}})-x_1\cdots x_{2^{\widetilde{d}}}\right|\leq\frac{1}{2}\left(3^{\widetilde{d}}-1\right)\widetilde{\epsilon}.
\end{align*}
The case of $\widetilde{d}=1$ is verified by Lemma \ref{multiplication}. Now we assume $f_{FF,2^{\widetilde{d}-1}}^{(mtp)}$ is of width $21\cdot2^{\widetilde{d}-2}$, depth $C(\widetilde{d}-1)\ln\frac{1}{\widetilde{\epsilon}}$ and weight bound $C$, and for any  $\boldsymbol{y}\in[0,1]^{2^{\widetilde{d}-1}}$,
\begin{align*}
\left|f_{FF,2^{\widetilde{d}-1}}^{(mtp)}(\boldsymbol{{y}})-y_1\cdots y_{2^{\widetilde{d}-1}}\right|\leq\frac{1}{2}\left(3^{\widetilde{d}-1}-1\right)\widetilde{\epsilon}.
\end{align*}
Supposing $\widetilde{\epsilon}\leq\frac{2}{3^{\widetilde{d}-1}-1}$, we have
\begin{align*}
\left|f_{FF,2^{\widetilde{d}-1}}^{(mtp)}(\boldsymbol{{y}})\right|&\leq\left|f_{FF,2^{\widetilde{d}-1}}^{(mtp)}(\boldsymbol{{y}})-y_1\cdots y_{2^{\widetilde{d}-1}}\right|+\left|y_1\cdots y_{2^{\widetilde{d}-1}}\right|\\
&\leq\frac{1}{2}\left(3^{\widetilde{d}-1}-1\right)\widetilde{\epsilon}+1\leq2.
\end{align*}
By \eqref{mul-multiplication1}, $f_{FF,2^{\widetilde{d}}}^{(mtp)}$ is of width $21\cdot2^{\widetilde{d}-1}$, depth $C\widetilde{d}\ln\frac{1}{\widetilde{\epsilon}}$ and weight bound $C$. Furthermore, for any $\boldsymbol{x}\in[0,1]^{2^{\widetilde{d}}}$,  
\begin{align*}
&\left|f_{FF,2^{\widetilde{d}}}^{(mtp)}(\boldsymbol{{x}})-x_1\cdots x_{2^{\widetilde{d}}}\right|\\
&\leq\left|f_{FF,2}^{(mtp)}\left(f_{FF,2^{\widetilde{d}-1}}^{(mtp)}(\boldsymbol{{x}}_1),f_{FF,2^{\widetilde{d}-1}}^{(mtp)}(\boldsymbol{{x}}_2)\right)-f_{FF,2^{\widetilde{d}-1}}^{(mtp)}(\boldsymbol{{x}}_1)f_{FF,2^{\widetilde{d}-1}}^{(mtp)}(\boldsymbol{{x}}_2)\right|\\
&\quad+\left|f_{FF,2^{\widetilde{d}-1}}^{(mtp)}(\boldsymbol{{x}}_1)f_{FF,2^{\widetilde{d}-1}}^{(mtp)}(\boldsymbol{{x}}_2)-x_1\cdots x_{2^{\widetilde{d}-1}}\cdot f_{FF,2^{\widetilde{d}-1}}^{(mtp)}(\boldsymbol{{x}}_2)\right|\\
&\quad+\left|x_1\cdots x_{2^{\widetilde{d}-1}}\cdot f_{FF,2^{\widetilde{d}-1}}^{(mtp)}(\boldsymbol{{x}}_2)-x_1\cdots x_{2^{\widetilde{d}}}\right|\\
&\leq\left|f_{FF,2}^{(mtp)}\left(f_{FF,2^{\widetilde{d}-1}}^{(mtp)}(\boldsymbol{{x}}_1),f_{FF,2^{\widetilde{d}-1}}^{(mtp)}(\boldsymbol{{x}}_2)\right)-f_{FF,2^{\widetilde{d}-1}}^{(mtp)}(\boldsymbol{{x}}_1)f_{FF,2^{\widetilde{d}-1}}^{(mtp)}(\boldsymbol{{x}}_2)\right|\\
&\quad+2\left|f_{FF,2^{\widetilde{d}-1}}^{(mtp)}(\boldsymbol{{x}}_1)-x_1\cdots x_{2^{\widetilde{d}-1}}\right|+\left|f_{FF,2^{\widetilde{d}-1}}^{(mtp)}(\boldsymbol{{x}}_2)-x_{2^{\widetilde{d}-1}+1}\cdots x_{2^{\widetilde{d}}}\right|\\
&\leq\widetilde{\epsilon}+2\cdot\frac{1}{2}\left(3^{\widetilde{d}-1}-1\right)\widetilde{\epsilon}+\frac{1}{2}\left(3^{\widetilde{d}-1}-1\right)\widetilde{\epsilon}\\
&=\frac{1}{2}\left(3^{\widetilde{d}}-1\right)\widetilde{\epsilon},
\end{align*}
where in the third step we use Lemma \ref{multiplication}. The proof is completed by setting $\widetilde{\epsilon}=\frac{2}{3^{\widetilde{d}}-1}\epsilon$.

\end{proof}

\begin{proof}[Proof of Lemma \ref{monomial}]
The first hidden layer is used to transform $\boldsymbol{x}$ into 
\begin{align*}
\boldsymbol{\widetilde{x}}:=(\underbrace{x_1,\cdots,x_1}_{\alpha_1},\underbrace{x_2,\cdots,x_2}_{\alpha_2},\cdots,\underbrace{x_d,\cdots,x_d}_{\alpha_d})^{\top}\in\mathbb{R}^s.
\end{align*}
According to Lemma \ref{identity}, the number of neurons in this layer is $2\bar{\alpha}$. From Lemma \ref{mul-multiplication}, we can find a ReLU FNN function $f_{FF,s}^{(mtp)}:\mathbb{R}^{\bar{\alpha}}\to\mathbb{R}$ with width $21\cdot2^{\lceil\log_2\bar{\alpha}\rceil-1}$, depth $C\ln\frac{3^{\lceil\log_2\bar{\alpha}\rceil}-1}{2\epsilon}$ and weight bound $C$ such that 
\begin{align*}
\left|f_{FF,d}^{(mtp)}(\boldsymbol{\widetilde{x}})-x_1^{\alpha_1}x_2^{\alpha_2}\cdots x_d^{\alpha_d}\right|\leq\epsilon.
\end{align*}
\end{proof}

\section{Proof of Lemma \ref{Transformer memorization}: Memorization of Transformers}\label{proof of transformer memorization}

Following the path of researches on Transformer memorization \cite{Yun2020Are,kim2023provable,kajitsuka2024are,kajitsuka2025on}, we first prove that the Transformer can achieve contextual mapping (Lemma \ref{contextual mapping}), and then associate the resulting context ids with the corresponding labels via an ReLU FNN (Lemma \ref{memorization}), thereby realizing the memorization task.

This section is divided into two subsections: In Section \ref{Proof of Lemma Transformer memorization}, we prove Lemma \ref{Transformer memorization} based on Lemmas \ref{contextual mapping} and \ref{memorization}. Since the proof of Lemma 14 is lengthy, we defer it to Section \ref{proof of lemma contextual mapping}.

\subsection{Proof of Lemma \ref{Transformer memorization}}\label{Proof of Lemma Transformer memorization}

The contextual mapping defined below assigns a unique id to each token in the input sequence.

\begin{definition}[Contextual mapping]
Let $N\in\mathbb{N}_{\geq1}$ and $r,\phi\in\mathbb{R}_{>0}$. Let \( \boldsymbol{X}^{(1)}, \cdots, \boldsymbol{X}^{(N)} \in \mathbb{R}^{d \times n} \) be a set of \( N \) input sequences. A map \( \mathcal{A} : \mathbb{R}^{d \times n} \to \mathbb{R}^{1\times n} \) is called an \((r, \phi)\)-contextual mapping if the following two conditions hold:
\begin{itemize}
\item For any \( i \in [N] \) and \( k \in [n] \), \( \left|\mathcal{A}\left(\boldsymbol{X}^{(i)}\right)_{1,k}\right| \leq r \).

\item For any \( i, j \in [N] \) and \( k, l \in [n] \) such that \( \boldsymbol{X}_{:,k}^{(i)} \neq \boldsymbol{X}_{:,l}^{(j)} \) or \( \boldsymbol{X}^{(i)} \neq \boldsymbol{X}^{(j)} \) up to permutations, \( \left|\mathcal{A}\left(\boldsymbol{X}^{(i)}\right)_{1,k} - \mathcal{A}\left(\boldsymbol{X}^{(j)}\right)_{1,l}\right| \geq \phi \).
\end{itemize}
In particular, \( \mathcal{A}(\boldsymbol{X}^{(i)})_{1,k} \) is called a \textbf{context id} of the \( k \)-th token in \( X^{(i)} \).
\end{definition}
The following lemma shows that Transformers can realize contextual mapping (in fact, this is one of the key reasons for their tremendous success in natural language processing and other domains). Its proof is presented in Section \ref{proof of lemma contextual mapping}.

\begin{lemma}\label{contextual mapping}
Let \( d,n,N \in \mathbb{N}_{\geq1} \) and $r,\phi\in\mathbb{R}_{>0}$. Let \( \boldsymbol{X}^{(1)}, \cdots, \boldsymbol{X}^{(N)} \in \mathbb{R}^{d \times n} \) be a set of \( N \) input sequences that are tokenwise \((r, \phi)\)-separated. There eixsts a Transformer $\boldsymbol{T}^{(cm)}:\mathbb{R}^{d\times n}\to\mathbb{R}^{1\times n}$ with size $\{(2,\max\{d,5\}),((3,3),(3,11))\times (n-1)\text{ times},(3,3),(2,5)\}$ and dimension vector $\begin{pmatrix}
d&d&5\cdot\boldsymbol{1}_{1\times n}&1
\end{pmatrix}$ such that it is an $(R,2)$-contextual mapping with
\begin{align*}
R:=(\sqrt{2\pi d}n^{2} N^{2}r\phi^{-1}+1)\left(\frac{3{\pi}}{4}\sqrt{d}n^{3} N^{4}r\phi^{-1}+\frac{3}{2}\right).
\end{align*}
The weight bounds of $\boldsymbol{T}^{(cm)}$ are $B_{FF}=\max\left\{\sqrt{2} n^{2} N^{2} \sqrt{\pi d} r\phi^{-1}+1,\frac{3\sqrt{2}}{8}N^2 \sqrt{\pi n}\right\},B_{SA}=\frac{1}{2}\log(3\sqrt{ d} {\pi} n^{4} N^{4} r\phi^{-1})$. Furthermore, if $\boldsymbol{X}\in\mathbb{R}^{d\times n}$ satisfies $\left\|\boldsymbol{X}_{:,k}\right\|_2\leq r$ for all $k\in[n]$, then $\left|\boldsymbol{T}^{(cm)}\left(\boldsymbol{X}\right)_{1,k}\right|\leq R$ for all $k\in[n]$.

\end{lemma}

\begin{lemma}\label{memorization}
Let $N\in\mathbb{N}_{\geq2}$ and $\phi,B_x,B_y\in\mathbb{R}_{>0}$. For any $N$ data pairs $\{(x_i,y_i)\}_{i\in[N]}\subset[-B_x,B_x]\times[-B_y,B_y]$ satisfying $|x_i-x_j|\geq\phi$ for any $i,j\in[N]$ with $i\neq j$, there exists a ReLU FNN function $f_{FF}^{(mmr)}:\mathbb{R}\to\mathbb{R}$ with width $N-1$, depth $2$ and weight bound $\max\{1,B_x,B_y,4B_y/\phi\}$ such that
\begin{align}\label{memorization1}
f_{FF}^{(mmr)}(x_i)=y_i,\quad i\in[N].
\end{align}
Furthermore, $\left|f_{FF}^{(mmr)}(x)\right|\leq\frac{8(N-1)B_xB_y}{\phi}+B_y$ for any $x\in[-B_x,B_x]$.
\end{lemma}
\begin{proof}
Without loss of generality, we assmue $x_1<x_2<\cdots<x_N$. $f_{FF}^{(mmr)}$ is defined as
\begin{align*}
f_{FF}^{(mmr)}(x):=\boldsymbol{W}_2\sigma_R(\boldsymbol{W}_1x+\boldsymbol{b}_1)+b_2
\end{align*}
with
\begin{align*}
\boldsymbol{W}_1:=\boldsymbol{1}_{N-1}\in\mathbb{R}^{N-1},\quad\boldsymbol{b}_1:=(-x_1,-x_2,\cdots,-x_{N-1})^{\top}\in\mathbb{R}^{N-1},\quad b_2:=y_1\in\mathbb{R}
\end{align*}
and
\begin{align*}
\boldsymbol{W}_2:=\left(\frac{y_2-y_1}{x_2-x_1},\frac{y_3-y_2}{x_3-x_2}-\frac{y_2-y_1}{x_2-x_1},\cdots,\frac{y_N-y_{N-1}}{x_N-x_{N-1}}-\frac{y_{N-1}-y_{N-2}}{x_{N-1}-x_{N-2}}\right)\in\mathbb{R}^{1\times (N-1)}.
\end{align*}
It follows that
\begin{align*}
f_{FF}^{(mmr)}(x_i)&=\boldsymbol{W}_2\sigma_R(\boldsymbol{W}_1x_i+\boldsymbol{b}_1)+b_2\\
&=\frac{y_2-y_1}{x_2-x_1}\cdot(x_i-x_1)+\sum_{j=2}^{i-1}\left(\frac{y_{j+1}-y_j}{x_{j+1}-x_j}-\frac{y_j-y_{j-1}}{x_j-x_{j-1}}\right)\cdot(x_i-x_j)+y_1.
\end{align*}
If we express the right-hand side of the above equation as a linear combination of \(\{y_j\}_{j\in[N]}\), we can verify that the coefficient of \(y_i\) is $1$ while the coefficients of all other \(y_j\) are $0$, which indicates that \eqref{memorization1} holds.

\end{proof}

\begin{proof}[Proof of Lemma \ref{Transformer memorization}]

The Transformer $\boldsymbol{T}^{(mmr)}:\mathbb{R}^{d\times n}\to\mathbb{R}^{1\times n}$ is defined as
\begin{align*}
\boldsymbol{T}^{(mmr)}:=\boldsymbol{\mathcal{F}}_{FF}^{(mmr)}\circ\boldsymbol{T}^{(cm)},
\end{align*}
where $\boldsymbol{T}^{(cm)}:\mathbb{R}^{d\times n}\to\mathbb{R}^{1\times n}$ is from Lemma \ref{contextual mapping} with size $\{(2,\max\{d,5\}),((3,3),(3,11))\times (n-1)\text{ times},(3,3),(2,5)\}$ and dimension vector $\begin{pmatrix}
d&d&5\cdot\boldsymbol{1}_{1\times n}&1
\end{pmatrix}$; $\boldsymbol{\mathcal{F}}_{FF}^{(mmr)}:\mathbb{R}^{1\times n}\to\mathbb{R}^{1\times n}$ is generated from $f_{FF}^{(mmr)}:\mathbb{R}\to\mathbb{R}$ in Lemma \ref{memorization} with width $nN-1$ and depth $2$.

We only prove the positional encoding case. For $i\in[N]$, denote
\begin{align*}
\boldsymbol{\widetilde{X}}^{(i)}:=\boldsymbol{{X}}^{(i)}+\boldsymbol{E}.
\end{align*}
By using the triangle inequality, we derive that for $i\in[N],k\in[n]$,
\begin{align}
(3k-1)r\leq\left\|\boldsymbol{E}_{:,k}\right\|_2-\left\|\boldsymbol{{X}}_{:,k}^{(i)}\right\|_2&\leq\left\|\boldsymbol{\widetilde{X}}_{:,k}^{(i)}\right\|_2=\left\|\boldsymbol{{X}}_{:,k}^{(i)}+\boldsymbol{E}_{:,k}\right\|_2\nonumber\\
&\leq\left\|\boldsymbol{\widetilde{X}}_{:,k}^{(i)}\right\|_2+\left\|\boldsymbol{E}_{:,k}\right\|_2\leq(3k+1)r,\label{mmr1}
\end{align}
which implies that $\boldsymbol{\widetilde{X}}_{:,k}^{(i)}$ and $\boldsymbol{\widetilde{X}}_{:,l}^{(j)}$ with $k\neq l$ are impossible to be identical. Moreover, the triangle inequlity also yields 
\begin{align}
\left\|\boldsymbol{\widetilde{X}}_{:,k}^{(i)}-\boldsymbol{\widetilde{X}}_{:,l}^{(j)}\right\|_2&=\left\|\left(\boldsymbol{{X}}_{:,k}^{(i)}-\boldsymbol{{X}}_{:,l}^{(j)}\right)+(k-l)p\boldsymbol{1}_{d\times 1}\right\|_2\nonumber\\
&\geq\left\{\begin{matrix}
\left\|\left(\boldsymbol{{X}}_{:,k}^{(i)}-\boldsymbol{{X}}_{:,l}^{(j)}\right)\right\|_2-\left\|(k-l)p\boldsymbol{1}_{d\times 1}\right\|_2\geq \phi,\text{ when }k= l \\
\left\|(k-l)p\boldsymbol{1}_{d\times 1}\right\|_2-\left\|\left(\boldsymbol{{X}}_{:,k}^{(i)}-\boldsymbol{{X}}_{:,l}^{(j)}\right)\right\|_2\geq r,\text{ when }k\neq l
\end{matrix}\right.\geq\phi.\label{mmr2}
\end{align}
We conclude from \eqref{mmr1}\eqref{mmr2} that $\left\{\boldsymbol{\widetilde{X}}^{(i)}\right\}_{i\in[N]}$ are $((3n+1)r,\phi)$-seperated. Applying Lemma \ref{contextual mapping} to $\left\{\boldsymbol{\widetilde{X}}^{(i)}\right\}_{i\in[N]}$, we have
\begin{itemize}
\item For any \( i \in [N] \) and \( k \in [n] \), \( \left|\boldsymbol{T}^{(cm)}\left(\boldsymbol{\widetilde{X}}^{(i)}\right)_{1,k}\right| \leq R \). 

\item For any \( i, j \in [N] \) and \( k, l \in [n] \) with $i\neq j$ or $k\neq l$, 
\begin{align*}
\left|\boldsymbol{T}^{(cm)}\left(\boldsymbol{\widetilde{X}}^{(i)}\right)_{1,k} - \boldsymbol{T}^{(cm)}\left(\boldsymbol{\widetilde{X}}^{(j)}\right)_{1,l}\right| \geq 2.
\end{align*} 
\end{itemize}
Applying Lemma \ref{memorization} to $\left\{\left(\boldsymbol{T}^{(cm)}\left(\boldsymbol{\widetilde{X}}^{(i)}\right)_{1,k},y_{1,k}^{(i)}\right)\right\}_{i\in[N],k\in[n]}$, we obtain
\begin{align*}
\boldsymbol{\mathcal{F}}_{FF}^{(mmr)}\left(\boldsymbol{T}^{(cm)}\left(\boldsymbol{\widetilde{X}}^{(i)}\right)_{1,k}\right)=y_{1,k}^{(i)},\quad i\in[N],k\in[n].
\end{align*}
The weight bound of $\boldsymbol{\mathcal{F}}_{FF}^{(mmr)}$ is $\max\{R,2B_y\}$. The weight bounds of $\boldsymbol{T}^{(cm)}$ are $B_{FF}=\max\left\{\sqrt{2} n^{2}(3n+1) N^{2} \sqrt{\pi d} r\phi^{-1}+1,\frac{3\sqrt{2}}{8}N^2 \sqrt{\pi n}\right\},B_{SA}=\frac{1}{2}\log(3\sqrt{ d} {\pi} n^{4}(3n+1) N^{4} r\phi^{-1})$. The weight bounds of $\boldsymbol{T}^{(mmr)}$ can be obtained by making a comparison between them.

\end{proof}

\subsection{Proof of Lemma \ref{contextual mapping}: Realizing Contextual Mapping by Transformers}\label{proof of lemma contextual mapping}

Our proof of Lemma \ref{contextual mapping} adapts the construction from \cite{kim2023provable}. Specifically, we first construct a feedforward block that maps the input sequence $  \boldsymbol{X}^{(i)}  $ to the token id $  \boldsymbol{x}^{(i)}  $ (Lemma \ref{token id}), followed by a Transformer that computes the sequence id $  z^{(i)}  $ (Lemma \ref{sequence id}). The Transformer in Lemma \ref{contextual mapping} is then obtained via a linear combination of the token id and the sequence id. The proofs of Lemmas \ref{token id} and \ref{sequence id} depend on the following three technical lemmas (Lemmas \ref{projection}–\ref{max}).

\begin{lemma}[\cite{park2021provable}, Lemma 13]\label{projection}
Let $d,N\in \mathbb{N}_{\geq1}$. Let $\left\{\boldsymbol{x}^{(i)}\right\}_{i\in[N]}\subset\mathbb{R}^d$. There exists a unit vector $\boldsymbol{u} \in \mathbb{R}^{d}$ such that for any $i,j\in[N]$,
\[
\frac{1}{N^2}\sqrt{\frac{8}{\pi d}} \left\|\boldsymbol{x}^{(i)} - \boldsymbol{x}^{(j)}\right\|_2 \leq \left|\boldsymbol{u}^{\top}\left(\boldsymbol{x}^{(i)} - \boldsymbol{x}^{(j)}\right)\right| \leq \left\|\boldsymbol{x}^{(i)} - \boldsymbol{x}^{(j)}\right\|_2.
\]
\end{lemma}

\begin{lemma}[\cite{kim2023provable}, Lemma E.3]\label{eliminate}
Let \( r' \in \mathbb{R} \). There exists a ReLU FNN \( f_{FF}^{(elm)} : \mathbb{R}^2 \to \mathbb{R} \) with depth $2$, width $4$ and weight bound $\max\{r',2\}$ such that for any $\boldsymbol{x}\in\mathbb{R}^2$,
\[
f_{FF}^{(elm)}(\boldsymbol{x}) = 
\begin{cases}
r', & \text{if } |x_1 - x_2| < \frac{1}{2}; \\
0, & \text{if } |x_1 - x_2| > 1.
\end{cases}
\]
\end{lemma}
The following lemma shows that the self-attention layer can approximate the maximum over the input sequence. This is an adaptation of \cite[Lemma E.2]{kim2023provable}, where the self-attention layer is defined without skip connections and the softmax includes a bias term.

\begin{lemma}\label{max}
Let \( n \in \mathbb{N}_{\geq1} \) and $r',P\in\mathbb{R}_{>1}$. There exists a self-attention layer $\boldsymbol{\mathcal{F}}_{SA}^{(max)}:\mathbb{R}^{3\times n}\to\mathbb{R}^{3\times n}$ with head number $1$, head size $3$ and weight bound $\frac{1}{2}\log(8n^{3/2}r'P)$ such that for any $\boldsymbol{x}=\begin{pmatrix}
x_1&x_2&\cdots&x_n    
\end{pmatrix}\in\mathbb{R}^{1\times n}$ satisfying
\begin{itemize}
\item \(|x_i| \leq 2r'\) for \(i \in [n]\);
\item \(x_i \leq x_{\max} - 2\) for \(i \in [n]\) with \(x_i \neq x_{\text{max}}\) ($x_{\max}:=\max_{i\in[n]}x_i$),
\end{itemize}
there holds
\begin{align*}
\boldsymbol{\mathcal{F}}_{SA}^{(max)}\left(\begin{pmatrix}
\boldsymbol{x}\\
\boldsymbol{1}_{1\times n}\\
\boldsymbol{0}_{1\times n}
\end{pmatrix}\right)=\begin{pmatrix}
\boldsymbol{x}\\
\boldsymbol{1}_{1\times n}\\
\widetilde{x}_{\max}\boldsymbol{1}_{1\times n}
\end{pmatrix},
\end{align*}
where \({x}_{\max} - \frac{1}{2P\sqrt{n}} \leq \widetilde{x}_{\max} \leq {x}_{\max}\). In particular, if $\boldsymbol{x}=\boldsymbol{0}_{1\times n}$, there holds
\begin{align*}
\boldsymbol{\mathcal{F}}_{SA}^{(max)}\left(\begin{pmatrix}
\boldsymbol{0}_{1\times n}\\
\boldsymbol{1}_{1\times n}\\
\boldsymbol{0}_{1\times n}
\end{pmatrix}\right)=\begin{pmatrix}
\boldsymbol{0}_{1\times n}\\
\boldsymbol{1}_{1\times n}\\
\boldsymbol{0}_{1\times n}
\end{pmatrix}.
\end{align*}
\end{lemma}
\begin{proof}
The proof is a modification of the proof of \cite[Lemma E.2]{kim2023provable}. The matrices in $\boldsymbol{\mathcal{F}}_{SA}^{(max)}$ are set as
\begin{align*}
\boldsymbol{W}_O:=\begin{pmatrix}
0\\0\\1
\end{pmatrix},\quad
\boldsymbol{W}_V:=\begin{pmatrix}
1&0&0
\end{pmatrix},
\quad\boldsymbol{W}_K:=\begin{pmatrix}
t&0&0\\0&0&0\\0&0&0
\end{pmatrix},
\quad\boldsymbol{W}_Q:=\begin{pmatrix}
0&1&0\\0&0&0\\0&0&0
\end{pmatrix},
\end{align*}
where $t>0$ is some parameter that will be defined later. Denoting
\begin{align*}
\boldsymbol{X}=\begin{pmatrix}
\boldsymbol{x}\\
\boldsymbol{1}_{1\times n}\\
\boldsymbol{0}_{1\times n}
\end{pmatrix},
\end{align*}
we have
\begin{align*}
\boldsymbol{W}_O\boldsymbol{W}_V\boldsymbol{X}=\begin{pmatrix}
0\\0\\1
\end{pmatrix}\begin{pmatrix}
1&0&0
\end{pmatrix}\begin{pmatrix}
\boldsymbol{x}\\
\boldsymbol{1}_{1\times n}\\
\boldsymbol{0}_{1\times n}
\end{pmatrix}=\begin{pmatrix}
\boldsymbol{0}_{1\times n}\\
\boldsymbol{0}_{1\times n}\\
\boldsymbol{x}
\end{pmatrix}
\end{align*}
and
\begin{align*}
\left(\boldsymbol{W}_K\boldsymbol{X}\right)^{\top}\left(\boldsymbol{W}_Q\boldsymbol{X}\right)
&=\left(\begin{pmatrix}
t&0&0\\0&0&0\\0&0&0
\end{pmatrix}\begin{pmatrix}
\boldsymbol{x}\\
\boldsymbol{1}_{1\times n}\\
\boldsymbol{0}_{1\times n}
\end{pmatrix}\right)^{\top}\left(\begin{pmatrix}
0&1&0\\0&0&0\\0&0&0
\end{pmatrix}\begin{pmatrix}
\boldsymbol{x}\\
\boldsymbol{1}_{1\times n}\\
\boldsymbol{0}_{1\times n}
\end{pmatrix}\right)\\
&=\begin{pmatrix}
t\boldsymbol{x}^{\top}&\boldsymbol{0}_{n\times 1}&\boldsymbol{0}_{n\times 1}
\end{pmatrix}\begin{pmatrix}
\boldsymbol{1}_{1\times n}\\
\boldsymbol{0}_{1\times n}\\
\boldsymbol{0}_{1\times n}
\end{pmatrix}
=t\boldsymbol{x}^{\top}\boldsymbol{1}_{1\times n}.
\end{align*}
Hence
\begin{align*}
\boldsymbol{\mathcal{F}}_{SA}^{(max)}\left(\begin{pmatrix}
\boldsymbol{x}\\
\boldsymbol{1}_{1\times n}\\
\boldsymbol{0}_{1\times n}
\end{pmatrix}\right)&=\boldsymbol{X}+\boldsymbol{W}_O\boldsymbol{W}_V\boldsymbol{X}\sigma_S\left(\left(\boldsymbol{W}_K\boldsymbol{X}\right)^{\top}\left(\boldsymbol{W}_Q\boldsymbol{X}\right)\right)\\
&=\begin{pmatrix}
\boldsymbol{x}\\
\boldsymbol{1}_{1\times n}\\
\boldsymbol{0}_{1\times n}
\end{pmatrix}+\begin{pmatrix}
\boldsymbol{0}_{1\times n}\\
\boldsymbol{0}_{1\times n}\\
\boldsymbol{x}
\end{pmatrix}\sigma_S\left(t\boldsymbol{x}^{\top}\boldsymbol{1}_{1\times n}\right)\\
&=\begin{pmatrix}
\boldsymbol{x}\\
\boldsymbol{1}_{1\times n}\\
\boldsymbol{0}_{1\times n}
\end{pmatrix}+\begin{pmatrix}
\boldsymbol{0}_{1\times n}\\
\boldsymbol{0}_{1\times n}\\
\boldsymbol{x}
\end{pmatrix}\sigma_S\left(t\boldsymbol{x}^{\top}\right)\boldsymbol{1}_{1\times n}\\
&=\begin{pmatrix}
\boldsymbol{x}\\
\boldsymbol{1}_{1\times n}\\
\left[\boldsymbol{x}\sigma_S\left(t\boldsymbol{x}^{\top}\right)\right]\boldsymbol{1}_{1\times n}
\end{pmatrix}.
\end{align*}
Define
\begin{align}\label{max1}
\widetilde{x}_{\max}:=\boldsymbol{x}\sigma_S\left(t\boldsymbol{x}^{\top}\right)=  \frac{\sum_{i=1}^nx_i\exp(tx_i)}{\sum_{i=1}^n \exp(tx_i)}.
\end{align}
Since \(\widetilde{x}_{\max}\) is a convex combination of \(\{x_i\}_{i\in[n]}\), it is easy to see that \(x_{\max}\) upper bounds \(\widetilde{x}_{\max}\). It suffices to find \(t\) that satisfies the lower bound condition. We lower bound the softmax weights on \(x_{\max}\) as
\begin{align*}
p_{\max} &:=\frac{\sum_{i:x_i=x_{\max}}\exp(tx_i)}{\sum_{i=1}^{ n}\exp(tx_i)} \\
&=\frac{\sum_{i:x_i=x_{\max}}\exp(tx_i)}{\sum_{i:x_i=x_{\max}}\exp(tx_i)+\sum_{i:x_i\neq x_{\max}}\exp(tx_i)} \\
&\geq\frac{\sum_{i:x_i=x_{\max}}\exp(tx_{\max})}{\sum_{i:x_i=x_{\max}}\exp(tx_{\max})+\sum_{i:x_i\neq x_{\max}}\exp(t(x_{\max}-2))} \\
&=\frac{n_{\max}}{n_{\max}+(n-n_{\max})\exp(-2t)}\\
&=\frac{1}{1+(\frac{n}{n_{\max}}-1)\exp(-2t)},
\end{align*}
where \(n_{\max}:=\left|\{i:x_i=x_{\max}\}\right|\). Choosing \(t=\frac{1}{2}\log(8n^{3/2}r'P)\), we have
\begin{align*}
p_{\max} 
&\geq\frac{1}{1+(\frac{n}{n_{\max}}-1)\frac{1}{8n^{3/2}r'P}} 
\geq\frac{1}{1+\frac{1}{8r'P\sqrt{n}}}.
\end{align*}
Now, we can lower bound \(\widetilde{x}_{\max}\) as
\begin{align*}
\widetilde{x}_{\max} &\geq x_{\max}p_{\max}-2r'(1-p_{\max}) \\
&=x_{\max}-(x_{\max}+2r^{\prime})(1-p_{\max}) \\
&\geq x_{\max}-4r^{\prime}\left(1-p_{\max}\right) \\
&\geq x_{\max}-4r^{\prime}\left(1-\frac{1}{1+\frac{1}{8r'P\sqrt{n}}} \right) \\
&=x_{\max}-\frac{\frac{1}{2P\sqrt{n}}}{1+\frac{1}{8r'P\sqrt{n}}} \\
&\geq x_{\max}-\frac{1}{2P\sqrt{n}}.
\end{align*}
When $\boldsymbol{x}=\boldsymbol{0}_{1\times n}$, the definition \eqref{max1} shows that $\widetilde{x}_{\max}=0$.

\end{proof}

\begin{lemma}\label{token id}
Let \( d,n,N \in \mathbb{N}_{\geq1} \). Let $r,\phi\in\mathbb{R}_{>0}$. Let \( \boldsymbol{X}^{(1)}, \cdots, \boldsymbol{X}^{(N)} \in \mathbb{R}^{d \times n} \) be a set of \( N \) input sequences that are tokenwise \((r, \phi)\)-separated. Denote \( r' = \frac{\sqrt{2}}{2} n^{2} N^{2} \sqrt{\pi d} r\phi^{-1} \). There exists a feedforward block $\boldsymbol{\mathcal{F}}_{FF}^{(pjt)}:\mathbb{R}^{d\times n}\to\mathbb{R}^{4\times n}$ with depth $2$, width $\max\{d,4\}$ and weight bound $\frac{\sqrt{2}}{2} n^{2} N^{2} \sqrt{\pi d}r\phi^{-1}$ such that for $i\in[N]$,
\begin{align*}
\boldsymbol{\mathcal{F}}_{FF}^{(pjt)}\left(\boldsymbol{X}^{(i)}\right) 
&= \begin{pmatrix}
\boldsymbol{x}^{(i)}\\ 
\boldsymbol{1}_{1\times n}\\
\boldsymbol{0}_{1\times n}\\ \boldsymbol{0}_{1\times n}
\end{pmatrix},
\end{align*}
where $\left\{\boldsymbol{x}^{(i)}\right\}_{i\in[N]}\subset\mathbb{R}^{1\times n}$ are non-negative and tokenwise \((2r', 2)\)-separated. Moreover, for \( i, j \in [N] \) and \( k, l \in [n] \),
\begin{align*}
{x}_{1,k}^{(i)} = {x}_{1,l}^{(j)}
\text{ if and only if }   \boldsymbol{X}^{(i)}_{:, k} = \boldsymbol{X}^{(j)}_{:, l}. 
\end{align*}
Furthermore, if $\boldsymbol{X}\in\mathbb{R}^{d\times n}$ satisfies $\left\|\boldsymbol{X}_{:,k}\right\|_2\leq r$ for all $k\in[n]$, then $\left|\boldsymbol{\mathcal{F}}_{FF}^{(pjt)}\left(\boldsymbol{X}\right)_{1,k} \right|\leq2r'$ for all $k\in[n]$.
\end{lemma}
\begin{proof}
Recall the definition of the vocabulary 
\[
\mathcal{V} = \bigcup_{i \in [N]} \mathcal{V}^{(i)} = \left\{\boldsymbol{v} \in \mathbb{R}^{d} : \boldsymbol{v} = \boldsymbol{X}^{(i)}_{:, k} \text{ for some } i \in [N], k \in [n]\right\}.
\]
Note that $|\mathcal{V}| \leq n N$. We use Lemma \ref{projection} on $\mathcal{V}$ to find a unit vector ${\boldsymbol{u}}'$ such that
\begin{align}\label{projection1}
\frac{1}{n^{2} N^{2}} \sqrt{\frac{8}{\pi d}} \|\boldsymbol{v} - \boldsymbol{v}^{\prime}\|_2 
\leq \frac{1}{|\mathcal{V}|^{2}} \sqrt{\frac{8}{\pi d}} \|\boldsymbol{v} - \boldsymbol{v}^{\prime}\|_2 
\leq \left| {\boldsymbol{u}}'^{\top} (\boldsymbol{v} - \boldsymbol{v}^{\prime}) \right| 
\leq \|\boldsymbol{v} - \boldsymbol{v}^{\prime}\|_2
\end{align}
for every $\boldsymbol{v}, \boldsymbol{v}^{\prime} \in \mathcal{V}$. Let $\boldsymbol{u} := S {\boldsymbol{u}}'$ with $S = \frac{\sqrt{2}}{2} n^{2} N^{2} \sqrt{\pi d} \phi^{-1}$. Then $\boldsymbol{\mathcal{F}}_{FF}^{(pjt)}:\mathbb{R}^{d\times n}\to\mathbb{R}^{4\times n}$ is defined as
\begin{align*}
\boldsymbol{\mathcal{F}}_{FF}^{(pjt)}(\boldsymbol{X}) 
&:= \begin{pmatrix}
\sigma_R\left(\boldsymbol{{u}}^{\top} \boldsymbol{X} + r'\mathbf{1}_{1\times n}\right)\\ 
\boldsymbol{1}_{1\times n}\\
\boldsymbol{0}_{1\times n}\\ \boldsymbol{0}_{1\times n}
\end{pmatrix}.
\end{align*}
Let
\begin{align}\label{projection2}
\boldsymbol{x}^{(i)}:=\sigma_R\left(\boldsymbol{{u}}^{\top} \boldsymbol{X}^{(i)} + r'\mathbf{1}_{1\times n}\right)\in\mathbb{R}^{1\times n},\quad i\in[N].
\end{align}
For any $i\in[N],k\in[n]$, since $\boldsymbol{X}^{(i)}_{:, k}\in\mathcal{V}$, according to\eqref{projection1}, we have
\begin{align*}
\left| \boldsymbol{u}^{\top} \boldsymbol{X}^{(i)}_{:, k} \right| 
=\left| S{\boldsymbol{u}}'^{\top} \boldsymbol{X}^{(i)}_{:, k} \right| 
\leq S \left\|\boldsymbol{X}^{(i)}_{:, k}\right\|_2 
\leq Sr \leq r^{\prime}.
\end{align*}
Hence we can remove the ReLU activation in \eqref{projection2}:
\begin{align*}
\boldsymbol{x}^{(i)}=\boldsymbol{{u}}^{\top} \boldsymbol{X}^{(i)} + r'\mathbf{1}_{1\times n}.
\end{align*}
It follows that for any $i,j\in[N],k,l\in[n]$,
\begin{align*}
\left|{x}_{1,k}^{(i)}\right|\leq\left|\boldsymbol{{u}}^{\top} \boldsymbol{X}_{:,k}^{(i)} + r'\right|\leq2r'
\end{align*}
and
\begin{align*}
\left| {x}^{(i)}_{1, k} - {x}^{(j)}_{1, l} \right| 
&= \left| \boldsymbol{u}^{\top} \left(\boldsymbol{X}^{(i)}_{:, k} - \boldsymbol{X}^{(j)}_{:, l}\right) \right| \\
&= S \left| {\boldsymbol{u}}'^{\top} \left(\boldsymbol{X}^{(i)}_{:, k} - \boldsymbol{X}^{(j)}_{:, l}\right) \right| \\
&\geq \frac{S}{n^{2} N^{2}} \sqrt{\frac{8}{\pi d}} \left\|\boldsymbol{X}^{(i)}_{:, k} - \boldsymbol{X}^{(j)}_{:, l}\right\|_2  \\
&\geq \frac{S}{n^{2} N^{2}} \sqrt{\frac{8}{\pi d}} \phi \geq 2,
\end{align*}
where in the third step we make use of \eqref{projection1}. The above inequality also implies
\[
{x}_{1,k}^{(i)} = {x}_{1,l}^{(j)}
\]
if and only if \( \boldsymbol{X}^{(i)}_{:, k} = \boldsymbol{X}^{(j)}_{:, l} \).
\end{proof}

\begin{lemma}\label{sequence id}
Let \( N, n\in \mathbb{N}_{\geq1} \) and \( r' \in \mathbb{R}_{>0} \). Let \( \boldsymbol{x}^{(1)}, \cdots, \boldsymbol{x}^{(N)} \in \mathbb{R}^{1\times n} \) be a set of \( N \) input sequences that are non-negative and tokenwise \((2r', 2)\)-separated. There exists a Transformer $\boldsymbol{T}^{(sid)}:\mathbb{R}^{4\times n}\to\mathbb{R}^{1\times n}$ with size $\{(2,4),((2,3),(3,10))\times (n-1)\text{ times},(2,3),(2,4)\}$ and dimension vector $\begin{pmatrix}
4\cdot\boldsymbol{1}_{1\times (n+2)}&1
\end{pmatrix}$ such that
\[
\boldsymbol{T}^{(sid)} \left( 
\begin{pmatrix}
\boldsymbol{x}^{(i)}\\ 
\boldsymbol{1}_{1\times n}\\
\boldsymbol{0}_{1\times n}\\ \boldsymbol{0}_{1\times n}
\end{pmatrix}
\right) =
z^{(i)}\boldsymbol{1}_{1\times n},\quad i\in[N],
\]
where \( \left\{z^{(i)}\right\}_{i \in [N]} \) satisfies the following conditions: 
\begin{itemize}
\item For any \( i \in [N] \), \( \left|z^{(i)}\right| \leq \frac{3\sqrt{2\pi}}{4}nN^2 r'+\frac{1}{2} \).

\item For any \( i, j \in [N] \) such that \( \boldsymbol{x}^{(i)} \neq \boldsymbol{x}^{(j)} \) up to permutations, \( \left|z^{(i)}-z^{(j)}\right| \geq 2 \).
\end{itemize}
The weight bounds of $\boldsymbol{T}^{(sid)}$ are $B_{FF}=\max\left\{2r',\frac{3\sqrt{2}}{8}N^2 \sqrt{\pi n}\right\},B_{SA}=\frac{1}{2}\log(3\sqrt{2\pi}n^{2}N^2r')$. Furthermore, if the first component of the input vector is replaced with any $\boldsymbol{x}\in\mathbb{R}^{1\times n}$ and the components of $\boldsymbol{x}$ are bounded by $2r'$, then the components of the output of $\boldsymbol{T}^{(sid)}$ are bounded by $\frac{3\sqrt{2\pi}}{4}nN^2 r'+\frac{1}{2}$.

\end{lemma}
\begin{proof}

For $i\in[N]$, let $n_i\in\mathbb{N}_{\geq1}$ be the number of components of $\boldsymbol{x}^{(i)}$ that take different values. We define a new sequence \(\{\bar{\boldsymbol{x}}^{(i)}\}_{i\in[N]}\subset
\mathbb{R}^{1\times n}\) constructed as follows: for \(1 \leq j \leq n_i\), \(\bar{x}_j^{(i)}\) is taken as the \(j\)-th largest component of \(\boldsymbol{x}^{(i)}\); for \(n_i < j \leq n\), \(\bar{x}_j^{(i)}\) is set to $0$. Then according to Lemma \ref{projection}, we can find a unit vector $\boldsymbol{{w}}'\in\mathbb{R}^n$ such that for any \(i, j \in [N]\),
\begin{align}\label{cm-1}
\frac{1}{N^2} \sqrt{\frac{8}{\pi n}} \left\|\boldsymbol{\bar{x}}^{(i)} - \boldsymbol{\bar{x}}^{(j)}\right\|_2 \leq \left| \left( \boldsymbol{\bar{x}}^{(i)} - \boldsymbol{\bar{x}}^{(j)} \right)\boldsymbol{{w}}' \right| \leq \left\|\boldsymbol{\bar{x}}^{(i)} - \boldsymbol{\bar{x}}^{(j)}\right\|_2.
\end{align}
Define 
\begin{align}\label{cm0}
\boldsymbol{{w}} := P \boldsymbol{{w}}'  
\end{align}
with \( P = \frac{3\sqrt{2}}{8}N^2 \sqrt{\pi n} \). 

By Lemma \ref{max}, there exists a self-attention layer $\boldsymbol{\mathcal{F}}_{SA}^{(max)}:\mathbb{R}^{3\times n}\to\mathbb{R}^{3\times n}$ with head number $1$, head size $3$ and weight bound $\frac{1}{2}\log(8n^{3/2}r'P)$ such that for any $\boldsymbol{x}\in\mathbb{R}^{1\times n}$ satisfying the condition in Lemma \ref{max}, there holds
\begin{align*}
\boldsymbol{\mathcal{F}}_{SA}^{(max)}\left(\begin{pmatrix}
\boldsymbol{x}\\
\boldsymbol{1}_{1\times n}\\
\boldsymbol{0}_{1\times n}
\end{pmatrix}\right)=\begin{pmatrix}
\boldsymbol{x}\\
\boldsymbol{1}_{1\times n}\\
\widetilde{x}_{\max}\boldsymbol{1}_{1\times n}
\end{pmatrix},
\end{align*}
where \({x}_{\max} - \frac{1}{2P\sqrt{n}} \leq \widetilde{x}_{\max} \leq {x}_{\max}\). It follows that, according to Lemma \ref{parallel}, there exists a self attention layer $\boldsymbol{\mathcal{F}}_{SA}^{(sid)}:\mathbb{R}^{4\times n}\to\mathbb{R}^{4\times n}$ such that for any $\boldsymbol{x},\boldsymbol{z}\in\mathbb{R}^{1\times n}$ with $\boldsymbol{x}$ satifsying the condition in Lemma \ref{max},
\begin{align}\label{cm1}
\boldsymbol{\mathcal{F}}_{SA}^{(sid)}\left(\begin{pmatrix}
\boldsymbol{x}\\
\boldsymbol{1}_{1\times n}\\
\boldsymbol{0}_{1\times n}\\
\boldsymbol{z}
\end{pmatrix}\right):=\begin{pmatrix}
\boldsymbol{\mathcal{F}}_{SA}^{(max)}\begin{pmatrix}
\boldsymbol{x}\\
\boldsymbol{1}_{1\times n}\\
\boldsymbol{0}_{1\times n}
\end{pmatrix}\\
\boldsymbol{\mathcal{F}}_{SA}^{(idt)}\begin{pmatrix}
\boldsymbol{z}
\end{pmatrix}
\end{pmatrix}
=\begin{pmatrix}
\boldsymbol{x}\\
\boldsymbol{1}_{1\times n}\\
\widetilde{x}_{\max}\boldsymbol{1}_{1\times n}\\
\boldsymbol{z}
\end{pmatrix},
\end{align}
where $\boldsymbol{\mathcal{F}}_{SA}^{(idt)}:\mathbb{R}^{1\times n}\to\mathbb{R}^{1\times n}$ is the self-attention layer from Lemma \ref{SA-identity} with head number $1$, head size $1$ and weight bound $1$. According to Lemma \ref{parallel}, $\boldsymbol{\mathcal{F}}_{SA}^{(sid)}$ is of head number $2$, head size $3$ and weight bound $\frac{1}{2}\log(8n^{3/2}r'P)$.

For $l\in[n-1]$, \(\boldsymbol{\mathcal{F}}_{FF,l}^{(sid)}: \mathbb{R}^{4 \times n} \to \mathbb{R}^{4 \times n}\) are defined as
\begin{align}\label{cm2}
\boldsymbol{\mathcal{F}}_{FF,l}^{(sid)}\left(\begin{pmatrix}
\boldsymbol{x}\\
\boldsymbol{1}_{1\times n}\\
\boldsymbol{y}\\
\boldsymbol{z}
\end{pmatrix}\right):=\begin{pmatrix}
\sigma_R\left(\boldsymbol{\mathcal{F}}_{FF}^{(idt)}(\boldsymbol{x})-2\boldsymbol{\mathcal{F}}_{FF}^{(elm)}\left(\begin{pmatrix}
\boldsymbol{x}\\ 
\boldsymbol{y}
\end{pmatrix}\right)\right)\\
\boldsymbol{1}_{1\times n}\\
\boldsymbol{0}_{1\times n}\\
\boldsymbol{\mathcal{F}}_{FF}^{(idt)}(\boldsymbol{z}+w_l\boldsymbol{y})
\end{pmatrix}
=\begin{pmatrix}
\boldsymbol{x}'\\
\boldsymbol{1}_{1\times n}\\
\boldsymbol{0}_{1\times n}\\
\boldsymbol{z}+w_l\boldsymbol{y}
\end{pmatrix},\quad\boldsymbol{x},\boldsymbol{y},\boldsymbol{z}\in\mathbb{R}^{1\times n},
\end{align}
with $\boldsymbol{x}':=\begin{pmatrix}
x_1'&\cdots&x_n'
\end{pmatrix}$ having components
\begin{align}\label{cm3}
x_i' = \begin{cases}
\sigma_R(x_i - 2r'), & \text{if } |x_i - y_i| < \frac{1}{2}; \\
\sigma_R(x_i), & \text{if } |x_i - y_i| > 1.
\end{cases}
\end{align}
Here, $\boldsymbol{\mathcal{F}}_{FF}^{(idt)}:\mathbb{R}^{1\times n}\to\mathbb{R}^{1\times n}$ is generated from the ReLU FNN $f_{FF}^{(idt)}$ in Lemma \ref{identity} that implements the identity mapping and hence has depth $2$, width $2$ and weight bound $1$; $\boldsymbol{\mathcal{F}}_{FF}^{(elm)}:\mathbb{R}^{2\times n}\to\mathbb{R}^{1\times n}$ is generated from the ReLU FNN $f_{FF}^{(elm)}$ in Lemma \ref{eliminate} and hence has depth $2$, width $4$ and weight bound $\max\{r',2\}$; $w_l$ is the $l$-th component of $\boldsymbol{w}$ defined in \eqref{cm0}. Therefore, $\boldsymbol{\mathcal{F}}_{FF,l}^{(sid)}$ is of depth $3$, width $10$ and weight bound $\max\{2r',P\}$.

\(\boldsymbol{\mathcal{F}}_{FF,n}^{(sid)}: \mathbb{R}^{4 \times n} \to \mathbb{R}^{1 \times n}\) is defined as
\begin{align*}
\boldsymbol{\mathcal{F}}_{FF,n}^{(sid)}\left(\begin{pmatrix}
\boldsymbol{x}\\
\boldsymbol{1}_{1\times n}\\
\boldsymbol{y}\\
\boldsymbol{z}
\end{pmatrix}\right):=
\boldsymbol{\mathcal{F}}_{FF}^{(idt)}(\boldsymbol{z}+w_l\boldsymbol{y})
=\boldsymbol{z}+w_n\boldsymbol{y},\quad\boldsymbol{x},\boldsymbol{y},\boldsymbol{z}\in\mathbb{R}^{1\times n}.
\end{align*}
It can be seen that $\boldsymbol{\mathcal{F}}_{FF,n}^{(sid)}$ is of depth $2$, width $4$ and weight bound $P$.

Now, our Transformer $\boldsymbol{T}^{(sid)}:\mathbb{R}^{4\times n}\to\mathbb{R}^{1\times n}$ is defined as
\begin{align*}
\boldsymbol{T}^{(sid)}:=\boldsymbol{\mathcal{F}}_{FF,n}^{(sid)}\circ\boldsymbol{\mathcal{F}}_{SA}^{(sid)}\circ\boldsymbol{\mathcal{F}}_{FF,n-1}^{(sid)}\circ\boldsymbol{\mathcal{F}}_{SA}^{(sid)}\circ\cdots\circ\boldsymbol{\mathcal{F}}_{FF,1}^{(sid)}\circ\boldsymbol{\mathcal{F}}_{SA}^{(sid)}\circ\boldsymbol{\mathcal{F}}_{FF}^{(idt)}.
\end{align*}
For $i\in[N],l\in[n-1]$, letting 
\begin{align*}
\boldsymbol{Z}_l^{(i)} := \boldsymbol{\mathcal{F}}_{FF,l}^{(sid)} \circ \boldsymbol{\mathcal{F}}_{SA}^{(sid)} \circ \boldsymbol{\mathcal{F}}_{FF,l-1}^{(sid)}\circ \boldsymbol{\mathcal{F}}_{SA}^{(sid)} \circ  \cdots \circ\boldsymbol{\mathcal{F}}_{FF,1}^{(sid)} \circ \boldsymbol{\mathcal{F}}_{SA}^{(sid)}\circ\boldsymbol{\mathcal{F}}_{FF}^{(idt)}  \left( \begin{pmatrix}
\boldsymbol{x}^{(i)}\\ 
\boldsymbol{1}_{1\times n}\\
\boldsymbol{0}_{1\times n}\\ \boldsymbol{0}_{1\times n}
\end{pmatrix} \right) \in \mathbb{R}^{4 \times n}
\end{align*}
be the output of the \(l\)-th step of the $i$-th sample, we show inductively that
\begin{align*}
\boldsymbol{Z}_l^{(i)}=\begin{pmatrix}
\boldsymbol{x}_l^{(i)}\\
\boldsymbol{1}_{1\times n}\\
\boldsymbol{0}_{1\times n}\\
\left(\sum_{j=1}^{\min \{l, n_i\}} w_j \widetilde{\bar{x}}_{j}^{(i)}\right)\boldsymbol{1}_{1\times n}
\end{pmatrix},
\end{align*} 
where \(\widetilde{\bar{x}}_{j}^{(i)}\) is \(\frac{1}{2P\sqrt{n}}\)-approximation of \({\bar{x}}_{j}^{(i)}\) and \( \boldsymbol{x}_l^{(i)} \) is generated from \( \boldsymbol{x}^{(i)} \) in the following way: when $l\leq{n_i}$, the first \( l \) largest components in \( \boldsymbol{x}^{(i)} \) (where components with the same value are considered identical) are replaced by 0, while the remaining components remain unchanged; when $l>{n_i}$, $\boldsymbol{x}_l^{(i)}:=\boldsymbol{0}_{1\times n}$. It is worth to note that $\boldsymbol{x}_l^{(i)}$ satisfies the condition in Lemma \ref{max}.

For \(l = 0\), the conclusion holds for the input obviously. Suppose that the conclusion holds for \(l = {l'} - 1\). When \({l'} \leq n_i\), since the largest value of $\boldsymbol{x}_{l'-1}^{(i)}$ is ${\bar{x}}_{l'}^{(i)}$, by the induction hypothesis and \eqref{cm1}\eqref{cm2}, we have
\begin{align*}
\boldsymbol{Z}_{l'}^{(i)}&=\boldsymbol{\mathcal{F}}_{FF,l'}^{(sid)} \circ \boldsymbol{\mathcal{F}}_{SA}^{(sid)} \left( \boldsymbol{Z}_{l'-1}^{(i)} \right) \\
&= \boldsymbol{\mathcal{F}}_{FF,l'}^{(sid)} \circ \boldsymbol{\mathcal{F}}_{SA,l'}^{(sid)}\left( \begin{pmatrix}
\boldsymbol{x}_{l'-1}^{(i)}\\
\boldsymbol{1}_{1\times n}\\
\boldsymbol{0}_{1\times n}\\
\left(\sum_{j=1}^{\min \{l'-1, n_i\}} w_j \widetilde{\bar{x}}_{j}^{(i)}\right)\boldsymbol{1}_{1\times n}
\end{pmatrix} \right) \\
&= \boldsymbol{\mathcal{F}}_{FF,l'}^{(sid)} \left( \begin{pmatrix}
\boldsymbol{x}_{l'-1}^{(i)}\\
\boldsymbol{1}_{1\times n}\\
\widetilde{\bar{x}}_{l'}^{(i)}\boldsymbol{1}_{1\times n}\\
\left(\sum_{j=1}^{\min \{l'-1, n_i\}} w_j \widetilde{\bar{x}}_{j}^{(i)}\right)\boldsymbol{1}_{1\times n}
\end{pmatrix} \right) \\
&= \begin{pmatrix}
\left(\boldsymbol{Z}_{l'}^{(i)}\right)_{1,:}\\
\boldsymbol{1}_{1\times n}\\
\boldsymbol{0}_{1\times n}\\
\left(\sum_{j=1}^{\min \{l', n_i\}} w_j \widetilde{\bar{x}}_{j}^{(i)}\right)\boldsymbol{1}_{1\times n}
\end{pmatrix}.
\end{align*}

Since $\left\{\boldsymbol{x}^{(i)}\right\}_{i\in[N]}$ is $(2r',2)$-separated and all components are non-negative, we know that $\left\{\boldsymbol{x}_{l'-1}^{(i)}\right\}_{i\in[N]}$ is also $(2r',2)$-separated with all components being non-negative from its definition. Therefore, the $\frac{1}{2P\sqrt{n}}$-approximations $\widetilde{\bar{x}}_{l'}^{(i)}$ differ from the maximum component ${\bar{x}}_{l'}^{(i)}$ of $\boldsymbol{x}_{l'-1}^{(i)}$ by less than $\frac{1}{2}$, while differing from all other components of $\boldsymbol{x}_{l'-1}^{(i)}$ by more than $1$. According to \eqref{cm3}, we obtain
\begin{align*}
\left(\boldsymbol{Z}_{l'}^{(i)}\right)_{1,:} =\boldsymbol{x}_{l'}^{(i)}.
\end{align*}
When \({l'} > n_i\), noticing $\boldsymbol{x}_{l'-1}^{(i)}=\boldsymbol{0}_{1\times n}$, by the induction hypothesis and \eqref{cm1}\eqref{cm2}, we have
\begin{align*}
\boldsymbol{Z}_{l'}^{(i)}&=\boldsymbol{\mathcal{F}}_{FF,l'}^{(sid)} \circ \boldsymbol{\mathcal{F}}_{SA}^{(sid)} \left( \boldsymbol{Z}_{l'-1}^{(i)} \right) \\
&= \boldsymbol{\mathcal{F}}_{FF,l'}^{(sid)} \circ \boldsymbol{\mathcal{F}}_{SA,l'}^{(sid)} \left( \begin{pmatrix}
\boldsymbol{0}_{1\times n}\\
\boldsymbol{1}_{1\times n}\\
\boldsymbol{0}_{1\times n}\\
\left(\sum_{j=1}^{\min \{l'-1, n_i\}} w_j \widetilde{\bar{x}}_{j}^{(i)}\right)\boldsymbol{1}_{1\times n}
\end{pmatrix} \right) \\
&= \boldsymbol{\mathcal{F}}_{FF,l'}^{(sid)}  \left( \begin{pmatrix}
\boldsymbol{0}_{1\times n}\\
\boldsymbol{1}_{1\times n}\\
\boldsymbol{0}_{1\times n}\\
\left(\sum_{j=1}^{\min \{l'-1, n_i\}} w_j \widetilde{\bar{x}}_{j}^{(i)}\right)\boldsymbol{1}_{1\times n}
\end{pmatrix} \right) \\
&= \begin{pmatrix}
\boldsymbol{0}_{1\times n}\\
\boldsymbol{1}_{1\times n}\\
\boldsymbol{0}_{1\times n}\\
\left(\sum_{j=1}^{\min \{l'-1, n_i\}} w_j \widetilde{\bar{x}}_{j}^{(i)}\right)\boldsymbol{1}_{1\times n}
\end{pmatrix}\\
&=\begin{pmatrix}
\boldsymbol{x}_{l'}^{(i)}\\
\boldsymbol{1}_{1\times n}\\
\boldsymbol{0}_{1\times n}\\
\left(\sum_{j=1}^{\min \{l', n_i\}} w_j \widetilde{\bar{x}}_{j}^{(i)}\right)\boldsymbol{1}_{1\times n}
\end{pmatrix}.
\end{align*}
Thus, the induction is completed and we have
\begin{align*}
\boldsymbol{Z}_{n-1}^{(i)}=\begin{pmatrix}
\boldsymbol{x}_{n-1}^{(i)}\\
\boldsymbol{1}_{1\times n}\\
\boldsymbol{0}_{1\times n}\\
\left(\sum_{j=1}^{\min \{n-1, n_i\}} w_j \widetilde{\bar{x}}_{j}^{(i)}\right)\boldsymbol{1}_{1\times n}
\end{pmatrix}.
\end{align*} 
From the definition of \(\boldsymbol{\mathcal{F}}_{FF,n}^{(sid)}\), it can be seen that \(\boldsymbol{Z}_{n}^{(i)}\) only outputs the aggregated summation information. Therefore, through an analysis similar to the one above, we obtain
\begin{align*}
\boldsymbol{Z}_{n}^{(i)}=\boldsymbol{\mathcal{F}}_{FF,n}^{(sid)} \circ \boldsymbol{\mathcal{F}}_{SA}^{(sid)} \left( \boldsymbol{Z}_{n-1}^{(i)} \right)=\left(\sum_{j=1}^{\min \{n, n_i\}} w_j \widetilde{\bar{x}}_{j}^{(i)}\right)\boldsymbol{1}_{1\times n}=z^{(i)}\boldsymbol{1}_{1\times n},
\end{align*}
where $z^{(i)}:=\boldsymbol{\widetilde{\bar{x}}}^{(i)}\boldsymbol{w}$, with each component of \(\boldsymbol{\widetilde{\bar{x}}}^{(i)}\) \(\frac{1}{2P\sqrt{n}}\)-approximating the corresponding component of \(\boldsymbol{\bar{x}}^{(i)}\). We now check that \(\left\{z^{(i)}\right\}_{i \in [N]}\) are \(\left(\frac{3\sqrt{2\pi}}{4}nN^2 r'+\frac{1}{2}, 2\right)\)-separated. Let \( i, j \in [N] \) with \( i \neq j \), noting that $\|\boldsymbol{w}'\|_2=1$ and $\left\{\boldsymbol{{\bar{x}}}^{(i)}\right\}_{i\in[N]}$ are tokenwise $(2r',2)$-seperated, we have
\begin{align*}
\left| {z}^{(i)} \right| 
&= \left|\boldsymbol{\widetilde{\bar{x}}}^{(i)}\boldsymbol{{w}} \right| 
= P \left| \boldsymbol{\widetilde{\bar{x}}}^{(i)}\boldsymbol{{w}}' \right| 
\leq P\left\|\boldsymbol{\widetilde{\bar{x}}}^{(i)} \right\|_2 \\
&\leq P\left\|\boldsymbol{\widetilde{\bar{x}}}^{(i)}-\boldsymbol{\bar{x}}^{(i)} \right\|_2+P\left\|\boldsymbol{\bar{x}}^{(i)} \right\|_2\\
&\leq P\cdot\frac{\sqrt{n}}{2P\sqrt{n}}+P\cdot 2r' \sqrt{n}\\
&\leq \frac{3\sqrt{2\pi}}{4}nN^2 r'+\frac{1}{2}
\end{align*}
and
\begin{align*}
\left| {z}^{(i)} - {z}^{(j)} \right| 
&= \left| \left(\boldsymbol{\widetilde{\bar{x}}}^{(i)} - \boldsymbol{\widetilde{\bar{x}}}^{(j)}\right) \boldsymbol{w} \right| \\
&=P \left|  \left(\boldsymbol{\widetilde{\bar{x}}}^{(i)} - \boldsymbol{\widetilde{\bar{x}}}^{(j)}\right)\boldsymbol{w}' \right|  \\
&\geq P \left|  \left(\boldsymbol{{\bar{x}}}^{(i)} - \boldsymbol{{\bar{x}}}^{(j)}\right)\boldsymbol{w}' \right|-P \left|\left(\boldsymbol{{\bar{x}}}^{(i)} - \boldsymbol{{\bar{x}}}^{(j)}\right)\boldsymbol{w}'-  \left(\boldsymbol{\widetilde{\bar{x}}}^{(i)} - \boldsymbol{\widetilde{\bar{x}}}^{(j)}\right)\boldsymbol{w}' \right|\\
&\geq  P \left|  \left(\boldsymbol{{\bar{x}}}^{(i)} - \boldsymbol{{\bar{x}}}^{(j)}\right)\boldsymbol{w}' \right|-P \left|\left(\boldsymbol{{\bar{x}}}^{(i)} - \boldsymbol{\widetilde{\bar{x}}}^{(i)}\right)\boldsymbol{w}'\right|-  P\left|\left(\boldsymbol{{\bar{x}}}^{(j)} - \boldsymbol{\widetilde{\bar{x}}}^{(j)}\right)\boldsymbol{w}' \right|  \\
&\geq\frac{P}{N^2} \sqrt{\frac{8}{\pi n}} \left\| \boldsymbol{\bar{x}}^{(i)} - \boldsymbol{\bar{x}}^{(j)}\right\|_2-P \left\|\boldsymbol{{\bar{x}}}^{(i)} - \boldsymbol{\widetilde{\bar{x}}}^{(i)}\right\|_2-  P\left\|\boldsymbol{{\bar{x}}}^{(j)} - \boldsymbol{\widetilde{\bar{x}}}^{(j)} \right\|_2\\
&\geq \frac{P}{N^2} \sqrt{\frac{8}{\pi n}} \cdot 2-P\cdot\frac{\sqrt{n}}{2P\sqrt{n}}-P\cdot\frac{\sqrt{n}}{2P\sqrt{n}} \geq 3-\frac{1}{2}-\frac{1}{2}=2,    
\end{align*}
where in the fifth step we make use of \eqref{cm-1}.

\end{proof}

\begin{proof}[Proof of Lemma \ref{contextual mapping}]

Denote \( r' := \frac{\sqrt{2}}{2} n^{2} N^{2} \sqrt{\pi d} r\phi^{-1} \). The Transformer $\boldsymbol{T}^{(cm)}$ is defined as
\begin{align*}
\boldsymbol{T}^{(cm)}(\boldsymbol{X}):&=\begin{pmatrix}
2r'+1&1
\end{pmatrix}\begin{pmatrix}
\boldsymbol{{T}}^{(sid)}\circ\boldsymbol{\mathcal{F}}_{FF}^{(pjt)}(\boldsymbol{X})\\
\begin{pmatrix}
1&0&0&0
\end{pmatrix}\boldsymbol{\mathcal{F}}_{FF}^{(pjt)}(\boldsymbol{X})
\end{pmatrix},
\end{align*}
where $\boldsymbol{\mathcal{F}}_{FF}^{(pjt)}:\mathbb{R}^{d\times n}\to\mathbb{R}^{4\times n}$ is from Lemma \ref{token id} with depth $2$, width $\max\{d,4\}$ and weight bound $\frac{\sqrt{2}}{2} n^{2} N^{2} \sqrt{\pi d}r\phi^{-1}$, $\boldsymbol{{T}}^{(sid)}:\mathbb{R}^{4\times n}\to\mathbb{R}^{1\times n}$ is from Lemma \ref{sequence id} with size $\{(2,4),((2,3),(3,10))\times (n-1)\text{ times},(2,3),(2,4)\}$, dimension vector $\begin{pmatrix}
4\cdot\boldsymbol{1}_{1\times (n+2)}&1
\end{pmatrix}$ and weight bounds $B_{FF}=\max\left\{2r',\frac{3\sqrt{2}}{8}N^2 \sqrt{\pi n}\right\},B_{SA}=\frac{1}{2}\log(3\sqrt{2\pi}n^{2}N^2r')$. According to Lemma \ref{parallel}, $\boldsymbol{T}^{(cm)}$ has size $\{(2,\max\{d,5\}),((3,3),(3,11))\times (n-1)\text{ times},(3,3),(2,5)\}$, dimension vector $\begin{pmatrix}
d&d&5\cdot\boldsymbol{1}_{1\times n}&1
\end{pmatrix}$ and weight bounds $$B_{FF}=\max\left\{\sqrt{2} n^{2} N^{2} \sqrt{\pi d} r\phi^{-1}+1,\frac{3\sqrt{2}}{8}N^2 \sqrt{\pi n}\right\},B_{SA}=\frac{1}{2}\log(3\sqrt{ d} {\pi} n^{4} N^{4} r\phi^{-1}).$$ 
Applying Lemma \ref{token id} and Lemma \ref{sequence id}, we have
\begin{align*}
\boldsymbol{a}^{(i)}:=\boldsymbol{T}^{(cm)}\left(\boldsymbol{X}^{(i)}\right)&=\begin{pmatrix}
2r'+1&1
\end{pmatrix}\begin{pmatrix}
\boldsymbol{{T}}^{(sid)}\circ\boldsymbol{\mathcal{F}}_{FF}^{(pjt)}\left(\boldsymbol{X}^{(i)}\right)\\
\begin{pmatrix}
1&0&0&0
\end{pmatrix}\boldsymbol{\mathcal{F}}_{FF}^{(pjt)}\left(\boldsymbol{X}^{(i)}\right)
\end{pmatrix}\\
&=\begin{pmatrix}
2r'+1&1
\end{pmatrix}\begin{pmatrix}
z^{(i)}\boldsymbol{1}_{1\times n}\\
\boldsymbol{x}^{(i)}
\end{pmatrix}\\
&=(2r'+1)z^{(i)}\boldsymbol{1}_{1\times n}+\boldsymbol{x}^{(i)}.
\end{align*}
It follows that for any $i\in[N],k\in[n]$,
\begin{align*}
\left|a_k^{(i)}\right| &\leq (2r' + 1)\left|z^{(i)}\right| + x_{1,k}^{(i)}\\
&\leq(2r'+1)\left(\frac{3\sqrt{2\pi}}{4}nN^2 r'+\frac{1}{2}\right)+2r'\\
&\leq(2r'+1)\left(\frac{3\sqrt{2\pi}}{4}nN^2 r'+\frac{3}{2}\right).
\end{align*}
It remains to show the separatedness of $\left\{\boldsymbol{a}^{(i)}\right\}_{i\in[N]}$ when \( \boldsymbol{X}_{:,k}^{(i)} \neq \boldsymbol{X}_{:,l}^{(j)} \) or \( \boldsymbol{X}^{(i)} \neq \boldsymbol{X}^{(j)} \) up to permutations. According to Lemma \ref{token id} and Lemma \ref{sequence id}, we have following equivalent conditions:
\begin{itemize}
\item  \( \boldsymbol{X}_{:,k}^{(i)} \neq \boldsymbol{X}_{:,l}^{(j)} \) $\Longleftrightarrow $ \( {x}_{1,k}^{(i)} \neq {x}_{1,l}^{(j)} \).
\item \( \boldsymbol{X}^{(i)} \neq \boldsymbol{X}^{(j)} \) up to permutations $\Longleftrightarrow $ \( z^{(i)} \neq z^{(j)} \).
\end{itemize}
Therefore, in the following we check the separatedness of $\left\{\boldsymbol{a}^{(i)}\right\}_{i\in[N]}$ when \( {x}_{1,k}^{(i)} \neq {x}_{1,l}^{(j)} \) or \( z^{(i)} \neq z^{(j)} \). From definition, we have
\begin{align*}
a_{1,k}^{(i)}-a_{1,l}^{(j)}=(2r' + 1)\left(z^{(i)}-z^{(j)}\right)+\left(x_{1,k}^{(i)}-x_{1,l}^{(j)}\right).
\end{align*}
If $z^{(i)}=z^{(j)}$ and $x_{1,k}^{(i)}\neq x_{1,l}^{(j)}$, by Lemma \ref{token id} there holds
\begin{align*}
\left|a_{1,k}^{(i)}-a_{1,l}^{(j)}\right|=\left|x_{1,k}^{(i)}-x_{1,l}^{(j)}\right|\geq2.
\end{align*}
If $z^{(i)}\neq z^{(j)}$ and $x_{1,k}^{(i)}=  x_{1,l}^{(j)}$, by Lemma \ref{sequence id} there holds
\begin{align*}
\left|a_{1,k}^{(i)}-a_{1,l}^{(j)}\right|=(2r' + 1)\left|z^{(i)}-z^{(j)}\right|\geq2(2r'+1)\geq2.
\end{align*}
If $z^{(i)}\neq z^{(j)}$ and $x_{1,k}^{(i)}\neq  x_{1,l}^{(j)}$, assuming without loss of generality that $z^{(i)}>z^{(j)}$, by Lemma \ref{token id} and Lemma \ref{sequence id} there holds
\begin{align*}
\left|a_{1,k}^{(i)}-a_{1,l}^{(j)}\right|&=\left|(2r' + 1)\left(z^{(i)}-z^{(j)}\right)+\left(x_{1,k}^{(i)}-x_{1,l}^{(j)}\right)\right|\\
&=(2r' + 1)\left(z^{(i)}-z^{(j)}\right)+\left(x_{1,k}^{(i)}-x_{1,l}^{(j)}\right)\\
&\geq2(2r'+1)-4r'=2.
\end{align*}
\end{proof}

\section{Proof of Lemma \ref{parallel}: Parallelization of Transformers}\label{proof of lemma parallel}

(1) Since the argument used here is equally applicable to cases where \( L > 2 \), we only discuss the case where \( L = 2 \). In this case, the two feedforward blocks are in the form of
\begin{align*}
\boldsymbol{\mathcal{F}}_{FF}^{(1)}(\boldsymbol{X})&=\boldsymbol{W}_2^{(1)}\sigma_R\left(\boldsymbol{W}_1^{(1)}\boldsymbol{X}+\boldsymbol{B}_1^{(1)}\right)+\boldsymbol{B}_2^{(1)},\\
\boldsymbol{\mathcal{F}}_{FF}^{(2)}(\boldsymbol{Y})&=\boldsymbol{W}_2^{(2)}\sigma_R\left(\boldsymbol{W}_1^{(2)}\boldsymbol{Y}+\boldsymbol{B}_1^{(2)}\right)+\boldsymbol{B}_2^{(2)},
\end{align*}
where $\boldsymbol{W}_1^{(1)}\in\mathbb{R}^{r\times d^{(1)}},\boldsymbol{W}_1^{(2)}\in\mathbb{R}^{r\times d^{(2)}},\boldsymbol{B}_1^{(1)},\boldsymbol{B}_1^{(2)}\in\mathbb{R}^{r\times n},\boldsymbol{W}_2^{(1)}\in\mathbb{R}^{\bar{d}^{(1)}\times r},\boldsymbol{W}_2^{(2)}\in\mathbb{R}^{\bar{d}^{(2)}\times r},\linebreak\boldsymbol{B}_2^{(1)}\in\mathbb{R}^{\bar{d}^{(1)}\times n},\boldsymbol{B}_2^{(2)}\in\mathbb{R}^{\bar{d}^{(2)}\times n}$. Denote
\begin{align*}
&\boldsymbol{W}_1^{(prl)}:=\begin{pmatrix}
\boldsymbol{W}_1^{(1)}&\boldsymbol{0}_{r\times d^{(2)}}\\
\boldsymbol{0}_{r\times d^{(1)}}&\boldsymbol{W}_1^{(2)}
\end{pmatrix},\quad\boldsymbol{B}_1^{(prl)}:=\begin{pmatrix}
\boldsymbol{B}_1^{(1)}\\ \boldsymbol{B}_1^{(2)}
\end{pmatrix},\\
&\boldsymbol{W}_2^{(prl)}:=\begin{pmatrix}
\boldsymbol{W}_2^{(1)}&\boldsymbol{0}_{\bar{d}^{(1)}\times r}\\
\boldsymbol{0}_{\bar{d}^{(2)}\times r}&\boldsymbol{W}_2^{(2)}
\end{pmatrix},\quad\boldsymbol{B}_2^{(prl)}:=\begin{pmatrix}
\boldsymbol{B}_2^{(1)}\\ \boldsymbol{B}_2^{(2)}
\end{pmatrix},
\end{align*}
and define
\begin{align*}
\boldsymbol{\mathcal{F}}_{FF}^{(prl)}\left(\begin{pmatrix}
\boldsymbol{X} \\
\boldsymbol{Y}
\end{pmatrix}\right)&:=\boldsymbol{W}_2^{(prl)}\sigma_R\left(\boldsymbol{W}_1^{(prl)}\begin{pmatrix}
\boldsymbol{X}\\ \boldsymbol{Y}
\end{pmatrix}+\boldsymbol{B}_1^{(prl)}\right)+\boldsymbol{B}_2^{(prl)}\\
&=\begin{pmatrix}
\boldsymbol{W}_2^{(1)}&\boldsymbol{0}_{\bar{d}^{(1)}\times r}\\
\boldsymbol{0}_{\bar{d}^{(2)}\times r}&\boldsymbol{W}_2^{(2)}
\end{pmatrix}\sigma_R\left(\begin{pmatrix}
\boldsymbol{W}_1^{(1)}&\boldsymbol{0}_{r\times d^{(2)}}\\
\boldsymbol{0}_{r\times d^{(1)}}&\boldsymbol{W}_1^{(2)}
\end{pmatrix}\begin{pmatrix}
\boldsymbol{X}\\ \boldsymbol{Y}
\end{pmatrix}+\begin{pmatrix}
\boldsymbol{B}_1^{(1)}\\ \boldsymbol{B}_1^{(2)}
\end{pmatrix}\right)+\begin{pmatrix}
\boldsymbol{B}_2^{(1)}\\ \boldsymbol{B}_2^{(2)}
\end{pmatrix}\\
&=\begin{pmatrix}
\boldsymbol{W}_2^{(1)}\sigma_R\left(\boldsymbol{W}_1^{(1)}\boldsymbol{X}+\boldsymbol{B}_1^{(1)}\right)+\boldsymbol{B}_2^{(1)}\\
\boldsymbol{W}_2^{(2)}\sigma_R\left(\boldsymbol{W}_1^{(2)}\boldsymbol{Y}+\boldsymbol{B}_1^{(2)}\right)+\boldsymbol{B}_2^{(2)}
\end{pmatrix}=\begin{pmatrix}
\boldsymbol{\mathcal{F}}_{FF}^{(1)}(\boldsymbol{X})\\
\boldsymbol{\mathcal{F}}_{FF}^{(2)}(\boldsymbol{Y})
\end{pmatrix}.
\end{align*}
(2) Consider the following two self-attention layers:
\begin{align*}
\boldsymbol{\mathcal{F}}_{SA}^{(1)}(\boldsymbol{X})&=\boldsymbol{X}+\sum_{h=1}^{H^{(1)}}\boldsymbol{W}_{O,h}^{(1)}\boldsymbol{W}_{V,h}^{(1)}\boldsymbol{X}\sigma_S\left(\boldsymbol{X}^{\top}\boldsymbol{W}_{K,h}^{(1){\top}}\boldsymbol{W}_{Q,h}^{(1)}\boldsymbol{X}\right),\\
\boldsymbol{\mathcal{F}}_{SA}^{(2)}(\boldsymbol{Y})&=\boldsymbol{Y}+\sum_{h=1}^{H^{(2)}}\boldsymbol{W}_{O,h}^{(2)}\boldsymbol{W}_{V,h}^{(2)}\boldsymbol{Y}\sigma_S\left(\boldsymbol{Y}^{\top}\boldsymbol{W}_{K,h}^{(2){\top}}\boldsymbol{W}_{Q,h}^{(2)}\boldsymbol{Y}\right),
\end{align*}
where $\boldsymbol{W}_{O,h}^{(1)}\in\mathbb{R}^{d^{(1)}\times S^{(1)}},\boldsymbol{W}_{V,h}^{(1)},\boldsymbol{W}_{K,h}^{(1)},\boldsymbol{W}_{Q,h}^{(1)}\in\mathbb{R}^{S^{(1)}\times d^{(1)}},\boldsymbol{W}_{O,h}^{(2)}\in\mathbb{R}^{d^{(2)}\times S^{(2)}},\boldsymbol{W}_{V,h}^{(2)},\linebreak\boldsymbol{W}_{K,h}^{(2)},\boldsymbol{W}_{Q,h}^{(2)}\in\mathbb{R}^{S^{(2)}\times d^{(2)}}$. Let
\begin{align*}
&\boldsymbol{W}_{O,h,1}^{(prl)}:=\begin{pmatrix}
\boldsymbol{W}_{O,h}^{(1)}\\
\boldsymbol{0}_{d^{(2)}\times S^{(1)}}
\end{pmatrix},\quad\boldsymbol{W}_{V,h,1}^{(prl)}:=\begin{pmatrix}
\boldsymbol{W}_{V,h}^{(1)}&
\boldsymbol{0}_{S^{(1)}\times d^{(2)}}
\end{pmatrix},\\
&\boldsymbol{W}_{K,h,1}^{(prl)}:=\begin{pmatrix}
\boldsymbol{W}_{K,h}^{(1)}&
\boldsymbol{0}_{S^{(1)}\times d^{(2)}}
\end{pmatrix},\quad\boldsymbol{W}_{Q,h,1}^{(prl)}:=\begin{pmatrix}
\boldsymbol{W}_{Q,h}^{(1)}&
\boldsymbol{0}_{S^{(1)}\times d^{(2)}}
\end{pmatrix}.
\end{align*}
It follows that
\begin{align*}
&\boldsymbol{W}_{O,h,1}^{(prl)}\boldsymbol{W}_{V,h,1}^{(prl)}\begin{pmatrix}
\boldsymbol{X} \\
\boldsymbol{Y}
\end{pmatrix}
\sigma_S\left(\begin{pmatrix}
\boldsymbol{X}^{\top}  & \boldsymbol{Y}^{\top}
\end{pmatrix}\boldsymbol{W}_{K,h,1}^{(prl)T}\boldsymbol{W}_{Q,h,1}^{(prl)}\begin{pmatrix}
\boldsymbol{X} \\
\boldsymbol{Y}
\end{pmatrix}\right)\\
&=\begin{pmatrix}
\boldsymbol{W}_{O,h}^{(1)}\\
\boldsymbol{0}_{d^{(2)}\times S^{(1)}}
\end{pmatrix}\begin{pmatrix}
\boldsymbol{W}_{V,h}^{(1)}&
\boldsymbol{0}_{S^{(1)}\times d^{(2)}}
\end{pmatrix}\begin{pmatrix}
\boldsymbol{X} \\
\boldsymbol{Y}
\end{pmatrix}\sigma_S\left(\begin{pmatrix}
\boldsymbol{X}^{\top}  & \boldsymbol{Y}^{\top}
\end{pmatrix}\begin{pmatrix}
\boldsymbol{W}_{K,h}^{(1){\top}}\\
\boldsymbol{0}_{S^{(1)}\times d^{(2)}}
\end{pmatrix}\begin{pmatrix}
\boldsymbol{W}_{Q,h}^{(1)}&
\boldsymbol{0}_{S^{(1)}\times d^{(2)}}
\end{pmatrix}\begin{pmatrix}
\boldsymbol{X} \\
\boldsymbol{Y}
\end{pmatrix}\right)\\
&=\begin{pmatrix}
\boldsymbol{W}_{O,h}^{(1)}
\boldsymbol{W}_{V,h}^{(1)}
\boldsymbol{X}\sigma_S\left(\boldsymbol{X}^{\top}
\boldsymbol{W}_{K,h}^{(1){\top}}
\boldsymbol{W}_{Q,h}^{(1)}
\boldsymbol{X}\right)\\
\boldsymbol{0}_{d^{(2)}\times n}
\end{pmatrix}.
\end{align*}
Similarly, letting
\begin{align*}
&\boldsymbol{W}_{O,h,2}^{(prl)}:=\begin{pmatrix}
\boldsymbol{0}_{d_1\times S_2}\\
\boldsymbol{W}_{O,h}^{(2)}
\end{pmatrix},\quad\boldsymbol{W}_{V,h,2}^{(prl)}:=\begin{pmatrix}
\boldsymbol{0}_{S^{(2)}\times d^{(1)}}&\boldsymbol{W}_{V,h}^{(2)}
\end{pmatrix},\\
&\boldsymbol{W}_{K,h,2}^{(prl)}:=\begin{pmatrix}
\boldsymbol{0}_{S^{(2)}\times d^{(1)}}&\boldsymbol{W}_{K,h}^{(2)}
\end{pmatrix},\quad\boldsymbol{W}_{Q,h,2}^{(prl)}:=\begin{pmatrix}
\boldsymbol{0}_{S^{(2)}\times d^{(1)}}&\boldsymbol{W}_{Q,h}^{(2)}
\end{pmatrix}
\end{align*}
and we have
\begin{align*}
&\boldsymbol{W}_{O,h,2}^{(prl)}\boldsymbol{W}_{V,h,2}^{(prl)}\begin{pmatrix}
\boldsymbol{X} \\
\boldsymbol{Y}
\end{pmatrix}
\sigma_S\left(\begin{pmatrix}
\boldsymbol{X}^{\top}  & \boldsymbol{Y}^{\top}
\end{pmatrix}\boldsymbol{W}_{K,h,2}^{(prl){\top}}\boldsymbol{W}_{Q,h,2}^{(prl)}\begin{pmatrix}
\boldsymbol{X} \\
\boldsymbol{Y}
\end{pmatrix}\right)\\
&=\begin{pmatrix}
\boldsymbol{0}_{d^{(1)}\times n}\\
\boldsymbol{W}_{O,h}^{(2)}
\boldsymbol{W}_{V,h}^{(2)}
\boldsymbol{Y}\sigma_S\left(\boldsymbol{Y}^{\top}
\boldsymbol{W}_{K,h}^{(2){\top}}
\boldsymbol{W}_{Q,h}^{(2)}
\boldsymbol{Y}\right)
\end{pmatrix}.
\end{align*}
Therefore,
\begin{align*}
&{\boldsymbol{\mathcal{F}}}_{SA}^{(prl)}\left(\begin{pmatrix}
\boldsymbol{X} \\
\boldsymbol{Y}
\end{pmatrix}\right)\\
&:=\begin{pmatrix}
\boldsymbol{X} \\
\boldsymbol{Y}
\end{pmatrix}+\sum_{h=1}^{H^{(1)}}\boldsymbol{W}_{O,h,1}^{(prl)}\boldsymbol{W}_{V,h,1}^{(prl)}\begin{pmatrix}
\boldsymbol{X} \\
\boldsymbol{Y}
\end{pmatrix}
\sigma_S\left(\begin{pmatrix}
\boldsymbol{X}^{\top}  & \boldsymbol{Y}^{\top}
\end{pmatrix}\boldsymbol{W}_{K,h,1}^{(prl){\top}}\boldsymbol{W}_{Q,h,1}^{(prl)}\begin{pmatrix}
\boldsymbol{X} \\
\boldsymbol{Y}
\end{pmatrix}\right)\\
&\qquad+\sum_{h=1}^{H^{(2)}}\boldsymbol{W}_{O,h,2}^{(prl)}\boldsymbol{W}_{V,h,2}^{(prl)}\begin{pmatrix}
\boldsymbol{X} \\
\boldsymbol{Y}
\end{pmatrix}
\sigma_S\left(\begin{pmatrix}
\boldsymbol{X}^{\top}  & \boldsymbol{Y}^{\top}
\end{pmatrix}\boldsymbol{W}_{K,h,2}^{(prl){\top}}\boldsymbol{W}_{Q,h,2}^{(prl)}\begin{pmatrix}
\boldsymbol{X} \\
\boldsymbol{Y}
\end{pmatrix}\right)\\
&=\begin{pmatrix}
\boldsymbol{X} \\
\boldsymbol{Y}
\end{pmatrix}+\sum_{h=1}^{H^{(1)}}\begin{pmatrix}
\boldsymbol{W}_{O,h}^{(1)}
\boldsymbol{W}_{V,h}^{(1)}
\boldsymbol{X}\sigma_S\left(\boldsymbol{X}^{\top}
\boldsymbol{W}_{K,h}^{(1){\top}}
\boldsymbol{W}_{Q,h}^{(1)}
\boldsymbol{X}\right)\\
\boldsymbol{0}_{d^{(2)}\times n}
\end{pmatrix}\\
&\qquad+\sum_{h=1}^{H^{(2)}}\begin{pmatrix}
\boldsymbol{0}_{d^{(1)}\times n}\\
\boldsymbol{W}_{O,h}^{(2)}
\boldsymbol{W}_{V,h}^{(2)}
\boldsymbol{Y}\sigma_S\left(\boldsymbol{Y}^{\top}
\boldsymbol{W}_{K,h}^{(2){\top}}
\boldsymbol{W}_{Q,h}^{(2)}
\boldsymbol{Y}\right)
\end{pmatrix}\\
&=\begin{pmatrix}
\boldsymbol{\mathcal{F}}_{SA}^{(1)}(\boldsymbol{X}) \\
\boldsymbol{\mathcal{F}}_{SA}^{(2)}(\boldsymbol{Y})
\end{pmatrix}.
\end{align*}
(3) A direct corollary from (1) and (2).

\section{Proof of Proposition \ref{general bound}}\label{proof of proposition general bound}

Our proof of Proposition \ref{general bound} follows techniques developed in \cite{van1996weak,nakada2020adaptive,yang2024nonparametric}. Before proving it, we first list the relevant concepts and results from probability theory that will be used in the proof. In the definitions and lemmas below, $  T  $ is a set in the metric space $  (\bar{T}, \tau)  $.

\begin{lemma}[Bernstein's inequality]\label{bernstein}
For i.i.d. random variables $\{Z_i\}_{i=1}^m$ satisfying $|Z_i| \leq c,\ E[Z_i] = 0,\ \operatorname{Var}(Z_i)=\sigma^2$, it holds that
\begin{align*}
\mathbb{P}\left(\left|\frac{1}{m}\sum_{i=1}^{m}{Z_i}\right| \geq u\right) \leq \exp \left( -\frac{mu^2}{2\sigma^2 + 2cu/3} \right)
\end{align*}
for any $u > 0$.
\end{lemma}

\begin{definition}[Gaussian process]
A stochastic process $\{X(t)\}_{t \in T}$ is a {Gaussian process} if for all $n \in \mathbb{N}$, $a_i \in \mathbb{R}$ and $t_i \in T$, the random variable $\sum_{i=1}^n a_i X(t_i)$ is normal or, equivalently, if all the finite-dimensional marginals of $X$ are multivariate normal. $X$ is a \textit{centred Gaussian process} if all these random variables are normal with mean zero.
\end{definition}

\begin{definition}[Sub-gaussian variable and sub-gaussian process]
A square integrable random variable $\xi$ is said to be sub-gaussian with parameter $\sigma > 0$
if for all $\lambda \in \mathbb{R}$,
\[
\mathbb{E} e^{\lambda \xi} \leq e^{\lambda^2 \sigma^2 / 2}.
\]
A centred stochastic process \(\{X(t)\}_{t \in T}\) is sub-gaussian relative to \(\tau\) if its increments satisfy the sub-gaussian inequality:
\[
\mathbb{E} e^{\lambda(X(t) - X(s))} \leq e^{\lambda^2 \tau^2(s,t)/2}, \quad \lambda \in \mathbb{R},\; s,t \in T.
\]
\end{definition}

\begin{lemma}[Borell-Sudakov-Tsirelson concentration inequality]\label{Gaussian concentration inequality}
Let \( \{X(t)\}_{t \in T}\) be a separable centred Gaussian process. Suppose $\mathbb{E}\sup_{t\in T}|X(t)|<\infty,\sigma^2:=\sup_{t\in T}\mathbb{E}X^2(t)<\infty$. Then,
\begin{align*}
\mathbb{P}\left(\sup_{t\in T}|X(t)|\geq \mathbb{E}\sup_{t\in T}|X(t)| + u\right) &\leq e^{-u^2/2\sigma^2}, \\
\mathbb{P}\left(\sup_{t\in T}|X(t)|\leq \mathbb{E}\sup_{t\in T}|X(t)| - u\right) &\leq e^{-u^2/2\sigma^2}. 
\end{align*}
\end{lemma}
\begin{proof}
See \cite[Theorem 2.5.8]{gine2021mathematical}. 
\end{proof}

\begin{lemma}\label{entropy bound}
Let \(\{X(t)\}_{t \in T}\) be a sub-Gaussian process relative to \(\tau\). Assume that
\[
\int_{0}^{\infty} \sqrt{\log \mathcal{N}(\epsilon,T,\tau)} \, d\epsilon < \infty.
\]
Then any separable version of \(\{X(t)\}_{t \in T}\), that we keep denoting by \(X(t)\) satisfies the inequalities  
\[
\mathbb{E} \sup_{t \in T} |X(t)| \leq \mathbb{E} |X(t_0)| + 4\sqrt{2} \int_{0}^{D/2} \sqrt{\log 2\mathcal{N}(\epsilon,T,\tau)} \, d\epsilon,
\]
where \( t_0 \in T \), \( D \) is the diameter of \(T\).
\end{lemma}
\begin{proof}
See \cite[Theorem 2.3.7]{gine2021mathematical}.
\end{proof}

\begin{proof}[Proof of Proposition \ref{general bound}]
By the triangle inequality, we have
\begin{align}\label{ge-2}
\left\|\widehat{f}-f_{0}\right\|_{L^{2}(\mu)}^{2}\leq 2\left\|\widehat{f}-f^{*}\right\|_{L^{2}(\mu)}^{2}+2\|f^{*}-f_{0}\|_{L^{2}(\mu)}^{2}, 
\end{align}
where $f^*$ can be any function in $\mathcal{F}$. For the remainder of the proof, we primarily focus on deriving an upper bound for $\left\|\widehat{f}-f^{*}\right\|_{L^{2}(\mu)}^2$. To this end, we first control $\left\|\widehat{f}-f^{*}\right\|_{L^{2}(\mu)}^2$ by its empirical counterpart $\left\|\widehat{f}-f^{*}\right\|_{m}^2:=\frac{1}{m}\sum_{i=1}^m\left|\widehat{f}(X_i)-f^{*}(X_i)\right|^2$, and then derive an upper bound for $\left\|\widehat{f}-f^{*}\right\|_{m}^2$ by evaluating a variance term.

\textbf{Step 1: Upper Bound of $\left\|\widehat{f} - f^{*}\right\|_{L^{2}(\mu)}^{2}$.}

Denote $N = \mathcal{N}(m^{-\gamma/(2\gamma+d)}, \mathcal{F}, \|\cdot\|_{L^{\infty}(\mu)})$, and let $\{f_{1}, \ldots, f_{N}\}$ be a set of centers of the minimal $m^{-\gamma/(2\gamma+d)}$-cover of $\mathcal{F}$ with $\|\cdot\|_{L^{\infty}(\mu)}$ norm. Suppose $f_{j'} \in \{f_{1}, \ldots, f_{N}\}$ satisfies $\left\|\widehat{f} - f_{j'}\right\|_{L^{\infty}(\mu)} \leq m^{-\gamma/(2\gamma+d)}$. By the triangle inequality, we have
\begin{align}\label{ge-1}
\left\|\widehat{f} - f^{*}\right\|_{L^{2}(\mu)}^{2} &\leq 2\left\|\widehat{f} - f_{j'}\right\|_{L^{2}(\mu)}^{2} + 2\left\|f_{j'} - f^{*}\right\|_{L^{2}(\mu)}^{2} \leq 2m^{-2\gamma/(2\gamma+d)} + 2\left\|f_{j'} - f^{*}\right\|_{L^{2}(\mu)}^{2}.
\end{align}
We bound the term $\left\|f_j - f^*\right\|_{L^2(\mu)}^2$ uniformly for all $j \in [N]$ in order to bound the random quantity $\left\|f_{j'} - f^*\right\|_{L^2(\mu)}^2$. Firstly, given  $j\in[N]$, we apply Bernstein's inequality (Lemma \ref{bernstein}) with 
\begin{align*}
Z_i:=(f_j(X_i) - f^*(X_i))^2 - \mathbb{E}[(f_j(X_i) - f^*(X_i))^2].
\end{align*}
Since
\begin{align*}
|Z_i|&=\left|(f_j(X_i) - f^*(X_i))^2 - \mathbb{E}[(f_j(X_i) - f^*(X_i))^2]\right|\leq8B_{\mathcal{F}}^2,\\
\mathrm{Var}\left(Z_i\right)&=\mathrm{Var}\left((f_j(X_i) - f^*(X_i))^2 - \mathbb{E}[(f_j(X_i) - f^*(X_i))^2]\right)\\
&=\mathrm{Var}\left((f_j(X_i) - f^*(X_i))^2 \right)\\
&=\mathbb{E}\left(\left[f_j(X_i) - f^*(X_i)\right]^4\right)-\left(\mathbb{E}\left[f_j(X_i) - f^*(X_i)\right]^2\right)^2\\
&\leq4B_{\mathcal{F}}^2\mathbb{E}\left(\left[f_j(X_i) - f^*(X_i)\right]^2\right)+4B_{\mathcal{F}}^2\left(\mathbb{E}\left[f_j(X_i) - f^*(X_i)\right]^2\right)\\
&=8B_{\mathcal{F}}^2\left\|f_j - f^*\right\|_{L^2(\mu)}^2\leq16B_{\mathcal{F}}^2 u,
\end{align*}
we substitute $u$ with $\max \left\{ v, \|f_j - f^*\|_{L^2(\mu)}^2/2 \right\}$, $c$ with $8B_{\mathcal{F}}^2$ and $\tau^2$ with $16B_{\mathcal{F}}^2 u$ in Lemma \ref{bernstein} and obtain 
\begin{align}\label{ge0}
\mathbb{P} \left( \|f_j - f^*\|_{L^2(\mu)}^2 \geq \|f_j - f^*\|_m^2 + u \right) \leq \exp \left( -\frac{3mv}{112B_{\mathcal{F}}^2} \right).
\end{align}
By the uniform bound argument, $\|f_j - f^*\|_{L^2(\mu)}^2 \geq \|f_j - f^*\|_m^2 + u$ holds for all $j \in [N]$ with probability at most $N \exp \left( -3mv/(112B_{\mathcal{F}}^2) \right)$. Substituting $v $ with $112B_{\mathcal{F}}^2(m^{d/(2\gamma+d)} + \log N)/(3m)$ leads to the following inequality:
\begin{align*}
\|f_j - f^*\|_m^2 + u &\leq\|f_j - f^*\|_m^2 +v+\frac{1}{2} \|f_j - f^*\|_{L^2(\mu)}^2  \\
&\leq\|f_j - f^*\|_m^2 + \frac{112B_{\mathcal{F}}^2 m^{-2\gamma/(2\gamma+d)}}{3} + \frac{112B_{\mathcal{F}}^2 \log N}{3m} + \frac{1}{2} \|f_j - f^*\|_{L^2(\mu)}^2.
\end{align*}
Combining the above inequality and \eqref{ge0}, we derive that 
\begin{align}\label{ge1}
\|f_j - f^*\|_{L^2(\mu)}^2 \leq 2 \|f_j - f^*\|_m^2 + \frac{224B_{\mathcal{F}}^2 m^{-2\gamma/(2\gamma+d)}}{3} + \frac{224B_{\mathcal{F}}^2 \log N}{3m}
\end{align}
holds for all $j \in [N]$ with probability at least $1 - \exp \left( -m^{d/(2\gamma+d)} \right)$. Plugging \eqref{ge1} into \eqref{ge-1} yields
\begin{align}
&\left\|\widehat{f} - f^*\right\|_{L^2(\mu)}^2\nonumber\\
&\leq 2m^{-2s/(2s+d)} + 4 \|f_{j'} - f^*\|_m^2 + \frac{448 B_{\mathcal{F}}^2 m^{-2\gamma/(2\gamma+d)}}{3} + \frac{448 B_{\mathcal{F}}^2 \log N}{3m}\nonumber\\
&\leq2m^{-2\gamma/(2\gamma+d)} + 8 \left\|\widehat{f}-f_{j'} \right\|_m^2 +8 \left\|\widehat{f} - f^*\right\|_m^2+ \frac{448 B_{\mathcal{F}}^2 m^{-2\gamma/(2\gamma+d)}}{3} + \frac{448 B_{\mathcal{F}}^2 \log N}{3m}\nonumber\\
&\leq 10m^{-2\gamma/(2\gamma+d)}+8 \left\|\widehat{f} - f^*\right\|_m^2+ \frac{448 B_{\mathcal{F}}^2 m^{-2\gamma/(2\gamma+d)}}{3} \nonumber\\
&\quad+ \frac{448 B_{\mathcal{F}}^2 \log \mathcal{N}(m^{-\gamma/(2\gamma+d)}, \mathcal{F}, \|\cdot\|_{L^{\infty}(\mu)})}{3m}\label{ge1.25}
\end{align}
with probability at least $1 - \exp \left( -m^{d/(2s+d)} \right)$. 

\textbf{Step 2: Upper Bound of $\left\|\widehat{f} - f^{*}\right\|_{m}^{2}$.}

Denote $\delta=\max\left\{2^8\sigma m^{-\gamma/(2\gamma+d)},2\left\|\widehat{f}-f_{0}\right\|_{m}\right\}$. Given the observed variables $\{X_i\}_{i=1}^m$, we bound $\left\|\widehat{f} - f^{*}\right\|_m^2$ by considering two cases. In the first case, we suppose that $\left\|\widehat{f} - f^{*}\right\|_m \leq \delta$ holds. By the definition of $\widehat{f}$, we have for any $f \in \mathcal{F}$,
\begin{align*}
\left\|Y - \widehat{f}\ \right\|_m^2 \leq \|Y - f\|_m^2. 
\end{align*}
By substituting $Y_i = f_0(X_i) + \xi_i$, we obtain the base inequality as
\[
\left\|\widehat{f} - f_0\right\|_m^2 \leq \left\|f - f_0\right\|_m^2 + \frac{2}{m} \sum_{i=1}^m \xi_i \left( \hat{f}(X_i) - f(X_i) \right),\quad f \in \mathcal{F}.
\]
Setting $f=f^*$ in the above inequality, we have
\[
\left\|\widehat{f} - f_0\right\|_m^2 \leq \left\|f^* - f_0\right\|_m^2 + \frac{2}{m} \sum_{i=1}^m \xi_i \left( \widehat{f}(X_i) - f^*(X_i) \right),
\]
from which we derive that
\begin{align}
\left\|\widehat{f} - f^{*}\right\|_m^2 
&\leq 2 \left\|\widehat{f} - f_0\right\|_m^2 + 2 \left\|f^{*} - f_0\right\|_m^2 \nonumber\\
&\leq 4 \|f^{*} - f_0\|_m^2 + 4 \sup_{g \in G_{\delta}} \left| \frac{1}{n} \sum_{i=1}^n \xi_i g(X_i) \right|,\label{ge1.5}
\end{align}
where 
\begin{align*}
\mathcal{G}_{\delta} := \left\{ g : g = f - f', \|g\|_{L^\infty(\mu)} \leq \delta, f, f' \in \mathcal{F} \right\}. 
\end{align*}
For $g\in\mathcal{G}_{\delta}$, denote
$Z_g:=\frac{1}{n} \sum_{i=1}^n \xi_i g(X_i)$. It is easy to see that $\{Z_g\}_{g\in\mathcal{G}_{\delta}}$ is a centred Gaussian process and
\begin{align*}
\mathbb{E}|Z_g|^2=\mathrm{Var}(Z_g)=\frac{\sigma^2}{m^2}\sum_{i=1}^mg^2(X_i)\leq\frac{\sigma^2\delta^2}{m}.
\end{align*}
According to Lemma \ref{Gaussian concentration inequality},
\begin{align}
&\mathbb{P} \left( \sup_{g \in G_{\delta}} \left| \frac{1}{m} \sum_{i=1}^m \xi_i g(X_i) \right| \geq \mathbb{E} \left[ \sup_{g \in G_{\delta}} \left| \frac{1}{m} \sum_{i=1}^m \xi_i g(X_i) \right| \right] + 2^{-7}\delta^2 \right)  \nonumber\\
& \leq \exp \left( -\frac{m\delta^2}{2^{15}\sigma^2 } \right)\leq\exp \left( -2m^{d/(2\gamma+d)} \right).\label{ge2}
\end{align}

For $g\in\mathcal{G}_{\delta}$, denote
$\bar{Z}_g:=\frac{1}{\sqrt{m}}\sum_{i=1}^mg(X_i)\frac{\xi_i}{\sigma}$. Since $\{\xi_i\}_{i=1}^m$ are centred Gaussian variables with variances $\sigma^2$, we have
\begin{align*}
\mathbb{E}e^{\lambda(\bar{Z}_g-\bar{Z}_h)}&=\mathbb{E}e^{\lambda\frac{1}{\sqrt{m}}\sum_{i=1}^m[g(X_i)-h(X_i)]\frac{\xi_i}{\sigma}}
=\prod_{i=1}^m\mathbb{E}e^{\lambda\frac{1}{\sqrt{n}}[g(X_i)-h(X_i)]\frac{\xi_i}{\sigma}}\\
&\leq\prod_{i=1}^me^{\lambda^2\frac{1}{2m}[g(X_i)-h(X_i)]^2}\leq e^{\lambda^2\|g-h\|_{L^{\infty}(\mu)}/2},
\end{align*}
which implies that $\{\bar{Z}_g\}_{g\in\mathcal{G}_{\delta}}$ is a sub-gaussian process relative to distance $\|\cdot\|_{L^{\infty}(\mu)}$. Applying Lemma \ref{entropy bound} to $\{\bar{Z}_g\}_{g\in\mathcal{G}_{\delta}}$ and noting that $\mathcal{N}(\varsigma,\mathcal{G}_{\delta},\|\cdot\|_{L^{\infty}(\mu)})\leq\mathcal{N}(\varsigma/2,\mathcal{F},\|\cdot\|_{L^{\infty}(\mu)})^2$, we have
\begin{align*}
\mathbb{E} \sup_{g \in \mathcal{G}_{\delta}} |\bar{Z}_g| &\leq 4\sqrt{2} \int_{0}^{\delta} \sqrt{\log 2\mathcal{N}(\varsigma,\mathcal{G}_{\delta},\|\cdot\|_{L^{\infty}(\mu)})} \, d\varsigma\\
&\leq4\sqrt{2} \int_{0}^{\delta} \sqrt{\log 2\mathcal{N}(\varsigma/2,\mathcal{F},\|\cdot\|_{L^{\infty}(\mu)})^2} \, d\varsigma.
\end{align*}
It follows directly that
\begin{align}
\mathbb{E} \left[ \sup_{g \in \mathcal{G}_{\delta}} \left| \frac{1}{m} \sum_{i=1}^m \xi_i g(X_i) \right| \right] \leq \frac{4\sqrt{2}\sigma}{\sqrt{m}}  \int_{0}^{\delta} \sqrt{\log 2\mathcal{N}(\varsigma/2,\mathcal{F},\|\cdot\|_{L^{\infty}(\mu)})^2} \, d\varsigma.\label{ge3}
\end{align}
Combining \eqref{ge2} and \eqref{ge3} yields that
\begin{align*}
\sup_{g \in G_{\delta}} \left| \frac{1}{m} \sum_{i=1}^m \xi_i g(X_i) \right| &\leq \frac{4\sqrt{2}\sigma}{\sqrt{m}}\int_{0}^{\delta} \sqrt{\log 2\mathcal{N}(\varsigma/2,\mathcal{F},\|\cdot\|_{L^{\infty}(\mu)})^2} \, d\varsigma + 2^{-7}\delta^2\\
&=\frac{4\sqrt{2}\sigma\delta}{\sqrt{m}}\int_{0}^{1} \sqrt{\log 2\mathcal{N}(\delta\varsigma/2,\mathcal{F},\|\cdot\|_{L^{\infty}(\mu)})^2} \, d\varsigma + 2^{-7}\delta^2\\
&\leq\frac{2^{10}\sigma^2}{{m}}\left(\int_{0}^{1} \sqrt{\log 2\mathcal{N}(\delta\varsigma/2,\mathcal{F},\|\cdot\|_{L^{\infty}(\mu)})^2} \, d\varsigma\right)^2+2^{-6}\delta^2\\
&\leq\frac{2^{10}\sigma^2}{{m}}\left(\int_{0}^{1} \sqrt{\log 2\mathcal{N}(2^7\sigma m^{-\gamma/(2\gamma+d)}\varsigma,\mathcal{F},\|\cdot\|_{L^{\infty}(\mu)})^2} \, d\varsigma\right)^2+2^{-6}\delta^2\\
&=\frac{8\sigma}{ m^{(\gamma+d)/(2\gamma+d)}}\left(\int_{0}^{2^7\sigma m^{-\gamma/(2\gamma+d)}} \sqrt{\log 2\mathcal{N}(\varsigma,\mathcal{F},\|\cdot\|_{L^{\infty}(\mu)})^2} \, d\varsigma\right)^2+2^{-6}\delta^2
\end{align*}
with probability at least $1 - \exp \left( -2m^{d/(2\gamma+d)} \right)$. Here we use the inequality $ab\leq\frac{a^2}{2}+\frac{b^2}{2}$. Plugging the above inequality into \eqref{ge1.5} yields
\begin{align*}
&\left\|\widehat{f} - f^{*}\right\|_m^2\\ 
&\leq 4 \|f^{*} - f_0\|_m^2 + \frac{32\sigma}{ m^{(\gamma+d)/(2\gamma+d)}}\left(\int_{0}^{2^7\sigma m^{-\gamma/(2\gamma+d)}} \sqrt{\log 2\mathcal{N}(\varsigma,\mathcal{F},\|\cdot\|_{L^{\infty}(\mu)})^2} \, d\varsigma\right)^2 + \frac{1}{16}\delta^2\\
&\leq4 \|f^{*} - f_0\|_m^2 + \frac{32\sigma}{ m^{(\gamma+d)/(2\gamma+d)}}\left(\int_{0}^{2^7\sigma m^{-\gamma/(2\gamma+d)}} \sqrt{\log 2\mathcal{N}(\varsigma,\mathcal{F},\|\cdot\|_{L^{\infty}(\mu)})^2} \, d\varsigma\right)^2 \\
&\quad+ 2^{12}\sigma^2 m^{-2\gamma/(2\gamma+d)}+\frac{1}{4}\left\|\widehat{f}-f_{0}\right\|_{m}^2\\
&\leq4 \|f^{*} - f_0\|_m^2 + \frac{32\sigma}{ m^{(\gamma+d)/(2\gamma+d)}}\left(\int_{0}^{2^7\sigma m^{-\gamma/(2\gamma+d)}} \sqrt{\log 2\mathcal{N}(\varsigma,\mathcal{F},\|\cdot\|_{L^{\infty}(\mu)})^2} \, d\varsigma\right)^2 \\
&\quad+ 2^{12}\sigma^2 m^{-2\gamma/(2\gamma+d)}+\frac{1}{2}\left\|\widehat{f}-f^*\right\|_{m}^2+\frac{1}{2}\left\|f^*-f_0\right\|_{m}^2
\end{align*}
with probability at least $1 - \exp \left( -2m^{d/(2\gamma+d)} \right)$, from which we can immediately obtain 
\begin{align}
\left\|\widehat{f} - f^{*}\right\|_m^2\leq& 9\|f^{*} - f_0\|_m^2+ 2^{13}\sigma^2 m^{-2\gamma/(2\gamma+d)}\nonumber\\
&+ \frac{64\sigma}{ m^{(\gamma+d)/(2\gamma+d)}}\left(\int_{0}^{2^7\sigma m^{-\gamma/(2\gamma+d)}} \sqrt{\log 2\mathcal{N}(\varsigma,\mathcal{F},\|\cdot\|_{L^{\infty}(\mu)})^2} \, d\varsigma\right)^2\label{ge4}
\end{align}
with probability at least $1 - \exp \left( -2m^{d/(2\gamma+d)} \right)$.

In the second case, we suppose that $\left\|\widehat{f}-f^{*}\right\|_{m}\geq\delta\geq2\left\|\widehat{f}-f_{0}\right\|_{m}$ holds. It follows that
\[
\left\|\widehat{f}-f^{*}\right\|_{m}^{2}\leq 2\left\|\widehat{f}-f_{0}\right\|_{m}^{2}+2\left\|f^{*}-f_{0}\right\|_{m}^{2}\leq\frac{1}{2}\left\|\widehat{f}-f^{*}\right\|_{m}^{2}+2\left\|f^{*}-f_{0}\right\|_{m}^{2}.
\]
which implies $\left\|\widehat{f}-f^{*}\right\|_{m}^{2}\leq 4\left\|f^{*}-f_{0}\right\|_{m}^{2}$. Hence in this case, the inequality \eqref{ge4} still holds.

\textbf{Step 3: Combine the Results.}

From the conclusion of \eqref{ge1.25} in Step 1 and \eqref{ge4} in Step 2, we obtain
\begin{align*}
&\left\|\widehat{f} - f^*\right\|_{L^2(\mu)}^2 \\
&\leq \left(\frac{448 B_{\mathcal{F}}^2}{3}+2^{16}\sigma^2+10\right)m^{-2\gamma/(2\gamma+d)} + \frac{448 B_{\mathcal{F}}^2 \log \mathcal{N}(m^{-\gamma/(2\gamma+d)}, \mathcal{F}, \|\cdot\|_{L^{\infty}(\mu)})}{3m}\\
&\quad+72\|f^{*} - f_0\|_m^2 + \frac{2^9\sigma}{ m^{(\gamma+d)/(2\gamma+d)}}\left(\int_{0}^{2^7\sigma m^{-\gamma/(2\gamma+d)}} \sqrt{\log 2\mathcal{N}(\varsigma,\mathcal{F},\|\cdot\|_{L^{\infty}(\mu)})^2} \, d\varsigma\right)^2 
\end{align*}
with probability at least $1 - 2\exp \left( -m^{d/(2\gamma+d)} \right)$. We complete the proof by combining this inequality and \eqref{ge-2}.
\end{proof}

\section{Proof of Lemma \ref{covering number bound}: Estimation of Lipschitz Constant}\label{proof of covering number bound}

In Lemmas \ref{norm of softmax}, \ref{softmax Jacobian} and \ref{Lip of softmax}, we derive several key properties of the softmax function.

\begin{lemma}\label{norm of softmax}
For any $\boldsymbol{x}\in\mathbb{R}^d$, there holds
\begin{align*}
\|\sigma_S(\boldsymbol{x})\|_2\leq1.
\end{align*}
\end{lemma}
\begin{proof}
Since $[\sigma_S(\boldsymbol{x})]_i>0$ for $i\in[d]$, we have
\begin{align*}
\|\sigma_S(\boldsymbol{x})\|_2^2=\sum_{i=1}^{d}[\sigma_S(\boldsymbol{x})]_i^2\leq\left(\sum_{i=1}^{d}[\sigma_S(\boldsymbol{x})]_i\right)^2=1.
\end{align*}

\end{proof}

\begin{lemma}\label{softmax Jacobian}
Let $\boldsymbol{x}\in\mathbb{R}^{d}$ and $\boldsymbol{y}=\sigma_S(\boldsymbol{x})$. Let $\sigma_S'$ be the Jacobian matrix of $\sigma_S$. There holds
\begin{align*}
\sigma_S'(\boldsymbol{x})=\mathrm{diag(\boldsymbol{y})}-\boldsymbol{y}\boldsymbol{y}^{\top}.
\end{align*}
\end{lemma}
\begin{proof}
For $i,j\in[d]$, direct calculation yields
\begin{align*}
\frac{\partial y_i}{\partial x_j}=\frac{\partial}{\partial x_j}\frac{e^{x_i}}{\sum_{k=1}^{d }e^{x_k}}=\frac{e^{x_i}\delta_{ij}}{\sum_{k=1}^{d }e^{x_k}}-\frac{e^{x_i}}{\sum_{k=1}^{d }e^{x_k}}\frac{e^{x_j}}{\sum_{k=1}^{d }e^{x_k}}.
\end{align*}
\end{proof}

\begin{lemma}[Mean-value theorem for vector-valued functions]\label{mvt}
Let $u,v\in\mathbb{N}_{\geq1}$. Let \( S \) be an open subset of \( \mathbb{R}^u \) and assume that \( \boldsymbol{f} : S \to \mathbb{R}^v \) is differentiable at each point of \( S \). Let \( \boldsymbol{x} \) and \( \boldsymbol{y} \) be two points in \( S \) such that \( L(\boldsymbol{x}, \boldsymbol{y}) \subseteq S \), where \(L(\boldsymbol{x}, \boldsymbol{y}):=\{t\boldsymbol{x}+(1-t)\boldsymbol{y}:t\in[0,1]\}\). Then for every vector \( \boldsymbol{a} \) in \( \mathbb{R}^v \), there is a point \( \boldsymbol{z} \) in \( L(\boldsymbol{x}, \boldsymbol{y}) \) such that
\[
\boldsymbol{a}^{\top}  [\boldsymbol{f}(\boldsymbol{y}) - \boldsymbol{f}(\boldsymbol{x})] = \boldsymbol{a}^{\top} \boldsymbol{f}'(\boldsymbol{z})(\boldsymbol{y} - \boldsymbol{x}),
\]
where $\boldsymbol{f}'$ is the Jacobian matrix of $\boldsymbol{f}$.
\end{lemma}
\begin{proof}
See, for example, \cite[Theorem 12.9]{Apostol1974Mathematical}. 
\end{proof}

\begin{lemma}\label{Lip of softmax}
For any $\boldsymbol{x},\widetilde{\boldsymbol{x}}\in\mathbb{R}^d$, there holds
\begin{align*}
\|\sigma_S(\boldsymbol{\widetilde{x}})-\sigma_S(\boldsymbol{{x}})\|_2\leq2\|\boldsymbol{\widetilde{x}}-\boldsymbol{{x}}\|_2.
\end{align*}
\end{lemma}
\begin{proof}
Choosing $\boldsymbol{a}=\sigma_S(\boldsymbol{\widetilde{x}}) - \sigma_S(\boldsymbol{x})$ in Lemma \ref{mvt}, we obtain 
\begin{align*}
[\sigma_S(\boldsymbol{\widetilde{x}}) - \sigma_S(\boldsymbol{x})]^{\top}  [\sigma_S(\boldsymbol{\widetilde{x}}) - \sigma_S(\boldsymbol{x})] = [\sigma_S(\boldsymbol{\widetilde{x}}) - \sigma_S(\boldsymbol{x})]^{\top} \sigma_S'(\boldsymbol{z})(\boldsymbol{\widetilde{x}} - \boldsymbol{x})
\end{align*}
for some $\boldsymbol{z}\in\mathbb{R}^d$. It follows that
\begin{align*}
\|\sigma_S(\boldsymbol{\widetilde{x}}) - \sigma_S(\boldsymbol{x})\|_2^2 \leq \|\sigma_S(\boldsymbol{\widetilde{x}}) - \sigma_S(\boldsymbol{x})\|_2\|\sigma_S'(\boldsymbol{z})\|_2\|\boldsymbol{\widetilde{x}} - \boldsymbol{x}\|_2,
\end{align*}
which implies
\begin{align*}
\|\sigma_S(\boldsymbol{\widetilde{x}}) - \sigma_S(\boldsymbol{x})\|_2\leq \|\sigma_S'(\boldsymbol{z})\|_2\|\boldsymbol{\widetilde{x}} - \boldsymbol{x}\|_2.
\end{align*}
Denote $\boldsymbol{y}=\sigma_S(\boldsymbol{z})$. By Lemma \ref{norm of softmax} and Lemma \ref{softmax Jacobian}, we have
\begin{align*}
\|\sigma_S'(\boldsymbol{z})\|_2=\|\mathrm{diag(\boldsymbol{y})}-\boldsymbol{y}\boldsymbol{y}^{\top}\|_2\leq\|\mathrm{diag(\boldsymbol{y})}\|_2+\|\boldsymbol{y}\|_2^2\leq1+1=2.
\end{align*}
\end{proof}
In the following three lemmas (Lemmas \ref{bound of FFSA} - \ref{Lip of FFSA}), we study properties of single feedforward block $\boldsymbol{\mathcal{F}}_{FF}:\mathbb{R}^{d_{FF}^{(in)}\times n}\to\mathbb{R}^{d_{FF}^{(out)}\times n}$ with depth $L$, width $W$ and weight bound $B_{FF}$, taking the form of
\begin{align*}
\boldsymbol{\mathcal{F}}_0&=\boldsymbol{X};\\
\boldsymbol{\mathcal{F}}_{l}&=\sigma_R(\boldsymbol{W}_{l}\boldsymbol{\mathcal{F}}_{l-1}+\boldsymbol{B}_l),\quad l\in\{1,2,\dots,L-1\};\\
\boldsymbol{\mathcal{F}}_{FF}&=\boldsymbol{W}_{L}\boldsymbol{\mathcal{F}}_{L-1}+\boldsymbol{B}_{L},
\end{align*}
properties of single softmax layer $\boldsymbol{\mathcal{F}}_{SA}:\mathbb{R}^{d_{SA}\times n}\to\mathbb{R}^{d_{SA}\times n}$ with head number $H$, head size $S$ and weight bound $B_{SA}$, taking the form of
\begin{align*}
&\boldsymbol{\mathcal{F}}_{SA}(\boldsymbol{Y})=
\boldsymbol{Y}+\sum_{h=1}^H \boldsymbol{W}_{O}^{(h)} \boldsymbol{W}_{V}^{(h)} \boldsymbol{Y} \sigma_S\left( \boldsymbol{Y}^{\top}\boldsymbol{W}_{K}^{(h)\top}\boldsymbol{W}_{Q}^{(h)} \boldsymbol{Y}\right),
\end{align*}
and properties of embedding layer $\boldsymbol{\mathcal{F}}_{EB}:\mathbb{R}^{d_{in}\times n}\to\mathbb{R}^{d_{EB}\times n}$ with weight bound $B_{EB}$, taking the form of
\begin{align*}
\boldsymbol{\mathcal{F}}_{EB}(\boldsymbol{Z})=\boldsymbol{W}_{EB}\boldsymbol{Z}+\boldsymbol{B}_{EB}.
\end{align*}

Without loss of generality we assume $B_{FF},B_{SA},B_{EB}\geq1$.

\begin{lemma}\label{bound of FFSA}
For any $\boldsymbol{X}\in\mathbb{R}^{d_{FF}^{(in)}\times n},\boldsymbol{Y}\in\mathbb{R}^{d_{SA}\times n},\boldsymbol{Z}\in\mathbb{R}^{d_{in}\times n}$, there holds
\begin{align*}
\left\|\boldsymbol{\mathcal{F}}_{FF}(\boldsymbol{X})\right\|_F
&\leq2\sqrt{d_{FF}^{(in)}d_{FF}^{(out)}n}LW^{L-1}B_{FF}^L\|\boldsymbol{X}\|_F,\\
\left\|\boldsymbol{\mathcal{F}}_{SA}(\boldsymbol{Y})\right\|_F&\leq2d_{SA}\sqrt{n}HSB_{SA}^2\left\|\boldsymbol{Y}\right\|_F,\\
\left\|\boldsymbol{\mathcal{F}}_{EB}(\boldsymbol{Z})\right\|_F&\leq2\sqrt{d_{EB}d_{in}n}B_{EB}\left\|\boldsymbol{Z}\right\|_F.
\end{align*}
\end{lemma}
\begin{proof}
For $\boldsymbol{\mathcal{F}}_{FF}$, by definition we have
\begin{align*}
\left\|\boldsymbol{\mathcal{F}}_{FF}(\boldsymbol{X})\right\|_F&=\|\boldsymbol{W}_{L}\boldsymbol{\mathcal{F}}_{L-1}+\boldsymbol{B}_{L}\|_F\\
&\leq\|\boldsymbol{W}_{L}\|_F\|\boldsymbol{\mathcal{F}}_{L-1}\|_F+\|\boldsymbol{B}_{L}\|_F\\
&\leq\|\boldsymbol{W}_{L}\|_F\|\boldsymbol{W}_{L-1}\boldsymbol{\mathcal{F}}_{L-2}+\boldsymbol{B}_{L-1}\|_F+\|\boldsymbol{B}_{L}\|_F\\
&\leq\|\boldsymbol{W}_{L}\|_F\|\boldsymbol{W}_{L-1}\|_F\|\boldsymbol{\mathcal{F}}_{L-2}\|_F+\|\boldsymbol{W}_{L}\|_F\|\boldsymbol{B}_{L-1}\|_F+\|\boldsymbol{B}_{L}\|_F,
\end{align*}
where in the third step, we use the property $\sigma_R(x)\leq x$ for any $x\in\mathbb{R}$. Repeating this process, we obtain
\begin{align*}
\left\|\boldsymbol{\mathcal{F}}_{FF}(\boldsymbol{X})\right\|_F
&\leq\sum_{l=1}^{L}\left(\prod_{l'=l+1}^{L}\|\boldsymbol{W}_{l'}\|_F\right)\|\boldsymbol{B}_l\|_F+\left(\prod_{l=1}^{L}\|\boldsymbol{W}_{l}\|_F\right)\|\boldsymbol{X}\|_F\\
&\leq\sqrt{d_{FF}^{(out)}n}LW^{L-1}B_{FF}^L+\sqrt{d_{FF}^{(in)}d_{FF}^{(out)}}W^{L-1}B_{FF}^L\|\boldsymbol{X}\|_F\\
&\leq2\sqrt{d_{FF}^{(in)}d_{FF}^{(out)}n}LW^{L-1}B_{FF}^L\|\boldsymbol{X}\|_F.
\end{align*}
For $\boldsymbol{\mathcal{F}}_{SA}$, by definition we have
\begin{align*}
\left\|\boldsymbol{\mathcal{F}}_{SA}(\boldsymbol{Y})\right\|_F&\leq\left\|\boldsymbol{Y}\right\|_F+\sum_{h=1}^H \left\|\boldsymbol{W}_{O}^{(h)}\right\|_F\left\|\boldsymbol{W}_{V}^{(h)}\right\|_F  \left\|\boldsymbol{Y}\right\|_F\left\|\sigma_S\left( \boldsymbol{Y}^{\top}\boldsymbol{W}_{K}^{(h)\top}\boldsymbol{W}_{Q}^{(h)}  \boldsymbol{Y}\right)\right\|_F\\
&\leq (d_{SA}\sqrt{n}HSB_{SA}^2+1)\left\|\boldsymbol{Y}\right\|_F\leq2d_{SA}\sqrt{n}HSB_{SA}^2\left\|\boldsymbol{Y}\right\|_F,
\end{align*}
where in the second step we use Lemma \ref{norm of softmax}. The bound of $\boldsymbol{\mathcal{F}}_{EB}$ can be obtained directly from the definition.

\end{proof}

\begin{lemma}\label{difference of FFSA}
Let $\varsigma\in\mathbb{R}_{>0}$. Let $\boldsymbol{\widetilde{\mathcal{F}}}_{FF},\boldsymbol{{\mathcal{F}}}_{FF}$ be two feedforward blocks with each trainable parameter differing by at most $\varsigma$. Let $\boldsymbol{\widetilde{\mathcal{F}}}_{SA},\boldsymbol{{\mathcal{F}}}_{SA}$ be two self-attention layers with each trainable parameter also differing by at most $\varsigma$. Let $\boldsymbol{\widetilde{\mathcal{F}}}_{EB},\boldsymbol{{\mathcal{F}}}_{EB}$ be two embedding layers with each trainable parameter also differing by at most $\varsigma$. For any $\boldsymbol{X}\in\mathbb{R}^{d_{FF}^{(in)}\times n},\boldsymbol{Y}\in\mathbb{R}^{d_{SA}\times n},\boldsymbol{Z}\in\mathbb{R}^{d_{in}\times n}$, there holds
\begin{align*}
\left\|\boldsymbol{\widetilde{\mathcal{F}}}_{FF}(\boldsymbol{X})-\boldsymbol{\mathcal{F}}_{FF}(\boldsymbol{X})\right\|_F
&\leq 4{d_{FF}^{(in)}\left(d_{FF}^{(out)}\right)^{3/2}}nL^2W^{2L-3/2}B_{FF}^{2L-1}\|\boldsymbol{X}\|_F\varsigma,\\
\left\|\boldsymbol{\mathcal{\widetilde{F}}}_{SA}(\boldsymbol{Y})-\boldsymbol{\mathcal{F}}_{SA}(\boldsymbol{Y})\right\|_F&\leq3d_{SA}^2\sqrt{n}HS^2B_{SA}^3\left\|\boldsymbol{Y}\right\|_F^3\varsigma,\\
\left\|\boldsymbol{\mathcal{\widetilde{F}}}_{EB}(\boldsymbol{Z})-\boldsymbol{\mathcal{F}}_{EB}(\boldsymbol{Z})\right\|_F&\leq2\sqrt{d_{EB}d_{in}n}\left\|\boldsymbol{Z}\right\|_F\varsigma.
\end{align*}
\end{lemma}
\begin{proof}
For $\boldsymbol{\mathcal{F}}_{FF}$, by definition we have
\begin{align*}
&\left\|\boldsymbol{\widetilde{\mathcal{F}}}_{FF}(\boldsymbol{X})-\boldsymbol{\mathcal{F}}_{FF}(\boldsymbol{X})\right\|_F\\
&=\left\|\boldsymbol{\widetilde{W}}_{L}\boldsymbol{\mathcal{\widetilde{F}}}_{L-1}+\boldsymbol{\widetilde{B}}_{L}-\boldsymbol{W}_{L}\boldsymbol{\mathcal{F}}_{L-1}-\boldsymbol{B}_{L}\right\|_F\\
&\leq\left\|\boldsymbol{\widetilde{W}}_{L}\right\|_F\left\|\boldsymbol{\mathcal{\widetilde{F}}}_{L-1}-\boldsymbol{\mathcal{F}}_{L-1}\right\|_F+\left\|\boldsymbol{\widetilde{W}}_{L}-\boldsymbol{W}_{L}\right\|_F\left\|\boldsymbol{\mathcal{F}}_{L-1}\right\|_F+\left\|\boldsymbol{\widetilde{B}}_{L}-\boldsymbol{B}_{L}\right\|_F\\
&\leq\left\|\boldsymbol{\widetilde{W}}_{L}\right\|_F\left\|\boldsymbol{\widetilde{W}}_{L-1}\boldsymbol{\mathcal{\widetilde{F}}}_{L-2}+\boldsymbol{\widetilde{B}}_{L-1}-\boldsymbol{W}_{L-1}\boldsymbol{\mathcal{F}}_{L-2}-\boldsymbol{B}_{L-1}\right\|_F\\
&\quad+\left\|\boldsymbol{\widetilde{W}}_{L}-\boldsymbol{W}_{L}\right\|_F\left\|\boldsymbol{\mathcal{F}}_{L-1}\right\|_F+\left\|\boldsymbol{\widetilde{B}}_{L}-\boldsymbol{B}_{L}\right\|_F\\
&\leq\left\|\boldsymbol{\widetilde{W}}_{L}\right\|_F\left\|\boldsymbol{\widetilde{W}}_{L-1}\right\|_F\left\|\boldsymbol{\mathcal{\widetilde{F}}}_{L-2}-\boldsymbol{\mathcal{F}}_{L-2}\right\|_F+\left\|\boldsymbol{\widetilde{W}}_{L}\right\|_F\left\|\boldsymbol{\widetilde{W}}_{L-1}-\boldsymbol{W}_{L-1}\right\|_F\left\|\boldsymbol{\mathcal{F}}_{L-2}\right\|_F\\
&\quad+\left\|\boldsymbol{\widetilde{W}}_{L}\right\|_F\left\|\boldsymbol{\widetilde{B}}_{L-1}-\boldsymbol{B}_{L-1}\right\|_F+\left\|\boldsymbol{\widetilde{W}}_{L}-\boldsymbol{W}_{L}\right\|_F\left\|\boldsymbol{\mathcal{F}}_{L-1}\right\|_F+\left\|\boldsymbol{\widetilde{B}}_{L}-\boldsymbol{B}_{L}\right\|_F,
\end{align*}
where in the third step we use the fact that $\sigma_R$ is $1$-Lipschitz. Repeating this process, we obtain
\begin{align}
&\left\|\boldsymbol{\widetilde{\mathcal{F}}}_{FF}(\boldsymbol{X})-\boldsymbol{\mathcal{F}}_{FF}(\boldsymbol{X})\right\|_F\nonumber\\
&\leq\sum_{l=1}^{L}\left(\prod_{l'=l+1}^{L}\left\|\boldsymbol{\widetilde{W}}_{l'}\right\|_F\right)\left(\left\|\boldsymbol{\widetilde{W}}_{l}-\boldsymbol{W}_{l}\right\|_F\left\|\boldsymbol{\mathcal{F}}_{l-1}\right\|_F+\left\|\boldsymbol{\widetilde{B}}_{l}-\boldsymbol{B}_{l}\right\|_F\right).\label{difference1}
\end{align}
From the derivation of Lemma \ref{norm of softmax}, we can find that for $l\in[L]$,
\begin{align}\label{difference2}
\left\|\boldsymbol{\mathcal{F}}_{l}\right\|_F\leq2\sqrt{d_{FF}^{(in)}d_{FF}^{(out)}n}LW^{L-1}B_{FF}^L\|\boldsymbol{X}\|_F.
\end{align}
Plugging \eqref{difference2} into \eqref{difference1}, we obtain
\begin{align*}
\left\|\boldsymbol{\widetilde{\mathcal{F}}}_{FF}(\boldsymbol{X})-\boldsymbol{\mathcal{F}}_{FF}(\boldsymbol{X})\right\|_F
\leq 4{d_{FF}^{(in)}\left(d_{FF}^{(out)}\right)^{3/2}}nL^2W^{2L-3/2}B_{FF}^{2L-1}\|\boldsymbol{X}\|_F\varsigma.
\end{align*}
For $\boldsymbol{\mathcal{F}}_{SA}$, by definition we have
\begin{align*}
&\left\|\boldsymbol{\mathcal{\widetilde{F}}}_{SA}(\boldsymbol{Y})-\boldsymbol{\mathcal{F}}_{SA}(\boldsymbol{Y})\right\|_F\\
&\leq
\sum_{h=1}^H \left\|\boldsymbol{\widetilde{W}}_{O}^{(h)}\boldsymbol{\widetilde{W}}_{V}^{(h)} \boldsymbol{Y} \sigma_S\left( \boldsymbol{Y}^{\top}\boldsymbol{\widetilde{W}}_{K}^{(h)\top}\boldsymbol{\widetilde{W}}_{Q}^{(h)}  \boldsymbol{Y}\right)-\boldsymbol{{W}}_{O}^{(h)}\boldsymbol{{W}}_{V}^{(h)}  \boldsymbol{Y} \sigma_S\left( \boldsymbol{Y}^{\top}\boldsymbol{{W}}_{K}^{(h)\top}\boldsymbol{{W}}_{Q}^{(h)} \boldsymbol{Y}\right)\right\|_F\\
&\leq\sum_{h=1}^H \left\|\boldsymbol{\widetilde{W}}_{O}^{(h)}\boldsymbol{\widetilde{W}}_{V}^{(h)} \boldsymbol{Y} \sigma_S\left( \boldsymbol{Y}^{\top}\boldsymbol{\widetilde{W}}_{K}^{(h)\top}\boldsymbol{\widetilde{W}}_{Q}^{(h)}  \boldsymbol{Y}\right)-\boldsymbol{\widetilde{W}}_{O}^{(h)}\boldsymbol{\widetilde{W}}_{V}^{(h)} \boldsymbol{Y} \sigma_S\left( \boldsymbol{Y}^{\top}\boldsymbol{{W}}_{K}^{(h)\top}\boldsymbol{{W}}_{Q}^{(h)} \boldsymbol{Y}\right)\right\|_F\\
&\quad+\sum_{h=1}^H \left\|\boldsymbol{\widetilde{W}}_{O}^{(h)}\boldsymbol{\widetilde{W}}_{V}^{(h)} \boldsymbol{Y} \sigma_S\left( \boldsymbol{Y}^{\top}\boldsymbol{{W}}_{K}^{(h)\top}\boldsymbol{{W}}_{Q}^{(h)} \boldsymbol{Y}\right)-\boldsymbol{{W}}_{O}^{(h)}\boldsymbol{{W}}_{V}^{(h)} \boldsymbol{Y} \sigma_S\left( \boldsymbol{Y}^{\top}\boldsymbol{{W}}_{K}^{(h)\top}\boldsymbol{{W}}_{Q}^{(h)} \boldsymbol{Y}\right)\right\|_F\\
&\leq\sum_{h=1}^H \left\|\boldsymbol{\widetilde{W}}_{O}^{(h)}\right\|_F\left\|\boldsymbol{\widetilde{W}}_{V}^{(h)}\right\|_F\left\|\boldsymbol{Y}\right\|_F\left\|\sigma_S\left( \boldsymbol{Y}^{\top}\boldsymbol{\widetilde{W}}_{K}^{(h)\top}\boldsymbol{\widetilde{W}}_{Q}^{(h)}  \boldsymbol{Y}\right)-\sigma_S\left( \boldsymbol{Y}^{\top}\boldsymbol{{W}}_{K}^{(h)\top}\boldsymbol{{W}}_{Q}^{(h)} \boldsymbol{Y}\right)\right\|_F\\
&\quad+\sum_{h=1}^H \left\|\boldsymbol{\widetilde{W}}_{O}^{(h)}\boldsymbol{\widetilde{W}}_{V}^{(h)} -\boldsymbol{{W}}_{O}^{(h)}\boldsymbol{{W}}_{V}^{(h)}\right\|_F\left\| \boldsymbol{Y}\right\|_F \left\|\sigma_S\left( \boldsymbol{Y}^{\top}\boldsymbol{{W}}_{K}^{(h)\top}\boldsymbol{{W}}_{Q}^{(h)}  \boldsymbol{Y}\right)\right\|_F\\
&\leq2\sum_{h=1}^H \left\|\boldsymbol{\widetilde{W}}_{O}^{(h)}\right\|_F\left\|\boldsymbol{\widetilde{W}}_{V}^{(h)}\right\|_F\left\|\boldsymbol{Y}\right\|_F^3\left\|\boldsymbol{\widetilde{W}}_{K}^{(h)\top}\boldsymbol{\widetilde{W}}_{Q}^{(h)} -\boldsymbol{{W}}_{K}^{(h)\top}\boldsymbol{{W}}_{Q}^{(h)} \right\|_F\\
&\quad+\sqrt{n}\sum_{h=1}^H \left\|\boldsymbol{\widetilde{W}}_{O}^{(h)}\boldsymbol{\widetilde{W}}_{V}^{(h)} -\boldsymbol{{W}}_{O}^{(h)}\boldsymbol{{W}}_{V}^{(h)}\right\|_F\left\| \boldsymbol{Y}\right\|_F \\
&\leq2\sum_{h=1}^H \left\|\boldsymbol{\widetilde{W}}_{O}^{(h)}\right\|_F\left\|\boldsymbol{\widetilde{W}}_{V}^{(h)}\right\|_F\left\|\boldsymbol{Y}\right\|_F^3\left(\left\|\boldsymbol{\widetilde{W}}_{K}^{(h)}\right\|_F\left\|\boldsymbol{\widetilde{W}}_{Q}^{(h)} -\boldsymbol{{W}}_{Q}^{(h)} \right\|_F+\left\|\boldsymbol{\widetilde{W}}_{K}^{(h)} -\boldsymbol{{W}}_{K}^{(h)} \right\|_F\left\|\boldsymbol{{W}}_{Q}^{(h)}\right\|_F\right)\\
&\quad+\sqrt{n}\sum_{h=1}^H \left(\left\|\boldsymbol{\widetilde{W}}_{O}^{(h)}\right\|_F\left\|\boldsymbol{\widetilde{W}}_{V}^{(h)} -\boldsymbol{{W}}_{V}^{(h)} \right\|_F+\left\|\boldsymbol{\widetilde{W}}_{O}^{(h)} -\boldsymbol{{W}}_{O}^{(h)} \right\|_F\left\|\boldsymbol{{W}}_{V}^{(h)}\right\|_F\right)\left\| \boldsymbol{Y}\right\|_F \\
&\leq3d_{SA}^2\sqrt{n}HS^2B_{SA}^3\left\|\boldsymbol{Y}\right\|_F^3\varsigma,
\end{align*}
where in the fourth step we use Lemma \ref{norm of softmax} and Lemma \ref{Lip of softmax}. For $\boldsymbol{\mathcal{F}}_{EB}$, by definition we have
\begin{align*}
\left\|\boldsymbol{\widetilde{\mathcal{F}}}_{EB}(\boldsymbol{Z})-\boldsymbol{\mathcal{F}}_{EB}(\boldsymbol{Z})\right\|_F
&=\left\|\boldsymbol{\widetilde{W}}_{EB}\boldsymbol{Z}+\boldsymbol{\widetilde{B}}_{EB}-\boldsymbol{W}_{EB}\boldsymbol{Z}-\boldsymbol{B}_{EB}\right\|_F\\
&\leq\left\|\boldsymbol{\widetilde{W}}_{EB}-\boldsymbol{W}_{EB}\right\|_F\left\|\boldsymbol{Z}\right\|_F+\left\|\boldsymbol{\widetilde{B}}_{EB}-\boldsymbol{B}_{EB}\right\|_F\\
&\leq2\sqrt{d_{EB}d_{in}n}\left\|\boldsymbol{Z}\right\|_F\varsigma.
\end{align*}

\end{proof}

\begin{lemma}\label{Lip of FFSA}
Let $E\in\mathbb{R}_{>0}$. Let $\boldsymbol{X},\boldsymbol{\widetilde{X}}\in\mathbb{R}^{d_{FF}^{(in)}\times n},\boldsymbol{Y},\boldsymbol{\widetilde{Y}}\in\mathbb{R}^{d_{SA}\times n}$. Suppose $\left\|\boldsymbol{{Y}}\right\|_F,\left\|\boldsymbol{\widetilde{Y}}\right\|_F\leq E$. There holds
\begin{align*}
\left\|\boldsymbol{\mathcal{F}}_{FF}\left(\boldsymbol{\widetilde{X}}\right)-\boldsymbol{\mathcal{F}}_{FF}(\boldsymbol{X})\right\|_F
&\leq \sqrt{d_{FF}^{(in)}d_{FF}^{(out)}}W^{L-1}B_{FF}^L\left\|\boldsymbol{\widetilde{X}}-\boldsymbol{X}\right\|_F,\\
\left\|\boldsymbol{\mathcal{F}}_{SA}\left(\boldsymbol{\widetilde{Y}}\right)-\boldsymbol{\mathcal{F}}_{SA}(\boldsymbol{Y})\right\|_F&\leq6d_{SA}^2\sqrt{n}E^2HS^2B_{SA}^4\left\|\boldsymbol{\widetilde{Y}}-\boldsymbol{Y}\right\|_F.
\end{align*}
\end{lemma}
\begin{proof}
For $\boldsymbol{\mathcal{F}}_{FF}$, by definition we have
\begin{align*}
&\left\|\boldsymbol{\mathcal{F}}_{FF}\left(\boldsymbol{\widetilde{X}}\right)-\boldsymbol{\mathcal{F}}_{FF}(\boldsymbol{X})\right\|_F\\
&\leq\left\|\boldsymbol{W}_{L}\boldsymbol{\mathcal{F}}_{L-1}\left(\boldsymbol{\widetilde{X}}\right)-\boldsymbol{W}_{L}\boldsymbol{\mathcal{F}}_{L-1}\left(\boldsymbol{{X}}\right)\right\|_F\\
&\leq\left\|\boldsymbol{W}_{L}\right\|_F\left\|\boldsymbol{\mathcal{F}}_{L-1}\left(\boldsymbol{\widetilde{X}}\right)-\boldsymbol{\mathcal{F}}_{L-1}\left(\boldsymbol{{X}}\right)\right\|_F\\
&=\left\|\boldsymbol{W}_{L}\right\|_F\left\|\sigma_R\left(\boldsymbol{W}_{L-1}\boldsymbol{\mathcal{F}}_{L-2}\left(\boldsymbol{\widetilde{X}}\right)+\boldsymbol{B}_{L-1}\right)-\sigma_R\left(\boldsymbol{W}_{L-1}\boldsymbol{\mathcal{F}}_{L-2}\left(\boldsymbol{{X}}\right)+\boldsymbol{B}_{L-1}\right)\right\|_F\\
&\leq\left\|\boldsymbol{W}_{L}\right\|_F\left\|\boldsymbol{W}_{L-1}\right\|_F\left\|\boldsymbol{\mathcal{F}}_{L-2}\left(\boldsymbol{\widetilde{X}}\right)-\boldsymbol{\mathcal{F}}_{L-2}\left(\boldsymbol{{X}}\right)\right\|_F,
\end{align*}
where in the final step we use the fact that $\sigma_R$ is $1$-Lipschitz. Repeating this process, we obtain
\begin{align*}
\left\|\boldsymbol{\mathcal{F}}_{FF}\left(\boldsymbol{\widetilde{X}}\right)-\boldsymbol{\mathcal{F}}_{FF}(\boldsymbol{X})\right\|_F&\leq\left(\prod_{l=1}^{L}\left\|\boldsymbol{W}_{l}\right\|_F\right)\left\|\boldsymbol{\widetilde{X}}-\boldsymbol{X}\right\|_F\\
&\leq \sqrt{d_{FF}^{(in)}d_{FF}^{(out)}}W^{L-1}B_{FF}^L\left\|\boldsymbol{\widetilde{X}}-\boldsymbol{X}\right\|_F.
\end{align*}
For $\boldsymbol{\mathcal{F}}_{SA}$, by definition we have
\begin{align*}
&\left\|\boldsymbol{\mathcal{F}}_{SA}\left(\boldsymbol{\widetilde{Y}}\right)-\boldsymbol{\mathcal{F}}_{SA}(\boldsymbol{Y})\right\|_F\leq\\
&
\left\|\boldsymbol{\widetilde{Y}}-\boldsymbol{Y}\right\|_F+\sum_{h=1}^H \left\|\boldsymbol{W}_{O}^{(h)}\boldsymbol{W}_{V}^{(h)} \boldsymbol{\widetilde{Y}} \sigma_S\left( \boldsymbol{\widetilde{Y}}^{\top}\boldsymbol{W}_{K}^{(h)\top}\boldsymbol{W}_{Q}^{(h)}  \boldsymbol{\widetilde{Y}}\right)-\boldsymbol{W}_{O}^{(h)}\boldsymbol{W}_{V}^{(h)}  \boldsymbol{Y} \sigma_S\left( \boldsymbol{Y}^{\top}\boldsymbol{W}_{K}^{(h)\top}\boldsymbol{W}_{Q}^{(h)} \boldsymbol{Y}\right)\right\|_F.
\end{align*}
The term to be summed can be bounded by
\begin{align*}
&\left\|\boldsymbol{W}_{O}^{(h)}\boldsymbol{W}_{V}^{(h)} \boldsymbol{\widetilde{Y}} \sigma_S\left( \boldsymbol{\widetilde{Y}}^{\top}\boldsymbol{W}_{K}^{(h)\top}\boldsymbol{W}_{Q}^{(h)}  \boldsymbol{\widetilde{Y}}\right)-\boldsymbol{W}_{O}^{(h)}\boldsymbol{W}_{V}^{(h)} \boldsymbol{Y} \sigma_S\left( \boldsymbol{Y}^{\top}\boldsymbol{W}_{K}^{(h)\top}\boldsymbol{W}_{Q}^{(h)}  \boldsymbol{Y}\right)\right\|_F\\
&\leq\left\|\boldsymbol{W}_{O}^{(h)}\boldsymbol{W}_{V}^{(h)} \boldsymbol{\widetilde{Y}} \sigma_S\left( \boldsymbol{\widetilde{Y}}^{\top}\boldsymbol{W}_{K}^{(h)\top}\boldsymbol{W}_{Q}^{(h)}  \boldsymbol{\widetilde{Y}}\right)-\boldsymbol{W}_{O}^{(h)}\boldsymbol{W}_{V}^{(h)} \boldsymbol{\widetilde{Y}} \sigma_S\left( \boldsymbol{Y}^{\top}\boldsymbol{W}_{K}^{(h)\top}\boldsymbol{W}_{Q}^{(h)}  \boldsymbol{Y}\right)\right\|_F\\
&\quad+\left\|\boldsymbol{W}_{O}^{(h)}\boldsymbol{W}_{V}^{(h)} \boldsymbol{\widetilde{Y}} \sigma_S\left( \boldsymbol{Y}^{\top}\boldsymbol{W}_{K}^{(h)\top}\boldsymbol{W}_{Q}^{(h)}  \boldsymbol{Y}\right)-\boldsymbol{W}_{O}^{(h)}\boldsymbol{W}_{V}^{(h)} \boldsymbol{Y} \sigma_S\left( \boldsymbol{Y}^{\top}\boldsymbol{W}_{K}^{(h)\top}\boldsymbol{W}_{Q}^{(h)}  \boldsymbol{Y}\right)\right\|_F\\
&\leq\left\|\boldsymbol{W}_{O}^{(h)}\right\|_F\left\|\boldsymbol{W}_{V}^{(h)}\right\|_F\left\|\boldsymbol{\widetilde{Y}}\right\|_F\left\|\sigma_S\left( \boldsymbol{\widetilde{Y}}^{\top}\boldsymbol{W}_{K}^{(h)\top}\boldsymbol{W}_{Q}^{(h)}  \boldsymbol{\widetilde{Y}}\right)-\sigma_S\left( \boldsymbol{Y}^{\top}\boldsymbol{W}_{K}^{(h)\top}\boldsymbol{W}_{Q}^{(h)}  \boldsymbol{Y}\right)\right\|_F\\
&\quad+\left\|\boldsymbol{W}_{O}^{(h)}\right\|_F\left\|\boldsymbol{W}_{V}^{(h)}\right\|_F\left\|\boldsymbol{\widetilde{Y}}-\boldsymbol{Y}\right\|_F\left\|\sigma_S\left( \boldsymbol{Y}^{\top}\boldsymbol{W}_{K}^{(h)\top}\boldsymbol{W}_{Q}^{(h)}  \boldsymbol{Y}\right)\right\|_F\\
&\leq2\left\|\boldsymbol{W}_{O}^{(h)}\right\|_F\left\|\boldsymbol{W}_{V}^{(h)}\right\|_F\left\|\boldsymbol{\widetilde{Y}}\right\|_F\left\| \boldsymbol{\widetilde{Y}}^{\top}\boldsymbol{W}_{K}^{(h)\top}\boldsymbol{W}_{Q}^{(h)}  \boldsymbol{\widetilde{Y}}- \boldsymbol{Y}^{\top}\boldsymbol{W}_{K}^{(h)\top}\boldsymbol{W}_{Q}^{(h)}  \boldsymbol{Y}\right\|_F\\
&\quad+\sqrt{n}\left\|\boldsymbol{W}_{O}^{(h)}\right\|_F\left\|\boldsymbol{W}_{V}^{(h)}\right\|_F\left\|\boldsymbol{\widetilde{Y}}-\boldsymbol{Y}\right\|_F\\
&\leq2\left\|\boldsymbol{W}_{O}^{(h)}\right\|_F\left\|\boldsymbol{W}_{V}^{(h)}\right\|_F\left\|\boldsymbol{\widetilde{Y}}\right\|_F^2\left\|\boldsymbol{W}_{K}^{(h)}\right\|_F\left\|\boldsymbol{W}_{Q}^{(h)}\right\|_F\left\|\boldsymbol{\widetilde{Y}}-  \boldsymbol{Y}\right\|_F\\
&\quad+2\left\|\boldsymbol{W}_{O}^{(h)}\right\|_F\left\|\boldsymbol{W}_{V}^{(h)}\right\|_F\left\|\boldsymbol{\widetilde{Y}}\right\|_F\left\|\boldsymbol{{Y}}\right\|_F\left\|\boldsymbol{W}_{K}^{(h)}\right\|_F\left\|\boldsymbol{W}_{Q}^{(h)}\right\|_F\left\|\boldsymbol{\widetilde{Y}}-  \boldsymbol{Y}\right\|_F\\
&\quad+\sqrt{n}\left\|\boldsymbol{W}_{O}^{(h)}\right\|_F\left\|\boldsymbol{W}_{V}^{(h)}\right\|_F\left\|\boldsymbol{\widetilde{Y}}-\boldsymbol{Y}\right\|_F\\
&\leq(4d_{SA}^2E^2S^2B_{SA}^4+d_{SA}\sqrt{n}SB_{SA}^2)\left\|\boldsymbol{\widetilde{Y}}-\boldsymbol{Y}\right\|_F\\
&\leq 5d_{SA}^2\sqrt{n}E^2S^2B_{SA}^4\left\|\boldsymbol{\widetilde{Y}}-\boldsymbol{Y}\right\|_F,
\end{align*}
where in the third step we use Lemma \ref{norm of softmax} and Lemma \ref{Lip of softmax}. Hence
\begin{align*}
\left\|\boldsymbol{\mathcal{F}}_{SA}\left(\boldsymbol{\widetilde{Y}}\right)-\boldsymbol{\mathcal{F}}_{SA}(\boldsymbol{Y})\right\|_F\leq6d_{SA}^2\sqrt{n}E^2HS^2B_{SA}^4\left\|\boldsymbol{\widetilde{Y}}-\boldsymbol{Y}\right\|_F.
\end{align*}
\end{proof}
We prove Lemma \ref{covering number bound} by employing Lemmas \ref{bound of FFSA} - \ref{Lip of FFSA}.
\begin{proof}[Proof of Lemma \ref{covering number bound}]
We examine the difference of $f_{\boldsymbol{\widetilde{T}}},f_{\boldsymbol{{T}}}\in\mathcal{F}_{\mathcal{T}}$, where each trainable parameter in $f_{\boldsymbol{\widetilde{T}}}$ and $f_{\boldsymbol{{T}}}$ differs by at most $\varsigma$. Since
\begin{align}\label{Lip-0.5}
\left|f_{\boldsymbol{\widetilde{T}}}(\boldsymbol{X})-f_{\boldsymbol{T}}(\boldsymbol{X})\right|
=\left|\left \langle \boldsymbol{\widetilde{T}}(\boldsymbol{X})-\boldsymbol{T}(\boldsymbol{X}),\boldsymbol{E}_{11} \right \rangle\right|
\leq\left \| \boldsymbol{\widetilde{T}}(\boldsymbol{X})-\boldsymbol{T}(\boldsymbol{X})\right\|_F,
\end{align}
what we need is an upper bound of $\left\|\boldsymbol{\widetilde{T}}(\boldsymbol{X})-\boldsymbol{T}(\boldsymbol{X})\right\|_F$. We split $\left\|\boldsymbol{\widetilde{T}}(\boldsymbol{X})-\boldsymbol{T}(\boldsymbol{X})\right\|_F$ in the following way:
\begin{align}
&\left\|\boldsymbol{\widetilde{T}}(\boldsymbol{X})-\boldsymbol{T}(\boldsymbol{X})\right\|_F\nonumber\\
&=\left\|\boldsymbol{\mathcal{\widetilde{F}}}_{FF}^{(K)}\circ\boldsymbol{\mathcal{\widetilde{F}}}_{SA}^{(K)}\circ\boldsymbol{\mathcal{\widetilde{F}}}_{FF}^{(K-1)}\circ\cdots\circ\boldsymbol{\mathcal{\widetilde{F}}}_{FF}^{(1)}\circ\boldsymbol{\mathcal{\widetilde{F}}}_{SA}^{(1)}\circ\boldsymbol{\mathcal{\widetilde{F}}}_{FF}^{(0)}\circ\boldsymbol{\mathcal{\widetilde{F}}}_{EB}(\boldsymbol{X})\right.\nonumber\\
&\qquad\left.-\boldsymbol{\mathcal{F}}_{FF}^{(K)}\circ\boldsymbol{\mathcal{F}}_{SA}^{(K)}\circ\boldsymbol{\mathcal{F}}_{FF}^{(K-1)}\circ\cdots\circ\boldsymbol{\mathcal{F}}_{FF}^{(1)}\circ\boldsymbol{\mathcal{F}}_{SA}^{(1)}\circ\boldsymbol{\mathcal{F}}_{FF}^{(0)}\circ\boldsymbol{\mathcal{F}}_{EB}(\boldsymbol{X})\right\|_F\nonumber\\
&\leq\left\|\boldsymbol{\mathcal{\widetilde{F}}}_{FF}^{(K)}\circ\boldsymbol{\mathcal{\widetilde{F}}}_{SA}^{(K)}\circ\boldsymbol{\mathcal{\widetilde{F}}}_{FF}^{(K-1)}\circ\cdots\circ\boldsymbol{\mathcal{\widetilde{F}}}_{FF}^{(1)}\circ\boldsymbol{\mathcal{\widetilde{F}}}_{SA}^{(1)}\circ\boldsymbol{\mathcal{\widetilde{F}}}_{FF}^{(0)}\circ\boldsymbol{\mathcal{\widetilde{F}}}_{EB}(\boldsymbol{X})\right.\nonumber\\
&\qquad\left.-\boldsymbol{\mathcal{\widetilde{F}}}_{FF}^{(K)}\circ\boldsymbol{\mathcal{\widetilde{F}}}_{SA}^{(K)}\circ\boldsymbol{\mathcal{\widetilde{F}}}_{FF}^{(K-1)}\circ\cdots\circ\boldsymbol{\mathcal{\widetilde{F}}}_{FF}^{(1)}\circ\boldsymbol{\mathcal{\widetilde{F}}}_{SA}^{(1)}\circ\boldsymbol{\mathcal{\widetilde{F}}}_{FF}^{(0)}\circ\boldsymbol{\mathcal{F}}_{EB}(\boldsymbol{X})\right\|_F\nonumber\\
&\quad+\left\|\boldsymbol{\mathcal{\widetilde{F}}}_{FF}^{(K)}\circ\boldsymbol{\mathcal{\widetilde{F}}}_{SA}^{(K)}\circ\boldsymbol{\mathcal{\widetilde{F}}}_{FF}^{(K-1)}\circ\cdots\circ\boldsymbol{\mathcal{\widetilde{F}}}_{FF}^{(1)}\circ\boldsymbol{\mathcal{\widetilde{F}}}_{SA}^{(1)}\circ\boldsymbol{\mathcal{\widetilde{F}}}_{FF}^{(0)}\circ\boldsymbol{\mathcal{F}}_{EB}(\boldsymbol{X})\right.\nonumber\\
&\qquad\left.-\boldsymbol{\mathcal{\widetilde{F}}}_{FF}^{(K)}\circ\boldsymbol{\mathcal{\widetilde{F}}}_{SA}^{(K)}\circ\boldsymbol{\mathcal{\widetilde{F}}}_{FF}^{(K-1)}\circ\cdots\circ\boldsymbol{\mathcal{\widetilde{F}}}_{FF}^{(1)}\circ\boldsymbol{\mathcal{\widetilde{F}}}_{SA}^{(1)}\circ\boldsymbol{\mathcal{{F}}}_{FF}^{(0)}\circ\boldsymbol{\mathcal{F}}_{EB}(\boldsymbol{X})\right\|_F\nonumber\\
&\quad+\left\|\boldsymbol{\mathcal{\widetilde{F}}}_{FF}^{(K)}\circ\boldsymbol{\mathcal{\widetilde{F}}}_{SA}^{(K)}\circ\boldsymbol{\mathcal{\widetilde{F}}}_{FF}^{(K-1)}\circ\cdots\circ\boldsymbol{\mathcal{\widetilde{F}}}_{FF}^{(1)}\circ\boldsymbol{\mathcal{\widetilde{F}}}_{SA}^{(1)}\circ\boldsymbol{\mathcal{{F}}}_{FF}^{(0)}\circ\boldsymbol{\mathcal{F}}_{EB}(\boldsymbol{X})\right.\nonumber\\
&\qquad\quad\left.-\boldsymbol{\mathcal{\widetilde{F}}}_{FF}^{(K)}\circ\boldsymbol{\mathcal{\widetilde{F}}}_{SA}^{(K)}\circ\boldsymbol{\mathcal{\widetilde{F}}}_{FF}^{(K-1)}\circ\cdots\circ\boldsymbol{\mathcal{\widetilde{F}}}_{FF}^{(1)}\circ\boldsymbol{\mathcal{{F}}}_{SA}^{(1)}\circ\boldsymbol{\mathcal{{F}}}_{FF}^{(0)}\circ\boldsymbol{\mathcal{F}}_{EB}(\boldsymbol{X})\right\|_F\nonumber\\
&\quad+\cdots\nonumber\\
&\quad+\left\|\boldsymbol{\mathcal{\widetilde{F}}}_{FF}^{(K)}\circ\boldsymbol{\mathcal{\widetilde{F}}}_{SA}^{(K)}\circ\boldsymbol{\mathcal{{F}}}_{FF}^{(K-1)}\circ\cdots\circ\boldsymbol{\mathcal{{F}}}_{FF}^{(1)}\circ\boldsymbol{\mathcal{{F}}}_{SA}^{(1)}\circ\boldsymbol{\mathcal{{F}}}_{FF}^{(0)}\circ\boldsymbol{\mathcal{F}}_{EB}(\boldsymbol{X})\right.\nonumber\\
&\qquad\quad\left.-\boldsymbol{\mathcal{\widetilde{F}}}_{FF}^{(K)}\circ\boldsymbol{\mathcal{{F}}}_{SA}^{(K)}\circ\boldsymbol{\mathcal{{F}}}_{FF}^{(K-1)}\circ\cdots\circ\boldsymbol{\mathcal{{F}}}_{FF}^{(1)}\circ\boldsymbol{\mathcal{{F}}}_{SA}^{(1)}\circ\boldsymbol{\mathcal{{F}}}_{FF}^{(0)}\circ\boldsymbol{\mathcal{F}}_{EB}(\boldsymbol{X})\right\|_F\nonumber\\
&\quad+\left\|\boldsymbol{\mathcal{\widetilde{F}}}_{FF}^{(K)}\circ\boldsymbol{\mathcal{{F}}}_{SA}^{(K)}\circ\boldsymbol{\mathcal{{F}}}_{FF}^{(K-1)}\circ\cdots\circ\boldsymbol{\mathcal{{F}}}_{FF}^{(1)}\circ\boldsymbol{\mathcal{{F}}}_{SA}^{(1)}\circ\boldsymbol{\mathcal{{F}}}_{FF}^{(0)}\circ\boldsymbol{\mathcal{F}}_{EB}(\boldsymbol{X})\right.\nonumber\\
&\qquad\quad\left.-\boldsymbol{\mathcal{{F}}}_{FF}^{(K)}\circ\boldsymbol{\mathcal{{F}}}_{SA}^{(K)}\circ\boldsymbol{\mathcal{{F}}}_{FF}^{(K-1)}\circ\cdots\circ\boldsymbol{\mathcal{{F}}}_{FF}^{(1)}\circ\boldsymbol{\mathcal{{F}}}_{SA}^{(1)}\circ\boldsymbol{\mathcal{{F}}}_{FF}^{(0)}\circ\boldsymbol{\mathcal{F}}_{EB}(\boldsymbol{X})\right\|_F.\label{Lip0}
\end{align}
We first handle the term that difference appears in the $k$-th feedforward block:
\begin{align*}
&\left\|\boldsymbol{\mathcal{\widetilde{F}}}_{FF}^{(K)}\circ\cdots\circ\boldsymbol{\mathcal{\widetilde{F}}}_{SA}^{(k+1)}\circ\boldsymbol{\mathcal{\widetilde{F}}}_{FF}^{(k)}\circ\boldsymbol{\mathcal{{F}}}_{SA}^{(k)}\circ\cdots\circ\boldsymbol{\mathcal{{F}}}_{FF}^{(0)}\circ\boldsymbol{\mathcal{F}}_{EB}(\boldsymbol{X})\right.\\
&\quad\left.-\boldsymbol{\mathcal{\widetilde{F}}}_{FF}^{(K)}\circ\cdots\circ\boldsymbol{\mathcal{\widetilde{F}}}_{SA}^{(k+1)}\circ\boldsymbol{\mathcal{{F}}}_{FF}^{(k)}\circ\boldsymbol{\mathcal{{F}}}_{SA}^{(k)}\circ\cdots\circ\boldsymbol{\mathcal{{F}}}_{FF}^{(0)}\circ\boldsymbol{\mathcal{F}}_{EB}(\boldsymbol{X})\right\|_F, 
\end{align*}
where $k\in\{0,1,\cdots,K\}$. Applying Lemma \ref{bound of FFSA} repeatedly, we can derive that
\begin{align*}
&\left.\begin{matrix}
\left\|\boldsymbol{\mathcal{\widetilde{F}}}_{FF}^{(K-1)}\circ\cdots\circ\boldsymbol{\mathcal{\widetilde{F}}}_{SA}^{(k+1)}\circ\boldsymbol{\mathcal{\widetilde{F}}}_{FF}^{(k)}\circ\boldsymbol{\mathcal{{F}}}_{SA}^{(k)}\circ\cdots\circ\boldsymbol{\mathcal{{F}}}_{FF}^{(0)}\circ\boldsymbol{\mathcal{F}}_{EB}(\boldsymbol{X})\right\|_F \\
\left\|\boldsymbol{\mathcal{\widetilde{F}}}_{FF}^{(K-1)}\circ\cdots\circ\boldsymbol{\mathcal{\widetilde{F}}}_{SA}^{(k+1)}\circ\boldsymbol{\mathcal{{F}}}_{FF}^{(k)}\circ\boldsymbol{\mathcal{{F}}}_{SA}^{(k)}\circ\cdots\circ\boldsymbol{\mathcal{{F}}}_{FF}^{(0)}\circ\boldsymbol{\mathcal{F}}_{EB}(\boldsymbol{X})\right\|_F
\end{matrix}\right\}\\
&\leq4^{K}n^{K+1/2}d_{in}^{1/2}d_0\left(\prod_{k=1}^{K-1}d_k^2\right)d_{K}^{1/2}H^{K-1}S^{K-1}L^{K}W^{(L-1)K}B_{EB}B_{FF}^{LK}B_{SA}^{2(K-1)}.
\end{align*}
Using the above estimates and Lemma \ref{Lip of FFSA}, we have
\begin{align*}
&\left\|\boldsymbol{\mathcal{\widetilde{F}}}_{FF}^{(K)}\circ\cdots\circ\boldsymbol{\mathcal{\widetilde{F}}}_{SA}^{(k+1)}\circ\boldsymbol{\mathcal{\widetilde{F}}}_{FF}^{(k)}\circ\boldsymbol{\mathcal{{F}}}_{SA}^{(k)}\circ\cdots\circ\boldsymbol{\mathcal{{F}}}_{FF}^{(0)}\circ\boldsymbol{\mathcal{F}}_{EB}(\boldsymbol{X})\right.\\
&\quad\left.-\boldsymbol{\mathcal{\widetilde{F}}}_{FF}^{(K)}\circ\cdots\circ\boldsymbol{\mathcal{\widetilde{F}}}_{SA}^{(k+1)}\circ\boldsymbol{\mathcal{{F}}}_{FF}^{(k)}\circ\boldsymbol{\mathcal{{F}}}_{SA}^{(k)}\circ\cdots\circ\boldsymbol{\mathcal{{F}}}_{FF}^{(0)}\circ\boldsymbol{\mathcal{F}}_{EB}(\boldsymbol{X})\right\|_F\\   
&\leq\sqrt{d_Kd_{out}}M^{L-1}B_{FF}^L\\
&\quad\left\|\boldsymbol{\mathcal{\widetilde{F}}}_{SA}^{(K)}\circ\cdots\circ\boldsymbol{\mathcal{\widetilde{F}}}_{SA}^{(k+1)}\circ\boldsymbol{\mathcal{\widetilde{F}}}_{FF}^{(k)}\circ\boldsymbol{\mathcal{{F}}}_{SA}^{(k)}\circ\cdots\circ\boldsymbol{\mathcal{{F}}}_{FF}^{(0)}\circ\boldsymbol{\mathcal{F}}_{EB}(\boldsymbol{X})\right.\\
&\qquad\left.-\boldsymbol{\mathcal{\widetilde{F}}}_{SA}^{(K)}\circ\cdots\circ\boldsymbol{\mathcal{\widetilde{F}}}_{SA}^{(k+1)}\circ\boldsymbol{\mathcal{{F}}}_{FF}^{(k)}\circ\boldsymbol{\mathcal{{F}}}_{SA}^{(k)}\circ\cdots\circ\boldsymbol{\mathcal{{F}}}_{FF}^{(0)}\circ\boldsymbol{\mathcal{F}}_{EB}(\boldsymbol{X})\right\|_F\\
&\leq 6\cdot4^{2K}n^{2K+3/2}d_{in}d_0^2\left(\prod_{k=1}^{K-1}d_k^4\right)d_{K}^{7/2}d_{out}^{1/2}H^{2K-1}S^{2K}L^{2K}W^{(L-1)(2K+1)}B_{EB}^2B_{FF}^{L(2K+1)}B_{SA}^{4K}\\
&\quad\left\|\boldsymbol{\mathcal{\widetilde{F}}}_{FF}^{(K-1)}\circ\cdots\circ\boldsymbol{\mathcal{\widetilde{F}}}_{SA}^{(k+1)}\circ\boldsymbol{\mathcal{\widetilde{F}}}_{FF}^{(k)}\circ\boldsymbol{\mathcal{{F}}}_{SA}^{(k)}\circ\cdots\circ\boldsymbol{\mathcal{{F}}}_{FF}^{(0)}\circ\boldsymbol{\mathcal{F}}_{EB}(\boldsymbol{X})\right.\\
&\qquad\left.-\boldsymbol{\mathcal{\widetilde{F}}}_{FF}^{(K-1)}\circ\cdots\circ\boldsymbol{\mathcal{\widetilde{F}}}_{SA}^{(k+1)}\circ\boldsymbol{\mathcal{{F}}}_{FF}^{(k)}\circ\boldsymbol{\mathcal{{F}}}_{SA}^{(k)}\circ\cdots\circ\boldsymbol{\mathcal{{F}}}_{FF}^{(0)}\circ\boldsymbol{\mathcal{F}}_{EB}(\boldsymbol{X})\right\|_F.
\end{align*}
Repeating this process and making use of the following estimates (obtained by applying Lemma \ref{bound of FFSA} repeatedly):
\begin{align*}
&\left.\begin{matrix}
\left\|\boldsymbol{\mathcal{\widetilde{F}}}_{FF}^{(k')}\circ\cdots\circ\boldsymbol{\mathcal{\widetilde{F}}}_{SA}^{(k+1)}\circ\boldsymbol{\mathcal{\widetilde{F}}}_{FF}^{(k)}\circ\boldsymbol{\mathcal{{F}}}_{SA}^{(k)}\circ\cdots\circ\boldsymbol{\mathcal{{F}}}_{FF}^{(0)}\circ\boldsymbol{\mathcal{F}}_{EB}(\boldsymbol{X})\right\|_F \\
\left\|\boldsymbol{\mathcal{\widetilde{F}}}_{FF}^{(k')}\circ\cdots\circ\boldsymbol{\mathcal{\widetilde{F}}}_{SA}^{(k+1)}\circ\boldsymbol{\mathcal{{F}}}_{FF}^{(k)}\circ\boldsymbol{\mathcal{{F}}}_{SA}^{(k)}\circ\cdots\circ\boldsymbol{\mathcal{{F}}}_{FF}^{(0)}\circ\boldsymbol{\mathcal{F}}_{EB}(\boldsymbol{X})\right\|_F
\end{matrix}\right\}\\
&\leq4^{k'+1}n^{k'+3/2}d_{in}^{1/2}d_0\left(\prod_{k=1}^{k'}d_k^2\right)d_{k'+1}^{1/2}H^{k'}S^{k'}L^{k'+1}W^{(L-1)(k'+1)}B_{EB}B_{FF}^{L(k'+1)}B_{SA}^{2k'},
\end{align*}
where $k'\in\{k+1,\cdots,K-1\}$, we derive that 
\begin{align}
&\left\|\boldsymbol{\mathcal{\widetilde{F}}}_{FF}^{(K)}\circ\cdots\circ\boldsymbol{\mathcal{\widetilde{F}}}_{SA}^{(k+1)}\circ\boldsymbol{\mathcal{\widetilde{F}}}_{FF}^{(k)}\circ\boldsymbol{\mathcal{{F}}}_{SA}^{(k)}\circ\cdots\circ\boldsymbol{\mathcal{{F}}}_{FF}^{(0)}\circ\boldsymbol{\mathcal{F}}_{EB}(\boldsymbol{X})\right.\nonumber\\
&\quad\left.-\boldsymbol{\mathcal{\widetilde{F}}}_{FF}^{(K)}\circ\cdots\circ\boldsymbol{\mathcal{\widetilde{F}}}_{SA}^{(k+1)}\circ\boldsymbol{\mathcal{{F}}}_{FF}^{(k)}\circ\boldsymbol{\mathcal{{F}}}_{SA}^{(k)}\circ\cdots\circ\boldsymbol{\mathcal{{F}}}_{FF}^{(0)}\circ\boldsymbol{\mathcal{F}}_{EB}(\boldsymbol{X})\right\|_F\nonumber\\
&\leq 6^{K-k}4^{(K-k)(K-k+1)}n^{(K-k)(K-k+5/2)}d_{in}^{K-k}d_0^{2(K-k)}\nonumber\\
&\quad\left(\prod_{k'=1}^{K-1}\prod_{k''=1}^{k'}d_{k''}^4\right)\left(\prod_{k'=k+1}^{K}d_{k'}\right)d_{k+1}^{5/2}\left(\prod_{k'=k+2}^{K}d_{k'}^3\right)d_{out}^{1/2}\nonumber\\
&\quad H^{(K-k)^2}S^{(K-k)(K-k+1)}L^{(K-k)(K-k+1)}W^{(L-1)(K-k)(K-k+2)}\nonumber\\
&\quad B_{EB}^{2(K-k)}B_{FF}^{L(K-k)(K-k+2)}B_{SA}^{2(K-k)(K-k+1)}\nonumber\\
&\quad\left\|\boldsymbol{\mathcal{\widetilde{F}}}_{FF}^{(k)}\circ\boldsymbol{\mathcal{{F}}}_{SA}^{(k)}\circ\cdots\circ\boldsymbol{\mathcal{{F}}}_{FF}^{(0)}\circ\boldsymbol{\mathcal{F}}_{EB}(\boldsymbol{X})
-\boldsymbol{\mathcal{{F}}}_{FF}^{(k)}\circ\boldsymbol{\mathcal{{F}}}_{SA}^{(k)}\circ\cdots\circ\boldsymbol{\mathcal{{F}}}_{FF}^{(0)}\circ\boldsymbol{\mathcal{F}}_{EB}(\boldsymbol{X})\right\|_F.\label{Lip1}
\end{align}
Applying Lemma \ref{bound of FFSA} repeatedly, we have
\begin{align*}
&\left\|\boldsymbol{\mathcal{{F}}}_{SA}^{(k)}\circ\cdots\circ\boldsymbol{\mathcal{{F}}}_{FF}^{(0)}\circ\boldsymbol{\mathcal{F}}_{EB}(\boldsymbol{X})
\right\|_F\\
&\leq4^{k+1/2}n^{k+1}d_{in}^{1/2}d_0\left(\prod_{k'=1}^{k-1}d_{k'}^2\right){d_k^{3/2}}H^{k}S^{k}L^{k}W^{(L-1)k}B_{EB}B_{FF}^{Lk}B_{SA}^{2k}.
\end{align*}
Using the above estimate and Lemma \ref{difference of FFSA}, we have
\begin{align}
&\left\|\boldsymbol{\mathcal{\widetilde{F}}}_{FF}^{(k)}\circ\boldsymbol{\mathcal{{F}}}_{SA}^{(k)}\circ\cdots\circ\boldsymbol{\mathcal{{F}}}_{FF}^{(0)}\circ\boldsymbol{\mathcal{F}}_{EB}(\boldsymbol{X})
-\boldsymbol{\mathcal{{F}}}_{FF}^{(k)}\circ\boldsymbol{\mathcal{{F}}}_{SA}^{(k)}\circ\cdots\circ\boldsymbol{\mathcal{{F}}}_{FF}^{(0)}\circ\boldsymbol{\mathcal{F}}_{EB}(\boldsymbol{X})\right\|_F\nonumber\\
&\leq4^{k+3/2}n^{k+2}d_{in}^{1/2}d_0\left(\prod_{k'=1}^{k-1}d_{k'}^2\right){d_k^{5/2}}{d_{k+1}^{3/2}}H^{k}S^{k}L^{k+2}W^{(k+2)L-k-3/2}B_{EB}B_{FF}^{(k+2)L-1}B_{SA}^{2k}\varsigma.\label{Lip2}
\end{align}
Plugging \eqref{Lip2} into \eqref{Lip1} and simplifying the expression, we derive that

\begin{align}
&\left\|\boldsymbol{\mathcal{\widetilde{F}}}_{FF}^{(K)}\circ\cdots\circ\boldsymbol{\mathcal{\widetilde{F}}}_{SA}^{(k+1)}\circ\boldsymbol{\mathcal{\widetilde{F}}}_{FF}^{(k)}\circ\boldsymbol{\mathcal{{F}}}_{SA}^{(k)}\circ\cdots\circ\boldsymbol{\mathcal{{F}}}_{FF}^{(0)}\circ\boldsymbol{\mathcal{F}}_{EB}(\boldsymbol{X})\right.\nonumber\\
&\quad\left.-\boldsymbol{\mathcal{\widetilde{F}}}_{FF}^{(K)}\circ\cdots\circ\boldsymbol{\mathcal{\widetilde{F}}}_{SA}^{(k+1)}\circ\boldsymbol{\mathcal{{F}}}_{FF}^{(k)}\circ\boldsymbol{\mathcal{{F}}}_{SA}^{(k)}\circ\cdots\circ\boldsymbol{\mathcal{{F}}}_{FF}^{(0)}\circ\boldsymbol{\mathcal{F}}_{EB}(\boldsymbol{X})\right\|_F\nonumber\\
&\leq 6^{K}4^{K^2+K+3/2}n^{K^2+5K/2+2}d_{in}^{K+1/2} d_0^{2K+1}\left(\prod_{k'=1}^{K}d_{k'}^{4(K-k')+5}\right)d_{out}^{1/2}\nonumber\\
&\quad H^{K^2}S^{K^2+K}L^{K^2+K+2}W^{(L-1)K(K+2)+2L-3/2}B_{EB}^{2K+1}B_{FF}^{LK(K+2)+2L-1}B_{SA}^{2K(K+1)}\varsigma.\label{Lip3}
\end{align}
In a similar manner, we can derive an upper bound for the term that difference appears in the $k$-th self-attention layer ($k\in\{1,2,\cdots,K\}$):

\begin{align}
&\left\|\boldsymbol{\mathcal{\widetilde{F}}}_{FF}^{(K)}\circ\cdots\circ\boldsymbol{\mathcal{\widetilde{F}}}_{FF}^{(k)}\circ\boldsymbol{\mathcal{\widetilde{F}}}_{SA}^{(k)}\circ\boldsymbol{\mathcal{{F}}}_{FF}^{(k-1)}\circ\cdots\circ\boldsymbol{\mathcal{{F}}}_{FF}^{(0)}\circ\boldsymbol{\mathcal{F}}_{EB}(\boldsymbol{X})\right.\nonumber\\
&\quad\left.-\boldsymbol{\mathcal{\widetilde{F}}}_{FF}^{(K)}\circ\cdots\circ\boldsymbol{\mathcal{\widetilde{F}}}_{FF}^{(k)}\circ\boldsymbol{\mathcal{{F}}}_{SA}^{(k)}\circ\boldsymbol{\mathcal{{F}}}_{FF}^{(k-1)}\circ\cdots\circ\boldsymbol{\mathcal{{F}}}_{FF}^{(0)}\circ\boldsymbol{\mathcal{F}}_{EB}(\boldsymbol{X})\right\|_F\nonumber\\
&\leq 3\cdot6^{K-1}4^{K^2+K+4}n^{K^2+5K/2+3}d_{in}^{K+1/2}d_0^{2K+1}\left(\prod_{k'=1}^{K}d_{k'}^{4(K-k')+6}\right)d_{out}^{1/2}\nonumber\\
&\quad H^{K^2+K-1}S^{K^2+2K+1}L^{K^2+2K+3}W^{(L-1)(K^2+3K+3)}B_{EB}^{2K+1}B_{FF}^{L(K^2+3K+3)}B_{SA}^{2(K^2+2K+1)}\varsigma,\label{Lip4}
\end{align}
and the term that difference appears in the embedding layer:
\begin{align}
&\left\|\boldsymbol{\mathcal{\widetilde{F}}}_{FF}^{(K)}\circ\boldsymbol{\mathcal{\widetilde{F}}}_{SA}^{(K)}\circ\boldsymbol{\mathcal{\widetilde{F}}}_{FF}^{(K-1)}\circ\cdots\circ\boldsymbol{\mathcal{\widetilde{F}}}_{FF}^{(1)}\circ\boldsymbol{\mathcal{\widetilde{F}}}_{SA}^{(1)}\circ\boldsymbol{\mathcal{\widetilde{F}}}_{FF}^{(0)}\circ\boldsymbol{\mathcal{\widetilde{F}}}_{EB}(\boldsymbol{X})\right.\nonumber\\
&\quad\left.-\boldsymbol{\mathcal{\widetilde{F}}}_{FF}^{(K)}\circ\boldsymbol{\mathcal{\widetilde{F}}}_{SA}^{(K)}\circ\boldsymbol{\mathcal{\widetilde{F}}}_{FF}^{(K-1)}\circ\cdots\circ\boldsymbol{\mathcal{\widetilde{F}}}_{FF}^{(1)}\circ\boldsymbol{\mathcal{\widetilde{F}}}_{SA}^{(1)}\circ\boldsymbol{\mathcal{\widetilde{F}}}_{FF}^{(0)}\circ\boldsymbol{\mathcal{F}}_{EB}(\boldsymbol{X})\right\|_F\nonumber\\
&\leq 3\cdot6^{K-1}4^{K^2+K+4}n^{K^2+5K/2+3}d_{in}^{K+1/2}d_0^{2K+1}\left(\prod_{k'=1}^{K}d_{k'}^{4(K-k')+6}\right)d_{out}^{1/2}\nonumber\\
&\quad H^{K^2+K-1}S^{K^2+2K+1}L^{K^2+2K+3}W^{(L-1)(K^2+3K+3)}B_{EB}^{2K+1}B_{FF}^{L(K^2+3K+3)}B_{SA}^{2(K^2+2K+1)}\varsigma.\label{Lip5}
\end{align}
Plugging \eqref{Lip3}-\eqref{Lip5} into \eqref{Lip0} yields
\begin{align*}
&\left\|\boldsymbol{\widetilde{T}}(\boldsymbol{X})-\boldsymbol{T}(\boldsymbol{X})\right\|_F\\
&\leq (2K+2)6^{K}4^{K^2+K+4}n^{K^2+5K/2+3}d_{in}^{K+1/2} d_0^{2K+1}\left(\prod_{k'=1}^{K}d_{k'}^{4(K-k')+6}\right)d_{out}^{1/2}\\
&\quad H^{K^2+K-1}S^{K^2+2K+1}L^{K^2+2K+3}W^{(L-1)(K^2+3K+3)}B_{EB}^{2K+1}B_{FF}^{L(K^2+3K+3)}B_{SA}^{2(K^2+2K+1)}\varsigma.
\end{align*}
Plugging this result into \eqref{Lip-0.5}, we finally obtain
\begin{align*}
\left|f_{\boldsymbol{\widetilde{T}}}(\boldsymbol{X})-f_{\boldsymbol{T}}(\boldsymbol{X})\right|\leq\mathfrak{L}\varsigma
\end{align*}
with $\mathfrak{L}$ defined in \eqref{Lip-1}.
We obtain the desired covering number bound by discretizing the trainable parameters in $f_{\boldsymbol{T}}$ with ${\varsigma}/{\mathfrak{L}}$ grid size.

\end{proof}

\section{Conclusions}\label{Conclusions}

In this work, we show that standard Transformers can approximate Hölder functions $  C^{s,\lambda}\left([0,1]^{d\times n}\right)  $ under the $L^t$ distance with arbitrary precision. Building upon this approximation result, we demonstrate that standard Transformers achieve the minimax optimal rate in nonparametric regression for Hölder target functions. By introducing the size tuple and the dimension vector, we provide a fine-grained characterization of Transformer structures. These findings demonstrate the powerful ability of Transformers at the theoretical level.

There are several promising directions for future research. For example, it is crucial to establish a theoretical foundation for Transformers in broader applications, such as pre-training in large language models (LLMs) and vision Transformers (ViT) in computer vision tasks. Recent theoretical studies have investigated the approximation and generalization errors of Transformers in the setting of in-context learning (ICL) \cite{kim2024transformers,shen2026understanding,ching2026efficient}. Their findings indicate that in ICL, only when both the number of tokens and the number of pre-training sequences are sufficiently large can the final error be made sufficiently small. However, these results often rely on architectural simplifications: \cite{kim2024transformers,ching2026efficient} utilize linear attention instead of softmax attention, while \cite{shen2026understanding} employs softmax only in the final layer of the network. Deriving the convergence rates of a standard Transformer in ICL settings remains an open problem. Furthermore, while the present work focuses on the approximation error and generalization error of Transformers, the optimization error incurred during the training process, specifically the convergence rates of Transformers under various optimization algorithms, also warrants further investigation.

\bibliographystyle{plain}
\bibliography{ref}

\end{document}